\documentclass{article}

\usepackage[final, nonatbib]{neurips_2022}

\usepackage[table]{xcolor}
\usepackage[utf8]{inputenc} 
\usepackage[T1]{fontenc}    
\usepackage{hyperref}       
\usepackage{url}            
\usepackage{booktabs}       
\usepackage{amsfonts}       
\usepackage{nicefrac}       
\usepackage{microtype}      
\usepackage{amsmath}
\usepackage{amsthm}
\usepackage{bbm}
\usepackage{graphicx}
\usepackage{epstopdf}
\usepackage{subfigure}
\usepackage{amssymb}
\usepackage{sectsty}
\usepackage{wrapfig}
\usepackage{multirow}
\usepackage{makecell}
\usepackage{algorithmicx,algorithm,balance}
\usepackage{algpseudocode}
\usepackage{pifont}
\usepackage{nicematrix,tikz}
\hypersetup{hidelinks} 

\newtheorem{definition}{Definition}
\newtheorem{proposition}{Proposition}

\newtheorem{theorem}{Theorem}
\newtheorem{assumption}{Assumption}

\title{Sustainable Online Reinforcement Learning for Auto-bidding}

\author{%
	Zhiyu Mou$^{1,2}$\thanks{Work was done during an internship at Alibaba Group.} \;\;
	Yusen Huo$^1$ \;
	Rongquan Bai$^1$\;
	Mingzhou Xie$^1$\;
		Chuan Yu$^1$\\
	\textbf{Jian Xu}$^1$\;
	\textbf{Bo Zheng}$^1$\thanks{Corresponding author.}
	\\
	$^1$ Alibaba Group, Beijing, China\\
	$^2$ Department of Automation, Tsinghua University, Beijing, China\\
	\texttt{mouzy20@mails.tsinghua.edu.cn} \\
	\texttt{\{huoyusen.huoyusen, rongquan.br, mingzhou.xmz,}\\
	\texttt{yuchuan.yc, xiyu.xj, bozheng\}@alibaba-inc.com}
}
\begin{document}

	\maketitle

	\begin{abstract}
		Recently, auto-bidding technique has become an essential tool to increase the revenue of advertisers. Facing the complex and ever-changing bidding environments in the real-world advertising system (RAS), state-of-the-art auto-bidding policies usually leverage reinforcement learning (RL) algorithms to generate real-time bids on behalf of the advertisers. Due to safety concerns, it was believed that the RL {training} process can only be carried out in an offline virtual advertising system (VAS) that is built based on the historical data generated in the RAS. In this paper, we argue that there exists significant gaps between the VAS and RAS, making the RL {training} process suffer from the problem of \emph{inconsistency between online and offline} (IBOO). Firstly, we formally define the IBOO and systematically analyze its causes and influences. Then, to avoid the IBOO,
		we propose a sustainable online RL (SORL) framework that trains the auto-bidding policy by directly interacting with the RAS, instead of learning in the VAS.
		Specifically, based on our proof of the Lipschitz smooth property of the Q function, we design a safe and efficient online exploration (SER) policy for continuously collecting data from the RAS. Meanwhile,
		we derive the theoretical lower bound on the safety of the SER policy. 
		We also develop a variance-suppressed conservative Q-learning (V-CQL) method to effectively and stably learn the auto-bidding policy with the collected data. Finally, extensive simulated and real-world experiments validate the superiority of our approach over the state-of-the-art auto-bidding algorithm.
	\end{abstract}
	\vspace{-4mm}
	
\section{Introduction}
\label{section:intro}
\vspace{-3mm}
In the era of Internet, online advertising business has become one of the main profit models for many companies, such as Google \cite{google} and Alibaba \cite{alibaba}, which, at the same time, benefits millions of advertisers who are willing to bid for impression opportunities. 
Contemporary online advertising systems, as auctioneers, usually have large amount of candidate advertisers contesting for numerous impression opportunities at every moment. Making auction decisions based on accurate evaluation of each impression opportunity for all advertisers within several milli-seconds is computationally infeasible. Therefore, 
a real-world advertising system (RAS) adopts a cascade architecture \cite{stage, auction}. In this architecture, the auction of each impression opportunity is completed through several stages. Without loss of generality, we here simply view the RAS as a system with two stages: stage 1 and stage 2, as shown in Fig. \ref{fig:RAS}. 
\begin{wrapfigure}{r}{3.2cm}
	\setlength{\abovecaptionskip}{-0.4cm}
	\includegraphics[width=3.2cm]{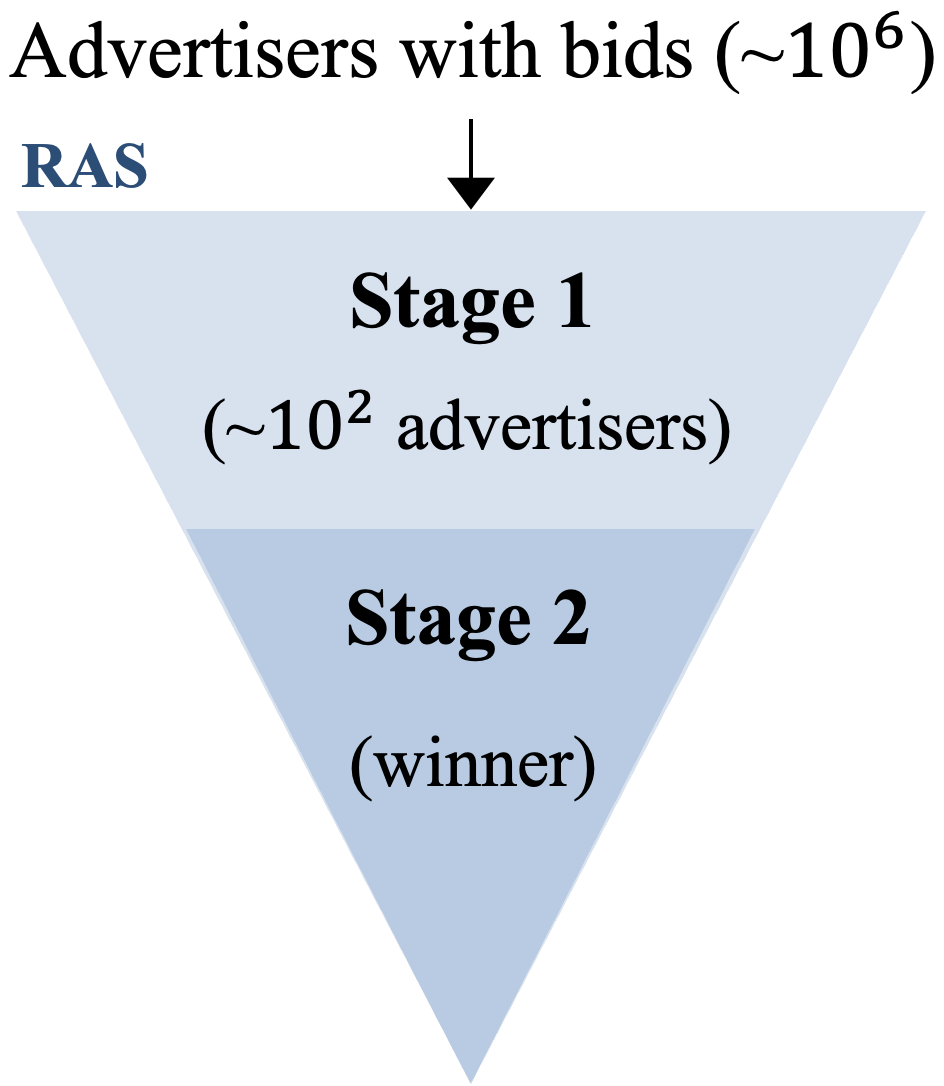}
	\caption{Two-stage auction in RAS.}
	\label{fig:RAS}
	\vspace{-2mm}
\end{wrapfigure}
The auction process of each impression opportunity is as follows: in stage 1, rough 
evaluations are conducted for all the candidate advertisers ($\sim 10^6$), and a group of the most promising advertisers ($\sim 10^2$) are fed to the stage 2; and in stage 2, accurate valuations and auction are carried out to determine the winning advertiser of the impression opportunity. The winner will gain the value of the impression opportunity and pay the market price. {See Appendix \ref{app:ras_vas_ras} for detailed explanations on the RAS structures.} 
Faced with huge amount of impression opportunities at every moment, advertisers cannot bid in the granularity of individual opportunities. Recently, many auto-bidding policies have emerged to realize automatic real-time biddings for each impression opportunity on behalf of advertisers \cite{RL:wangjun, RL:haoxiaotian, RL:cai, RL:USCB}, which significantly increase their revenues. State-of-the-art auto-bidding policies are usually learned with reinforcement learning (RL) algorithms \cite{RL:USCB}. It was believed that the auto-bidding policy being trained cannot directly access to the RAS during the RL training process for safety concerns \cite{RL:wangjun}.
A common solution in most existing works \cite{RL:wangjun, RL:haoxiaotian, RL:cai, RL:USCB} is to {train} the auto-bidding policy in a virtual advertising system (VAS) --- an offline simulated environment that is built based on the {advertisers'} historical data generated in the RAS. 
{See Appendix \ref{app:ras_vas_vas} for the details of the VAS structures.}

\textbf{IBOO Challenges.}
However, we argue that there exists gaps between the RAS and the VAS,
which makes this common solution suffer from the problem of \emph{inconsistencies between online and offline} (IBOO).
Here, \emph{online} refers to the RAS, while \emph{offline} refers to the VAS, and \emph{inconsistencies} refer to the gaps. Formally, we can define the IBOO as follows.
		\begin{definition}[\textnormal{Inconsistencies Between Online and Offline, IBOO}] The IBOO refers to the gaps between the RAS and the VAS that prevent the VAS from accurately simulating the RAS.
	\end{definition}
\vspace{-2mm}
Specifically, there are two dominated gaps between the RAS and the VAS. One is that the VAS {cannot accurately simulate the cascade auction architecture of the RAS. For example, due to the constraint on computing power, the VAS}
is built {only} based on the historical data {generated in} stage 2 {of the RAS\footnote{{Detailed reasons are shown in  Appendix \ref{app:ras_vas_vas}.}}, which makes the VAS lack the mechanism of stage 1 in the RAS.} 
In addition, the VAS does not incorporate the exact influences of other advertisers as RAS does, which makes up the second dominated gap. For example, in the RAS, the market prices are determined by the bids of all advertisers which can change during the training process. 
However, the VAS always provides constant market prices for the auto-bidding policy being trained. See figures in Appendix \ref{app:ill_IBOO} for illustrations.
Essentially, the IBOO makes the VAS provide the auto-bidding policy with false rewards and state transitions during RL training process. 
As presented in Table. \ref{table:IBOO}, the auto-bidding policy with better performance in the VAS (i.e., higher $R/R*$) may yield poorer performance in the RAS (i.e., lower A/B test value). Hence, the IBOO can seriously degrade the performance of the auto-bidding policy in the RAS.
One way to address the IBOO challenge is to improve the VAS and reduce the gaps as much as possible. However, due to the complex and ever-changing nature of the RAS, the reduction in IBOO is usually very limited, resulting in little improvement in the auto-bidding policy performance for illustrations. 
Hence, new schemes need to be devised so as to avoid the IBOO and improve the performance of auto-bidding policies. 
Besides, it is worth noting that the IBOO resembles the \emph{sim2real} problem studied in other realms, such as computer visions \cite{sim2real_cv} and robotics \cite{sim2real_robot}. Nonetheless, to the best of our knowledge, this is the first work that formally put forward the concept of IBOO in the field of auto-bidding and  systematically analyzes and resolves it.


\vspace{-2mm}
\begin{table}[h]
	\caption{Influence of IBOO. $R/R*$ and A/B test values are the performance evaluations of the auto-bidding policies in the VAS and RAS, respectively. Both are the higher the better. We rank the ten policies accordingly. $\uparrow$ means that the performance in the RAS is higher than that in the RAS, while $\uparrow$ means the opposite. $-$ means that the RAS and VAS evaluate the same. }
	\centering
	\small
	\begin{tabular}{cccccc}
		\toprule[1pt]
		\makecell[c]{Policy}   & \makecell[c]{$R/R*$ (rank)}  & \makecell[c]{A/B Tests (rank)}  &	\makecell[c]{Policy}   & \makecell[c]{$R/R*$ (rank)}    & \makecell[c]{A/B Tests (rank)}  \\ \midrule[0.7pt]
	1  &\cellcolor{green!100} $0.9118$ (1) &\cellcolor{red!20} $-2.50\%$  (9) $\downarrow$  & 	6   & \cellcolor{green!34.89} $0.8563$  (6)  &\cellcolor{red!60} $+4.30\%$  (2) $\uparrow$ \\
	2   & \cellcolor{green!51.569} $0.8731$ (2) &\cellcolor{red!51.569} $+3.10\%$   (6) $\downarrow$ &  	7   & \cellcolor{green!32.11} $0.8535$ (7) &\cellcolor{red!80} $+4.40\%$  (1)  $\uparrow$  \\
	3   & \cellcolor{green!44} $0.8656$ (3) & \cellcolor{red!10}$-3.20\%$  (10) $\downarrow$  & 	8   & \cellcolor{green!22.08} $0.8434$ (8)  & \cellcolor{red!25}$-1.50\%$  (8)  $-$ \\
	4   & \cellcolor{green!44} $0.8656$  (4) & \cellcolor{red!40}$+1.20\%$  (7)$\downarrow$    &	9   & \cellcolor{green!21.48} $0.8428$ (9)  &\cellcolor{red!55} $+3.60\%$  (5) $\uparrow$   \\
	5   & \cellcolor{green!38} $0.8594$ (5)  & \cellcolor{red!60}$+4.30\%$  (2) $\uparrow$   & 	10   & \cellcolor{green!15} $0.8111$ (10)   & \cellcolor{red!58}$+3.80\%$  (4) $\uparrow$   \\
 \bottomrule[1pt]	
\end{tabular}

\label{table:IBOO}
\end{table}

In this paper, we propose a novel {sustainable online reinforcement learning} (SORL) framework to address the IBOO challenge. For the first time, the SORL abandons the way of learning with the VAS and trains the auto-bidding policy by directly interacting with the RAS. Notably, the SORL can obtain true rewards and state transitions from the RAS and thereby does not suffer from the IBOO problem.
The SORL contains two main algorithms.
The first one is a \emph{safe and efficient} online exploration policy for collecting data from the RAS, named as the SER policy.
Specifically, to guarantee the safety of explorations, we design a safety zone to sample actions from based on the Lipschitz smooth property of the Q function we theoretically proved. We also derive the lower bound on the safety degree.
To increase the efficiency of explorations, we develop a sampling distribution that can make the collected data give more feedbacks to the auto-bidding policy being trained.
The second main algorithm is an \emph{effective and stable} offline RL method to train the auto-bidding policy based on the collected data, named as the {variance-suppressed conservative Q-learning} (V-CQL). 
Specifically, motivated by the observation that the optimal\footnote{Here "optimal" refers to the optimal Q function in the simulated experiments. See Section \ref{sec:V-CQL} for details.} Q function is in the quadratic form, we design a regularization term in the V-CQL to optimize the shape of the Q function. The V-CQL can train the auto-bidding policy with high average, hence {effective}, and low variance, hence {stable}, in performance under different random seeds.
The whole SORL works in an iterative manner, alternating between collecting data and training the auto-bidding policy.
Extensive simulated and real-world experiments validate the effectiveness of our approach.



\section{Related Work}

\label{section:related_work}
\textbf{VAS-based RL Auto-bidding Methods.}
Impressed by powerful contextual learning and sequential decision-making capabilities of RL, modern auto-bidding policies, such as DRLB \cite{RL:wangjun}, RLB \cite{RL:cai}, MSBCB \cite{RL:haoxiaotian} and USCB (state-of-the-art) \cite{RL:USCB}, are usually learned by RL algorithms in a manually built VAS. However, as stated before, they all suffer from the IBOO problem. The SORL avoids using the VAS and thereby completely address the IBOO challenge.

\textbf{Offline RL.} 	Offline RL (also known as batch RL) \cite{BCQ,CQL,BCQ_CQL,offline_review,BCQ_1, BCQ_2, CRR, BAIL, BRAC} aims to learn better policies based on a fixed offline dataset collected by some behavior policies. The main challenge offline RL addressed {is} the extrapolation error caused by missing data \cite{offline_review}. Specifically,
offline RL algorithms usually address this challenge {in three ways}, including {policy constraint} methods 
such as BCQ \cite{BCQ}, BEAR \cite{BCQ_1}, and {conservative regularization} methods 
such as CQL \cite{CQL}, BRAC \cite{BRAC}, as well as modifications of imitation learning \cite{onestep} such as ABM \cite{BCQ_2}, CRR \cite{CRR}, BAIL\cite{BAIL}. However, there is another important challenge that offline RL cannot solve:  the fixed offline dataset cannot be guaranteed to contain sufficient state-action pairs from high-reward regions \cite{offline_review}. This challenge exists in many practical applications, including the auto-bidding problem, and can inherently prevent offline RL algorithms from learning excellent policies. Hence, a great advantage of the SORL is its ability to continuously collect data from OOD high-reward regions that can give new feedbacks to the auto-bidding policy being trained.

\textbf{Off-Policy Evaluation (OPE).} OPE is an algorithm to evaluate the performance of the policy with offline data collected by other policies \cite{OPE}. The state-of-the-art OPE algorithm in auto-bidding is to calculate the $R/R*$ of the evaluated policy in the VAS (see Appendix \ref{app:related_work} for details). However, as shown in Table. \ref{table:IBOO}, the OPE in auto-bidding is not accurate. This indicates that we cannot rely on the OPE to select auto-bidding policies with good performance for directing further online explorations. Notably, the proposed offline training algorithm in the SORL, V-CQL, outperforms existing offline RL algorithms in training stabilities and helps to reduce the OPE process during iterations in SORL.

\textbf{Online RL.} Safety is of vital importance in the online exploration of many real-world industrial systems, including the RAS.  Many safe online RL algorithms have been proposed to guarantee the safety of explorations \cite{CSC, irl, icrl, safeRL:query, liyapnouv, learning-based, 5, 6,7}.
However, many of them are either developed based on the constraints that are not suitable for the auto-bidding problem or designed for systems with specific assumptions (see Appendix \ref{app:related_work} for details). 
Besides, many existing works \cite{5,6,7,liyapnouv} assume that there exists a specific set of safe state-action pairs that can be explored. Some work \cite{5} even requires to know the safe states in advance. However, in the auto-bidding problem, no specific actions at any state cannot be explored as long as the expected accumulative reward maintains at a high level. This requires the exploration policy to maintain high performance throughout the iterations, which is more challenging. Notably, the SER policy can meet this requirement with theoretical guarantees. 
Recently, with the development of the offline RL field, many algorithms for efficient online explorations on the premise of having an offline dataset  \cite{AWAC, O2O} have emerged. However, they often only focus on the efficiency of data collections but ignore the safety of it. Notably, the SER policy in the SORL strikes a good balance between the efficiency and safety.

Supplementary related works are described in Appendix \ref{app:related_work}.

\section{Problem Settings}
\label{sec:problem_setting}
In this paper, we study the auto-bidding problem from the perspective of a single advertiser, which can be viewed as an episodic task of $T\in\mathbb{N}_+$ time steps. Between time step $t$ and $t+1$, there are $N_t\in\mathbb{N}_+$ impression opportunities, each of which has a positive value $v_{j,t}\le v_M$ and is sold to the winning advertiser at a market price $p_{j,t}\le p_M$, where $v_M>0$ and $p_M>0$ are the upper bounds for values and market prices, respectively, $j\in\{1,2,...,N_t\}, t\in\{1,2,...,T\}$. Denote
$p^1_{j,t}\le p_M$ and $v^1_{j,t}\le v_M$ as the market price and rough value of impression opportunity $j$ in stage 1, and $p^2_{j,t}\le p_M$ and $v^2_{j,t}\le v_M$ as the market price and accurate value of it in stage 2.
Note that they are all positive values.
Let the total budget of the advertiser be $B>0$. The auto-bidding problem can be modeled as a Constraint Markov Decision Process (CMDP) \cite{CMDP} $<\mathcal{S},\mathcal{A},R,\mathbb{P},\gamma, C>$. The state $s_t\triangleq[b_t,T-t,B-b_t]\in\mathcal{S}$ is composed of three elements, including the budget left $s_t(1)=b_t$, the time left $s_t(2)=T-t$ and the budget consumed $s_t(3)=B-b_t$. The action $a_t\in\mathcal{A}\triangleq[A_\text{min},A_\text{max}]$ is the bidding price at time step $t$, where $A_\text{max}>A_\text{min}>0$ are the upper and lower bounds, respectively. The reward function $r_t(s_t,a_t)$ and the constraint function $c_t(s_t,a_t)$ are the total value of impression opportunities won between time step $t$ and $t+1$ and the corresponding costs, respectively, and $\mathbb{P}$ denotes the state transition rule. Note that $R$, $C$ and $\mathbb{P}$ are all directly determined by the RAS. Moreover, $\gamma\in[0,1]$ is the discounted factor.
We denote the auto-bidding policy as  $\mu:\mathcal{S}\rightarrow\mathcal{A}$, and let $\Pi$ be the policy space. The goal of the auto-bidding problem is to maximize the total value of impression opportunities earned by the advertiser under the budget constraint, which can be expressed as 
\begin{align}
	\label{optimal_prob}
	\max_{\mu\in\Pi}\;V(\mu)\triangleq\mathbb{E}_{s_{t+1}\sim\mathbb{P},a_t\sim\mu}\bigg[\sum_{t=1}^T\gamma^tr_t(s_t,a_t)\bigg],\quad \mathrm{subject \; to }\;\;\sum_{t=1}^T c_t(s_t,a_t)\le B.
\end{align}
One can leverage standard RL algorithms to {train} a state-of-the-art auto-bidding policy \cite{RL:USCB} with the VAS, where the constraint in \eqref{optimal_prob} is met by terminating the training episode early once the budget runs out. However, as stated before, this way of learning suffers from the challenge of IBOO. Hence, we
propose the SORL framework to avoid the IBOO in the following section.

\section{Our Approach}
The SORL framework consists of two algorithms, including the SER policy for online data collections and the V-CQL method for offline training of the auto-bidding policy based on the collected data. The SORL works in an iterative manner, alternating between online data collections and offline training.
\vspace{-1mm}
\subsection{Online Exploration: SER Policy}
\vspace{-1mm}
There are two requirements on the online exploration policy. As the exploration policy is directly deployed in the RAS, the primary requirement is \emph{safety}.
Unlike the safety in other realms such as robotics, it is not appropriate to construct an \emph{immediate-evaluated} safety function to assess the safety degree of each state-action pair merely based on their values in auto-bidding. 
Actually, any action in any state would be safe as long as the summation of all rewards of the whole episode maintains at a high level. 
See Appendix \ref{app:safety_explain} for detailed explanations.
Let $\mu_s$ be the auto-bidding policy originally\footnote{There is always an auto-bidding policy $\mu_s$ deployed in the RAS. For example, the state-of-the-art method USCB \cite{RL:USCB} will do.} deployed in the RAS that is safe, and denote the online exploration policy as $\pi_{e}$.  We can formally define the safety requirement as: \emph{the performance of $\pi_{e}$ cannot be lower than that of $\mu_s$ by a threshold $\epsilon_s>0$, i.e., $V(\mu_s) - V(\pi_{e}) \le \epsilon_s$}. Besides, $\pi_{e}$ needs to be \emph{efficient}, i.e., \emph{collecting data that can give new feedbacks to the auto-bidding policy being trained}, which constitutes the second requirement. We next propose the SER policy for exploration that can satisfy both requirements.
\begin{figure*}
		\setlength{\abovecaptionskip}{-0.01cm}
	\begin{center}
		\subfigure[
		The Q function of any auto-bidding policy is $L_Q$-Lipschitz smooth. The safety zone is designed as the range between $\mu_s(s_t)-\xi$ and $\mu_s(s_t)+\xi$, where  $\xi\le\frac{\epsilon_s}{L_Q\gamma^{t_1}\Delta T}$.
		]{
			\label{fig:Q_lipschitz}
			\includegraphics[width=6cm]{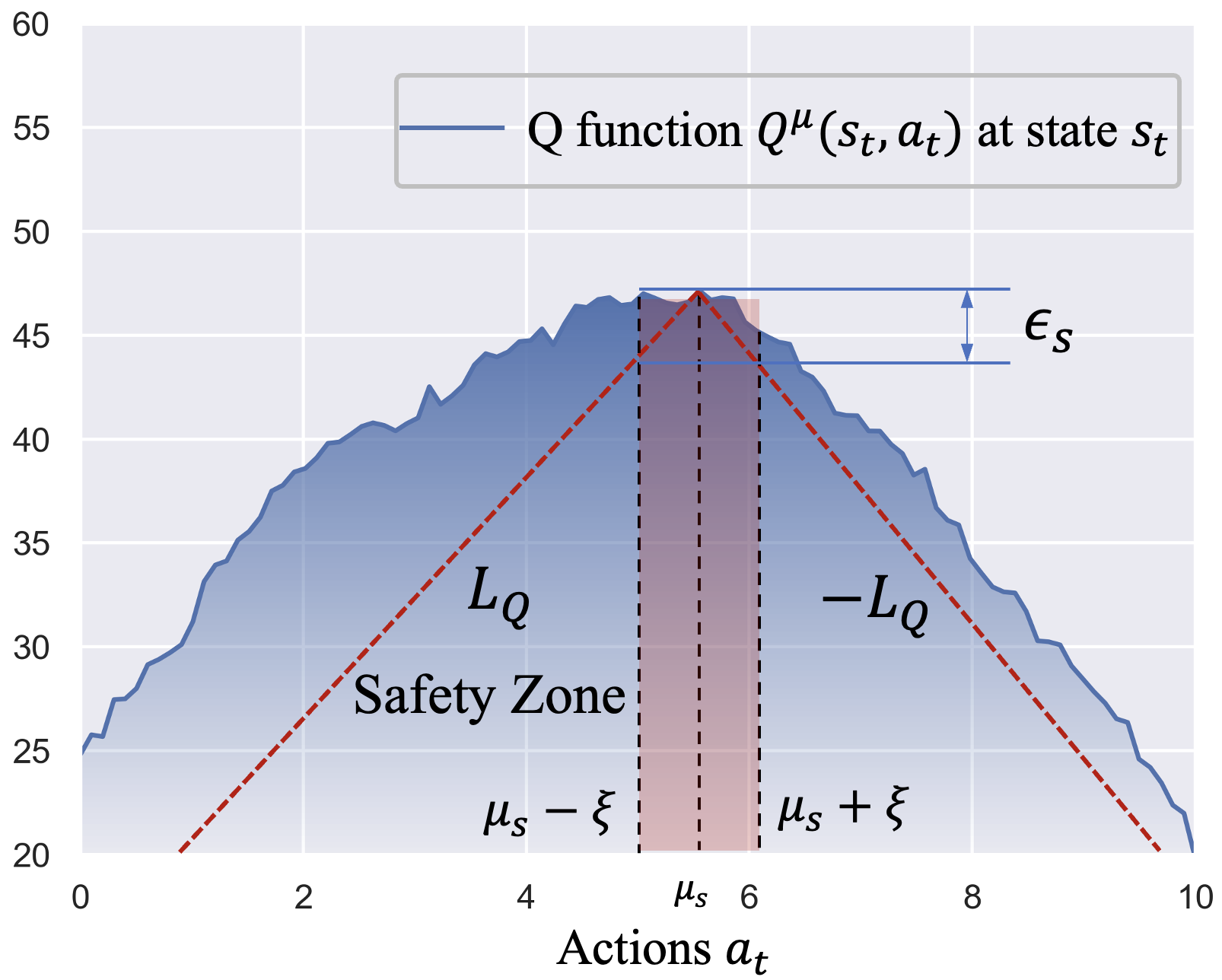}}
		\;\;\;
		\subfigure[Exploration policies $\pi_e$ under different hyper-parameters $\sigma$ and $\lambda$, as well as $\pi_{e,\mathcal{N}}$ with different hyper-parameter $\sigma$.]{
			\label{fig:interaction_policy}
			\includegraphics[width=6.2cm]{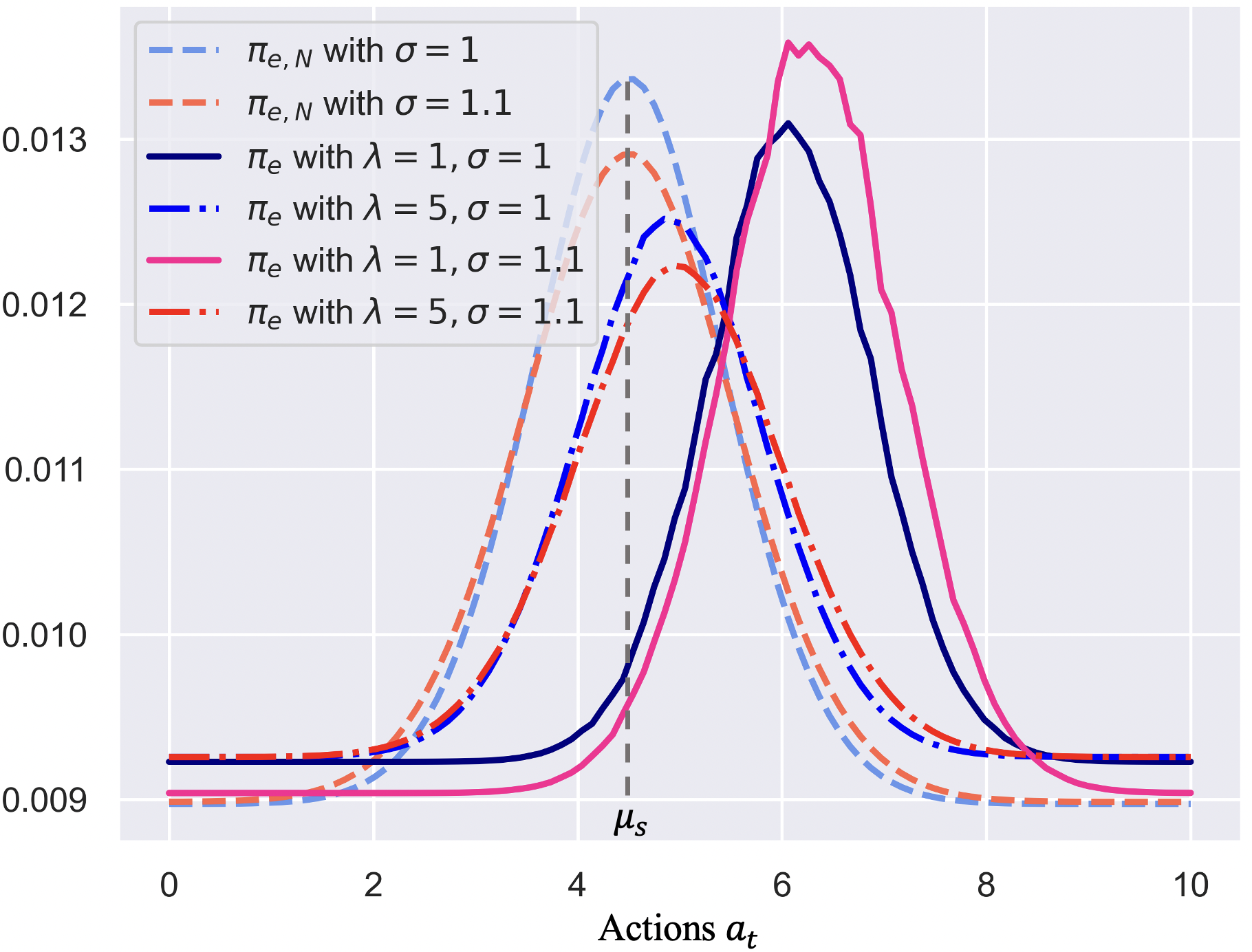}
		}
		\label{fig:alg_safe}
		\caption{Design of the safety zone and sampling distributions in the SER policy.}
	\end{center}
	\vspace{-4mm}
\end{figure*}

\subsubsection{Theory: Lipschitz Q Function}
\label{sec:theory}
Our basic idea of ensuring the safety of exploration is to design a safety zone for taking actions around the outputs of $\mu_s$. The motivation for this idea is our proof of the Lipschitz smooth property of the Q function in the auto-bidding problem. We next report the theorem of this Lipschitz smooth property, as well as corresponding propositions and assumptions.
Let $\mu$ be an arbitrary auto-bidding policy.
According to the Bellman equation, we have:
\begin{align}
	Q^\mu(s_t,a_t)=r_t(s_t,a_t)+\gamma\mathbb{E}_{s_{t+1}\sim \mathbb{P}(\cdot\mid s_t,a_t)}[Q^\mu(s_{t+1},\mu(s_{t+1}))].
\end{align}
Based on the mechanism of RAS, 
we formulate the reward function $R$, the constraint function $C$ and the transition rule $\mathbb{P}$ as follows, and the proof can be found in Appendix \ref{app:proposition_1}. 
\begin{proposition}[\textnormal{Analytical expressions of $R$, $C$ and $\mathbb{P}$}] 
	\label{proposition:form}
	Based on the characteristic of the two-stage cascaded auction in the RAS, we can formulate the reward function $R$ as $r_t(s_t,a_t)=\sum_{j}\mathbbm{1}\{a_tv^1_{j,t}\ge p^1_{j,t},a_tv^2_{j,t}\ge p^2_{j,t}\}v_{j,t}$, the constraint function $C$ as $c_t(s_t,a_t)=\sum_{j}\mathbbm{1}\{a_tv^1_{j,t}\ge p^1_{j,t},a_tv^2_{j,t}\ge p^2_{j,t}\}p_{j,t}$, and the state transition rule $\mathbb{P}$ as $s_{t+1}=s_{t}+[\triangle s_t(1),\triangle s_t(2),\triangle s_t(3)]$, where $\triangle s_t(1)=-\triangle s_t(3)=-\sum_{j}\mathbbm{1}\{a_tv^1_{j,t}\ge p^1_{j,t},a_tv^2_{j,t}\ge p^2_{j,t}\}p_{j,t}$ and $\triangle s_t(2)=-1$.
\end{proposition}
 We make the following assumption on the impression opportunities, whose rationalities can be found in Appendix \ref{app:assumption_1}.
\begin{assumption}[\textnormal{Bounded Impression Distributions}]
	\label{assumption:uniform}
	Between time step $t$ and $t+1$, we assume the numbers of winning impressions with action $a_t$ in the first stage $n_{t,1}$ and the second stage $n_{t,2}$ can both be bounded by linear functions, i.e., $n_{t,1}\le k_1a_t, n_{t,2}\le k_2a_t$, where $k_1,k_2>0$ 
	are constants.
\end{assumption}
Based on this assumption, we claim that the reward function $r_t(s_t,a_t)$ is Lipschitz smooth with respect to actions $a_t$ at any given state $s_t$. See Appendix \ref{app:theorem_1} for the proof.
\begin{theorem}[\textnormal{Lipschitz Smooth of $r_t(s_t,a_t)$}]
	\label{thm:r_smooth}
	Under Assumption \ref{assumption:uniform}, the reward function $r_t(s_t,a_t)$ is $L_r$-Lipschitz smooth with respect to actions $a_t$ at any given state $s_t$, where $L_r=(k_1+k_2)v_M$.
\end{theorem}
We make the following mild assumptions, whose rationalities can be found in Appendix \ref{app:assumption_2}.
\begin{assumption}[\textnormal{Bounded Partial Derivations of $Q^\mu(s_t,a_t)$}]
	\label{assumption:bounded_Q}
	We assume that the partial derivation of $Q^\mu(s_t,a_t)$ with respect to $s_t(1)$ and $s_t(3)$ is bounded, i.e., $\big\vert\frac{\partial Q^\mu(s_t,a_t)}{\partial s_t(1)}\big\vert\le k_3$ and $\big\vert\frac{\partial Q^\mu(s_t,a_t)}{\partial s_t(3)}\big\vert\le k_4$, where $k_3, k_4>0$ are constants.
\end{assumption}
Equipped with Theorem \ref{thm:r_smooth}, we can prove that the Q function $Q^\mu(s_t,a_t)$ is also Lipschitz smooth. See Appendix \ref{app:theorem_2} for proof.
\begin{theorem}[\textnormal{Lipschitz Smooth of $Q^\mu(s_t,a_t)$}]
	\label{thm:Q_smooth}
	Under Assumption \ref{assumption:uniform} and \ref{assumption:bounded_Q}, the Q function $Q^\mu(s_t,a_t)$ is an $L_Q$-Lipschitz smooth function with respect to the actions $a_t$ at any given state $s_t$, where $L_Q=[v_M+\gamma(k_3+k_4)p_M](k_1+k_2)$.
\end{theorem}

This means that the decrease rate of $Q^\mu$, the subsequent accumulated rewards, due to action offset at any time step $t$ is bounded by $L_Q$, which gives us a way to design the safety zone.
Specifically, as shown in Fig. \ref{fig:Q_lipschitz}, the safety zone at state $s_t$ can be design as the neighborhood of action $\mu_s(s_t)$, i.e.,   $[\mu_s(s_t)-\xi,\mu_s(s_t)+\xi]$, where $\xi >0$ is the range.
In this way, the online exploration policy $\pi_{e}$ can be designed as: sampling within the safety zone $[\mu_s(s_t)-\xi,\mu_s(s_t)+\xi]$ in certain $\Delta T\ge 1$ consecutive time steps, and sticking to $\mu_s$ in the rest of $T-\Delta T$ time steps, i.e.,
	\begin{align}
		\pi_e(s_t)=
		\begin{cases}
			\text{sampling from }[\mu_s(s_t)-\xi,\mu_s(s_t)+\xi],&t_1\le t\le t_2;\\
			\mu_s(s_t),&\text{otherwise}.
		\end{cases}\quad \forall t,
\end{align} 
where $0\le t_1<t_2\le T-1$, $\Delta T=t_2-t_1+1$, and $t_1,t_2\in\mathbb{N}^+$.
In the following theorem, 
{we give the upper bound of $|V(\pi_e)-V(\mu_s)|$.}
\begin{theorem}[{\textnormal{Upper Bound of $|V(\pi_e)-V(\mu_s)|$}}]
	\label{thm:lower_bound}
The expected accumulated reward $V(\pi_e)$ satisfies 
	\begin{align}
		\bigg\vert V(\pi_e) - V(\mu_s)\bigg\vert\le \xi\gamma^{t_1}\bigg[v_M+\gamma\big(k_3+k_4\big)p_M\bigg]\big(k_1+k_2\big)\Delta T.
	\end{align}
\end{theorem}
{See Appendix \ref{app:theorem_3} for the detailed proofs. 
Hence, to meet the safety requirement, we design the range of the safety zone as $\xi\le\frac{\epsilon_s}{L_Q\gamma^{t_1}\Delta T}$. 
\subsubsection{Sampling Method}
With the safety guarantee, we can further design the sampling method in $\pi_{e}$ to increase the efficiency of exploration. 
In the following, we use $\tilde{\pi}_\text{x}$ and $\pi_\text{x}$ to denote the sampling distribution and the corresponding exploration policy, respectively, where $\text{x}$ represents the version.

\textbf{Vanilla Sampling Method.} A vanilla {sampling method in the safety zone} can directly be {sampling from a (truncated) Gaussian distribution} $\tilde{\pi}_{e,\mathcal{N}}=\mathcal{N}(\mu_s(s_t),\sigma^2)$ with mean $\mu_s(s_t)$ and variance $\sigma^2$. 
However, $\tilde{\pi}_{e,\mathcal{N}}$ resembles the policy $\mu_s$ to a large extent, which makes the explorations conservative and cannot give sufficient feedbacks to the auto-bidding policy being trained.

\textbf{SER Sampling Method.} 
To lift up the efficiency of explorations, we shift the distribution $\tilde{\pi}_{e,\mathcal{N}}$ towards the
actions of the auto-bidding policy being trained, while constraining the {deviations} within a threshold $\epsilon_e>0$. This can be formulated as a functional optimization problem:
\begin{align}
	\label{optimization}
	\max_{\tilde{\pi}_e,\forall s_t}\;\;\mathbb{E}_{a_t\sim\tilde{\pi}_e(\cdot|s_t)}\widehat{Q}(s_t,a_t)\qquad
	\mathrm{ s.t. }\;\; D_\text{KL}(\tilde{\pi}_e,\tilde{\pi}_{e,\mathcal{N}})\le \epsilon_e,
\end{align}
where the optimization variable $\tilde{\pi}_e$ denotes the shifted distribution, $\widehat{Q}$ is the Q function of the auto-bidding policy being trained.
Using Euler equations, we can derive the form of $\tilde{\pi}_e$ from \eqref{optimization} as 
\begin{align}
	\label{equ:explore_policy}
	\tilde{\pi}_e=\frac{\tilde{\pi}_{e,\mathcal{N}}}{C(s_t)}
	\exp\bigg\{\frac{1}{\lambda}\widehat{Q}(s_t,a_t)\bigg\}=\frac{1}{C(s_t)}\exp\bigg\{\underbrace{-\frac{(a_t-\mu_s(s_t))^2}{2\sigma^2}}_\text{safety}+\underbrace{\frac{1}{\lambda}\widehat{Q}(s_t,a_t)}_\text{efficiency}\bigg\},
\end{align}
where $C(s_t)=\int_{a_t}\exp\{-\frac{(a_t-\mu_s(s_t))^2}{2\sigma^2}+\frac{1}{\lambda}\widehat{Q}(s_t,a_t)\}\mathrm{d}a_t$ acts as the normalization factor. The complete deductions of $\tilde{\pi}_e$ are given in Appendix \ref{app:exploration_derivation}. We can interpret $\tilde{\pi}_e$ as a deviation from $\tilde{\pi}_{e,\mathcal{N}}$ towards the Q function $\widehat{Q}(s_t,a_t)$, where $\tilde{\pi}_{e,\mathcal{N}}$ further guarantees the safety and the Q function ensures the efficiency.
Note that $\sigma$ and $\lambda$ in \eqref{equ:explore_policy} are both hyper-parameters that can control the randomness and {deviations} from $\mu_s$ of the SER policy, respectively. As shown in Fig. \ref{fig:interaction_policy}, the smaller the value of $\lambda$, the larger the {deviations} of $\tilde{\pi}_e$ from $\mu_s$; besides, the bigger the value of $\sigma$, the greater the randomness of $\tilde{\pi}_e$. Hence, we can control the randomness and {deviations} from the safe policy $\mu_s$ of exploration policy easily by adjusting the value of $\sigma$ and $\lambda$. In addition, we describe the way of practically implementing the sampling of distribution $\tilde{\pi}_e$ in Appendix \ref{app:practical_implementations}.
Then the complete SER policy is
	\begin{align}
	\pi_e(s_t)=
	\begin{cases}
		\text{sampling from }[\mu_s(s_t)-\xi,\mu_s(s_t)+\xi]\text{ with }\tilde{\pi}_e ,&t_1\le t\le t_2;\\
		\mu_s(s_t),&\text{otherwise}.
	\end{cases}\quad \forall t.
\end{align} 

\subsection{Offline Training: V-CQL}
\label{sec:V-CQL}
Due to the safety constraint on the explorations, the SER policy can only collect data within the safety zone and the data outside the safety zone will be missed. Hence, we leverage a strong baseline offline RL algorithm, CQL, to address the extrapolation error and make the offline training \emph{effective}.
Moreover, as the sampling distribution $\tilde{\pi}_e$ in the SER policy involves the Q function $\widehat{Q}$, the training result will affect the directions of further explorations. However, on the one hand, as stated in Section \ref{section:related_work}, the OPE in the auto-bidding is not reliable; and on the other hand, the performance variance under different random seeds of existing offline RL algorithms can be large, as shown in Fig. \ref{fig:large_variance} in the experiment. Hence, we further design a novel regularization term in the CQL loss to increase the \emph{stability} of offline training and thereby reduce the OPE process. This forms the V-CQL method. Specifically, we observe that the Q functions of the state-of-the-art auto-bidding policy are always in nearly quadratic forms. In addition, we conduct simulated experiments where the auto-bidding policy is directly trained in a simulated RAS\footnote{The details of this simulated RAS are described in Appendix \ref{app:concept_as}} with traditional RL algorithms. We find out the optimal Q functions in the simulated experiments are also in nearly quadratic forms. See Appendix \ref{app:V-CQL_1} for illustrations. Based on these observations, we assume that the optimal Q functions in the RAS are also in nearly quadratic forms. Hence, the key idea of V-CQL is to restrict the Q function to remain nearly quadratic, and the regularization term in the V-CQL is designed as:
\begin{align}
	\label{equ:V-CQL}
	\mathcal{R}(\mu)=\beta\mathbb{E}_{s_k\sim\mathcal{D}_s}\bigg[D_\text{KL}\bigg(\underbrace{\frac{\exp(\widehat{Q}(s_k,\cdot))}{\sum_a\exp(\widehat{Q}(s_k,a))}}_{\text{distribution(form) of $\widehat{Q}$}}\;,\;\underbrace{\frac{\exp(\widehat{Q}_\text{qua}(s_k,\cdot))}{\sum_a\exp(\widehat{Q}_\text{qua}(s_k,a))}}_{\text{distribution(form) of $\widehat{Q}_\text{qua}$}}\bigg)\bigg],
\end{align} 
where $Q_\text{qua}$ is selected as the Q function of the state-of-the-art auto-bidding policy, $D_\text{KL}(\cdot,\cdot)$ represents the KL-divergence, and $\beta>0$ is a constant controlling the weight of $\mathcal{R}(\mu)$. Moreover, we can interpret \eqref{equ:V-CQL} from another aspect, i.e., $\mathcal{R}(\mu)$ tries to limit the distance between the distribution of the Q function and the distribution of $\widehat{Q}_\text{qua}$. This can limit the derivation of the auto-bidding policy from the corresponding state-of-the-art policy, which can increase the training stability.
Complete implementations of the V-CQL are shown in Appendix \ref{app:V-CQL_2}.

\subsection{Iterative Updated Structure}
Fig. \ref{fig:SORL} shows the whole structure of the SORL framework. Specifically, the SORL works in an iterative manner, alternating between collecting data with the SER policy from the RAS and training the auto-bidding policy with the V-CQL method.

\begin{figure}[t]
	\centering
	\includegraphics[width=13.5cm]{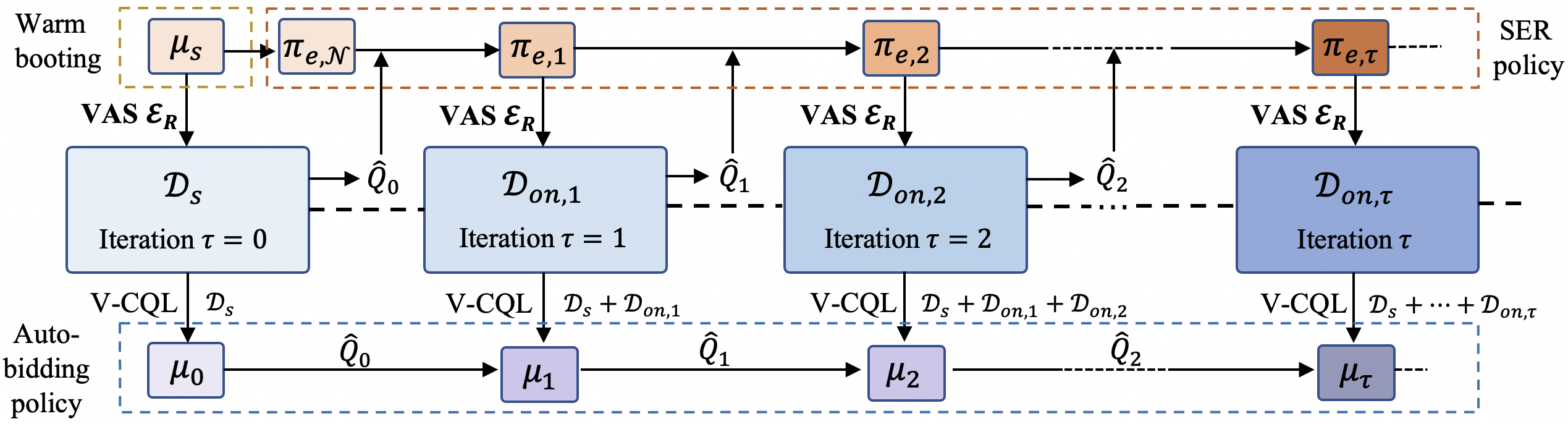}
	\caption{Sustainable online reinforcement learning (SORL) framework which learns the auto-bidding policy directly with the RAS in a safe and efficient way The SORL works in an iterative manner. }
	\label{fig:SORL}
\end{figure}


\textbf{Warm Booting.}
To start with, we use the safe policy $\mu_s$ to boot the explorations in the RAS. 
Denote the data collected by $\mu_s$ as $\mathcal{D}_s=\{(s_k,a_k,r_k,s'_k)\}_{k}$, and the auto-bidding policy trained by $\mathcal{D}_s$ with the V-CQL as $\mu_0$.

\textbf{Iteration Process.}
Denote the SER policy and the auto-bidding policy in the $\tau$-th iteration as $\pi_{e,\tau}$ and $\mu_{\tau}$, and the data collected in the $\tau$-th iteration as $\mathcal{D}_{on,\tau}$, where $\tau\in\{1,2,...\}$
The design for {the sampling distribution $\tilde{\pi}_{e,\tau}$ in iteration $\tau$ } is
$\tilde{\pi}_{e,\tau}=\frac{1}{C_{\tau}(s_t)}\tilde{\pi}_{e,\mathcal{N}}\exp\{\widehat{Q}_{\tau}(s_t,a_t)/\lambda_{\tau}\}$, {where $C_{\tau}(s_t)$ acts as the normalization factor and $\lambda_\tau$ is the hyper-parameter used in iteration $\tau$, and $\widehat{Q}_\tau$ is the Q function of $\mu_\tau$.}
{We note that the exploration policies in each iteration are safe, i.e., $|V(\pi_{e,\tau})-V(\mu_s)|\le \epsilon_s,\forall \tau$, since they all taking actions within the safety zone only with different sampling distributions.  
We leverage the V-CQL to in each iteration improve the auto-bidding policy. Specifically, at iteration $\tau$, we substitute $\widehat{Q}_\text{qua}$ in \eqref{equ:V-CQL} with $\widehat{Q}_{\tau-1}$, and train a new Q function $\widehat{Q}\leftarrow\widehat{Q}_\tau$ for the next iteration. The auto-bidding policy are expected to be continuously improved.
A summary of the SORL framework is present in Appendix \ref{app:sorl}.



\section{Experiments}
\label{sec:experiment}
\vspace{-2mm}
We conduct both simulated and real-world experiments to validate the effectiveness of our approach\footnote{The codes of simulated experiments are available at \href{https://github.com/nobodymx/SORL-for-Auto-bidding}{https://github.com/nobodymx/SORL-for-Auto-bidding}.}.
The following three questions are mainly studied in the experiments:
(1) What is the performance of the whole SORL framework? Is the SER policy safe during iterations? Can the V-CQL method continuously improve the auto-bidding policy and outperform existing offline RL algorithms and the state-of-the-art auto-bidding policy?
(2) Is the safety zone reliable? Is the SER policy still safe when using Q functions of auto-bidding policies with bad performance?
 (3) Does the V-CQL really help to reduce the performance variance compared to existing offline RL algorithms?

\textbf{Experiment Setup.}
We conduct the real-world experiments on one of the world's largest E-commerce platforms, TaoBao. See Appendix \ref{app:experiment_params} for details. The simulated experiments are conducted in a manually built offline RAS and the corresponding VAS. See Appendix \ref{app:concept_as} for details.
The safe auto-bidding policy $\mu_s$ used for warm booting and constructing the safety zone is trained by the state-of-the-art auto-bidding policy,  USCB\cite{RL:USCB}.
The safety threshold is set as $\epsilon_s=5\% V(\mu_s)$.

\textbf{Performance Index.}
The objective function $V(\mu)$ in \eqref{optimal_prob}, i.e., the total value of impression opportunities won by the advertiser in the episode, acts as the main performance index in our experiments and is referred as \emph{BuyCnt} in the following. 
In addition, we utilize three other metrics that are commonly used in the auto-bidding problem to evaluate the performance of our approach.
The first metric the total \emph{consumed budget (ConBdg)} of the advertiser. 
The second metric is the \emph{return on investment (ROI)} which is defined as the ratio between the total revenue and the {ConBdg} of the advertiser. The third metric is the \emph{cost per acquisition (CPA)} which is defined as the average cost for each successfully converted impression. Note that larger values of BuyCnt, {ROI}, and ConBdg with a smaller {CPA} indicate better performance of the auto-bidding policy. 

\textbf{Baselines.} 
We compare our approach with the state-of-the-art auto-bidding policy, USCB \cite{RL:USCB}, that is trained by RL in the VAS. 
We also compare the V-CQL method with modern offline RL algorithms, including BCQ \cite{BCQ} and CQL \cite{CQL}. 
Recall that many safe online RL algorithms are not suitable for explorations in the auto-bidding problem as stated in Section \ref{section:related_work}, we compare the SER policy with $\pi_{e,\mathcal{N}}$ in our experiments.

\vspace{-4mm}
\subsection{Main Results}
\label{section:main_results}
\vspace{-1mm}
\textbf{To Answer Question (1):} We first conduct simulated experiments with the SORL, and the results during iterations are shown in Fig. \ref{fig:main_res}. 
Specifically, from Fig. \ref{fig:SORL_sim_explore}, we can see that the decline rate in BuyCnt of both the SER policy and the vanilla exploration policy are smaller than $5\%$, which verifies the safety of the SER policy. Moreover, the BuyCnt of the SER policy $\pi_{e,\tau}$ rises with iterations and is alway higher than that of $\pi_{e,\mathcal{N}}$, which indicates that the SER policy is more efficient. From Fig. \ref{fig:SORL_sim_target}, we can see that the BuyCnt of the auto-bidding policy $\mu_{\tau}$ trained with V-CQL 
is higher than that of BCQ, CQL and USCB, which indicates the effectiveness of the V-CQL.
Besides, the BuyCnt rises with the number of iterations and converges to the optimal BuyCnt (i.e., the BuyCnt of the optimal policy) at the $5$-th iteration. This validates the superiority of the whole SORL framework.
This indicates the effectiveness of the V-CQL method. For real-world experiments, we utilize $10,000$ advertisers to collect data from the RAS with $\pi_e$ and compare the auto-bidding policies in $4$ iterations on $1,500$ advertisers using A/B tests, and the results are shown in Table. \ref{table:SORL}. We can see that the performances of auto-bidding policies are getting better with iterations and exceeds the state-of-the-art algorithm, USCB, which validates the superiority of the whole SORL framework.

\begin{figure*}
	\begin{center}
		\subfigure[BuyCnt of $\pi_{e,\tau}$ and $\pi_{e,\mathcal{N}}$.]{
			\label{fig:SORL_sim_explore}
			\includegraphics[width=6.5cm]{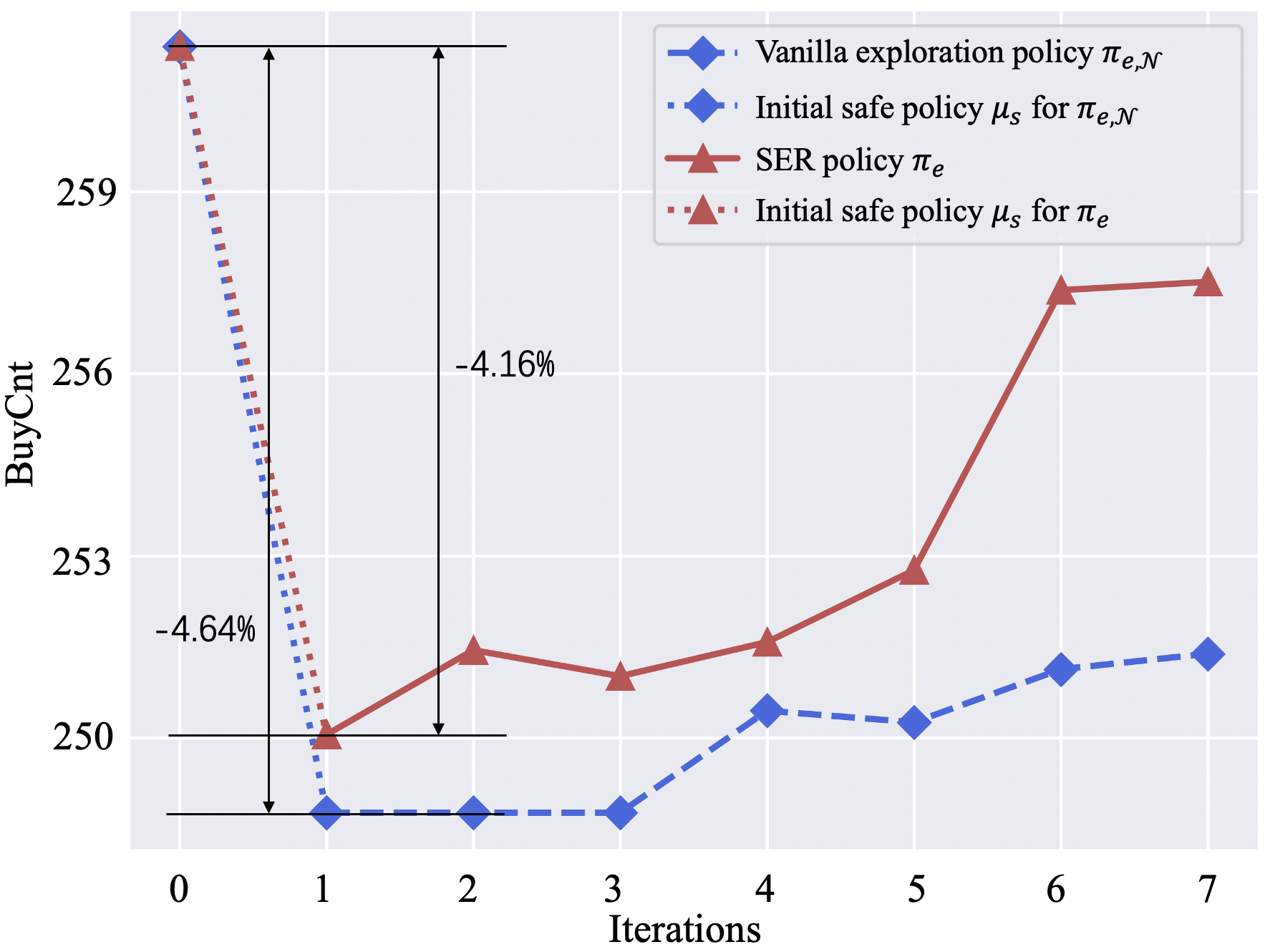}}
		\subfigure[BuyCnt of $\mu$ using different algorithms.]{
			\label{fig:SORL_sim_target}
			\includegraphics[width=6.5cm]{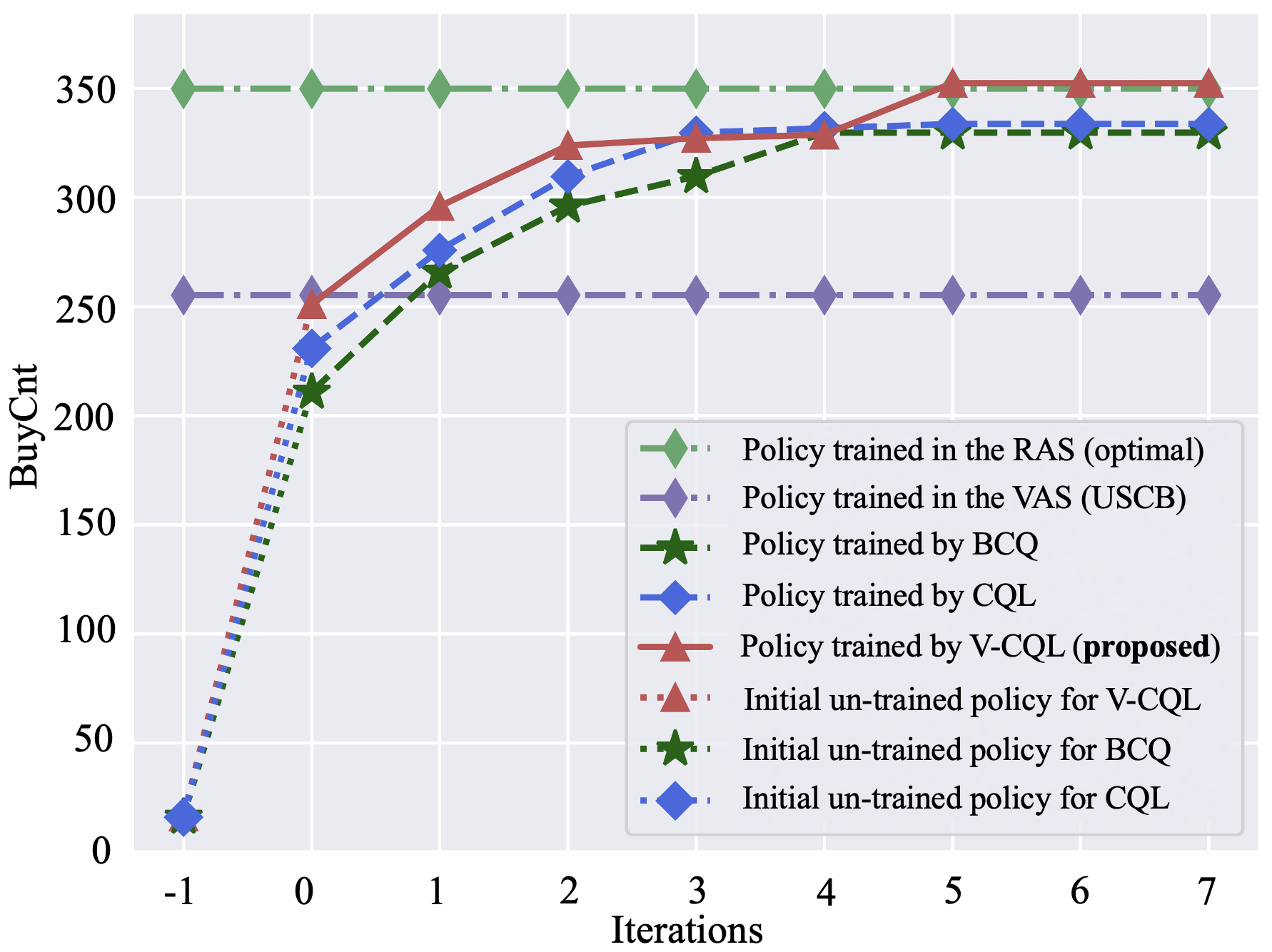}
		}
		\label{fig:main_res}
		\caption{The change of BuyCnt of both the SER policy $\pi_{e,\tau}$ and auto-bidding policy $\mu_{\tau}$ with iterations $\tau$ when applying the SORL framework.}
	\end{center}
\vspace{-2mm}
\end{figure*}

\begin{table}[h]
		\setlength{\belowcaptionskip}{-0.01cm}
	\caption{The results of SORL framework in the real-world experiments.}
	\centering
	\small
	\begin{tabular}{lllllllll}
		\toprule
		\multirow{2}{*}{Iterations}&\multicolumn{4}{c}{A/B Tests with safe policy $\mu_s$} & \multicolumn{4}{c}{A/B Tests with USCB }      \\   
		\cmidrule(r){2-5}		\cmidrule(r){6-9}
		&BuyCnt &ROI&CPA&ConBdg&BuyCnt&ROI&CPA&ConBdg\\
		\midrule
		$0$-th:$\;\mu_{0}$&\textbf{+3.21\%}&\textbf{+1.28\%}&\textbf{-2.01\%}&\textbf{+1.12\%}&\textbf{+3.20\%}&\textbf{+1.28\%}&\textbf{-2.01\%}&\textbf{+1.12\%}\\
		$1$-th:$\;\mu_{1}$&\textbf{+0.65\%}&\textbf{+1.96\%}&\textbf{-1.27\%}&-0.62\%&\textbf{+3.41\%}&\textbf{+2.88\%}&\textbf{-0.93\%}&\textbf{+2.45\%}\\
		$2$-th:$\;\mu_{2}$&\textbf{+0.47\%}&\textbf{+0.26\%}&\textbf{-0.13\%}&\textbf{+0.33\%}&\textbf{+3.57\%}&\textbf{+1.60\%}&\textbf{-0.98\%}&\textbf{+2.55\%}\\
		$3$-th:$\;\mu_{3}$&\textbf{+0.95\%}&\textbf{+3.20\%}&\textbf{-1.01\%}&\textbf{+0.06\%}&\textbf{+3.75\%}&\textbf{+2.48\%}&\textbf{-3.91\%}&-0.15\%\\
		\bottomrule
	\end{tabular}
	\label{table:SORL}
	\vspace{-3mm}
\end{table}

\subsection{Ablation Study}
\label{section:albation_study}
\vspace{-2mm}
\begin{wrapfigure}{r}{4.5cm}
	\vspace{-5mm}
		\setlength{\abovecaptionskip}{-0.1cm}
	\includegraphics[width=4.5cm]{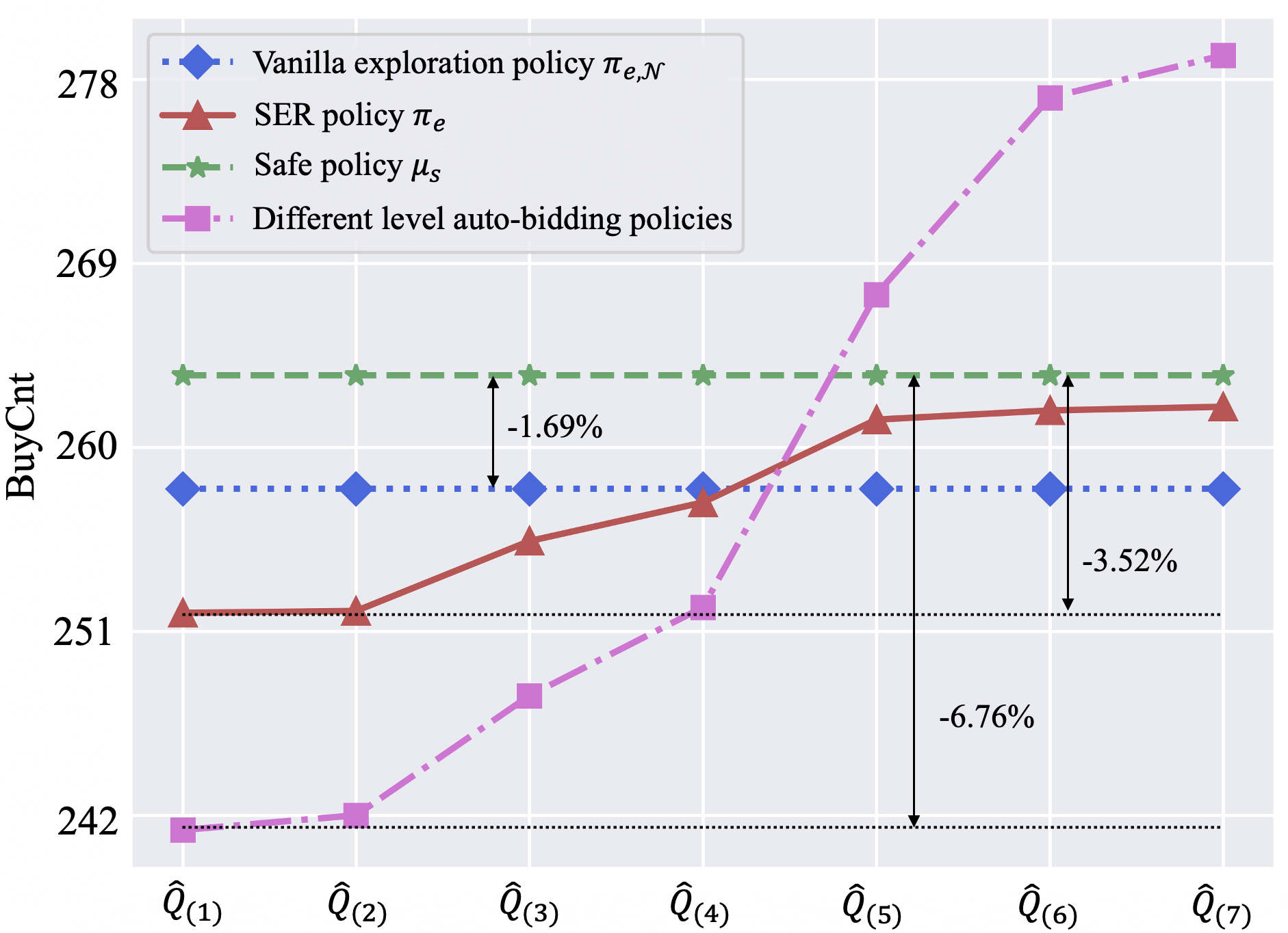}
	\caption{BuyCnt of $\pi_{e,\mathcal{N}}$, $\pi_{e}$.}
	\label{fig:safe_exploration}
	\vspace{-5mm}
\end{wrapfigure}
\textbf{To Answer Question (2):} 
We fully examine the safety of the SER policy using the Q function $\widehat{Q}$ of auto-bidding policies with different performance levels.
Specifically, 
in the simulated experiment, we utilize seven different versions of auto-bidding policies $\{\mu_{(k)}\}^7_{k=1}$ with Q functions $\{\widehat{Q}_{(k)}\}^7_{k=1}$ that have various performance levels to construct seven corresponding SER policies, where $V(\mu_{(1)})<V(\mu_{(2)})<...,V(\mu_{(7)})$. We also construct a vanilla exploration policy $\pi_{e,\mathcal{N}}$ for comparison. The hyper-parameters for all exploration policies are $\sigma=1, \lambda=0.1$.
We apply the $\pi_{e,\mathcal{N}}$ and seven SER policies to the simulated RAS, and the BuyCnt of explorations policies are shown in Fig. \ref{fig:safe_exploration}. 
We can see that the BuyCnt of the SER policy rises with the performance level of the auto-bidding policy. The worst BuyCnt of the SER policy drops about $3.52\%$ compared with $\mu_s$, which meets the safety requirement. Moreover, we can balance between the safety and efficiency of the SER policy by regulating the hyper-parameters $\sigma$ and $\lambda$. The corresponding results are shown in Appendix \ref{app:q2}.

\begin{wraptable}{r}{6cm}
	\vspace{-6.8mm}
	\caption{Real-world A/B tests between $\pi_e$ and $\mu_s$, as well as between $\pi_{e,\mathcal{N}}$ and  $\mu_s$.}
	\centering
	\small
	\begin{tabular}{lll}
		\toprule
		Methods&BuyCnt& ConBdg       \\   
		\midrule
		vanilla $\pi_{e,\mathcal{N}}$ &-3.32\%&+1.11\%\\
		SER policy $\pi_e$ &\textbf{-1.83\%}&\textbf{+0.67\%}\\
		\bottomrule
	\end{tabular}
	\label{table:safe_exploration}
	\vspace{-4mm}
\end{wraptable}
For real-world experiments, we apply $\pi_e$ and $\pi_{e,\mathcal{N}}$ to $1,300$ advertisers in the RAS and compare their performance to $\mu_s$ in A/B tests, as shown in Table. \ref{table:safe_exploration}. We can see that the varieties of the ConBdg and BuyCnt of $\pi_e$ are all with in $5\%$, and are better than those of $\pi_{e,\mathcal{N}}$, which indicates the safety of the SER policy.


\textbf{To Answer Question (3):} We leverage the V-CQL, CQL, BCQ and USCB to train auto-bidding policies under $100$ different random seeds in the simulated experiment, and the results are shown in Fig. \ref{fig:res_3_1}. We can see that the performance variance of the V-CQL is much smaller than those of other algorithms. At the same time, the average performance of the V-CQL can maintains at a high level. 
In addition, we leverage the V-CQL and CQL algorithm to train the auto-bidding policies based on the real-world data under $100$ different random seeds. To ensure fairness, we also utilize the USCB to train the auto-bidding policies in the VAS under exactly the same random seeds.
We carry out the OPE for these three sets of auto-bidding policies, as shown in  Fig. \ref{fig:res_3_2}.
Recall that $R/R*$ is the main

\begin{wraptable}{r}{5.4cm}
	\vspace{-2.7mm}
	\caption{Real-world A/B tests between the V-CQL and USCB under different random seeds.}
	\centering
	\small
	\label{table:random_seed}
	\begin{tabular}{llll}
		\toprule
		Seeds&BuyCnt& Seeds&BuyCnt     \\   
		\midrule
		\;\;\;1 &\textbf{+1.70\%}&\;\;\;6&\textbf{+1.94\%}\\
		\;\;\;2 &\textbf{+1.06\%}&\;\;\;7&\textbf{+1.38\%}\\
		\;\;\;3 &\textbf{+3.01\%}&\;\;\;8&\textbf{+1.19\%}\\
		\;\;\;4 &\textbf{+1.89\%}&\;\;\;9&\textbf{+0.31\%}\\
		\;\;\;5 &-0.53\%&\;\;10&\textbf{+0.71\%}\\
		\bottomrule
	\end{tabular}
	\label{table:}
	\vspace{-4mm}
\end{wraptable}
main metric used in the OPE for auto-bidding (see Appendix \ref{app:related_work}). 
We can see that the maximum $R/R^*$ of the V-CQL is much larger than that of the USCB, which indicates that the V-CQL is capable of training a better policy. Although the maximum $R/R*$ of the V-CQL and the CQL are about in the same level, the variance of $R/R*$ of the V-CQL is smaller than that of the CQL. 
For real-world experiments, 
we apply the auto-bidding policies trained by the V-CQL and USCB to $1,600$ advertisers in the A/B test for seven days in the RAS, and the average values of the metrics are shown in Table \ref{table:CQL_vs_UCSB}. We can see that the V-CQL outperforms USCB method in almost all metrics. Moreover, we present the A/B test results under $10$ random seeds in Table. \ref{table:random_seed}. We can see that the V-CQL outperforms the USCB under $9$ seeds, and only slightly worse under the other one seed.
All these results indicate that the V-CQL can really help to reduce the performance variance and increase the training stability, while keeping the average performance at a high level. 
\vspace{-1mm}
	\begin{table}[h]
			\setlength{\belowcaptionskip}{-0.01cm}
	\caption{Real-world A/B Tests between the V-CQL and USCB.}
	\centering
	\small
	\begin{tabular}{lllllllll}
		\toprule
		\multirow{2}*{Methods}&\multicolumn{4}{c}{USCB: newly trained policy} &   \multicolumn{4}{c}{USCB : safe policy $\mu_s$}             \\   
		\cmidrule(r){2-5}\cmidrule(r){6-9}
		& BuyCnt&ROI&CPA&ConBdg     & BuyCnt  &ROI &CPA&ConBdg    \\
		\midrule
		USCB & 40,926  & 3.90&20.71 &847,403.12&35,627&3.82&21.58&\textbf{768,832.64}  \\
		V-CQL    &\textbf{42,236} & \textbf{3.95} &\textbf{20.29}&\textbf{856,913.14}&\textbf{37,090} &\textbf{3.97}&\textbf{20.61}&764,467.39   \\
		\midrule
		varieties & \textbf{+3.20\%}&\textbf{+1.28\%}&\textbf{-2.03\%} &\textbf{+1.12\%}&\textbf{+4.11\%}&\textbf{+3.93\%}&\textbf{-4.49\%}&-0.57\%\\
		\bottomrule
	\end{tabular}
	\label{table:CQL_vs_UCSB}
	\vspace{-4mm}
\end{table}

\begin{figure*}
	\setlength{\abovecaptionskip}{-0cm}
	\begin{center}
		\subfigure[BuyCnt of the auto-bidding policy trained by different methods in the simulated RAS.]{
			\label{fig:res_3_1}
			\includegraphics[width=6cm]{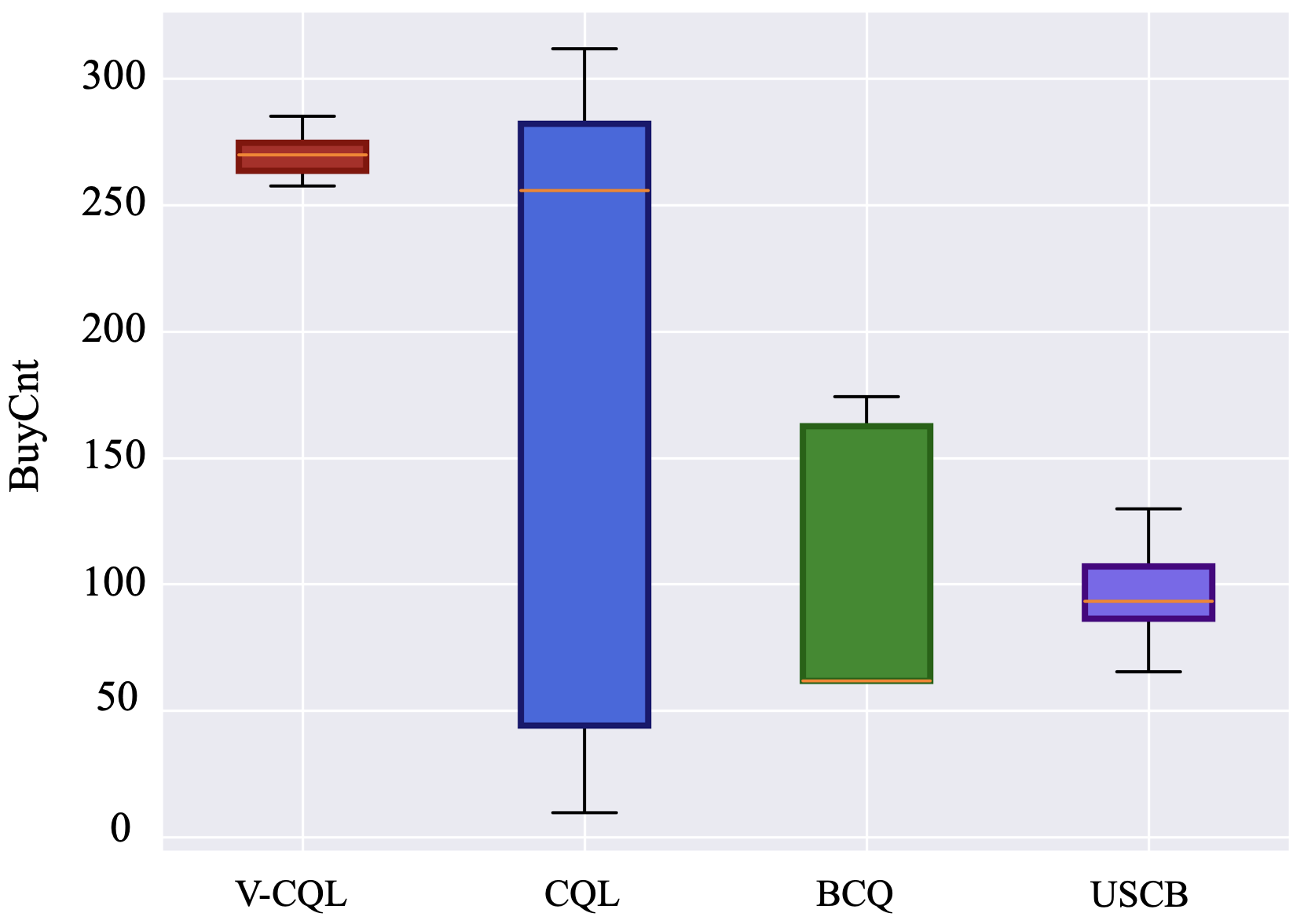}}
		\;\;\;
		\subfigure[OPE of the auto-bidding policy trained by different methods with real-world data.]{
			\label{fig:res_3_2}
			\includegraphics[width=6.5cm]{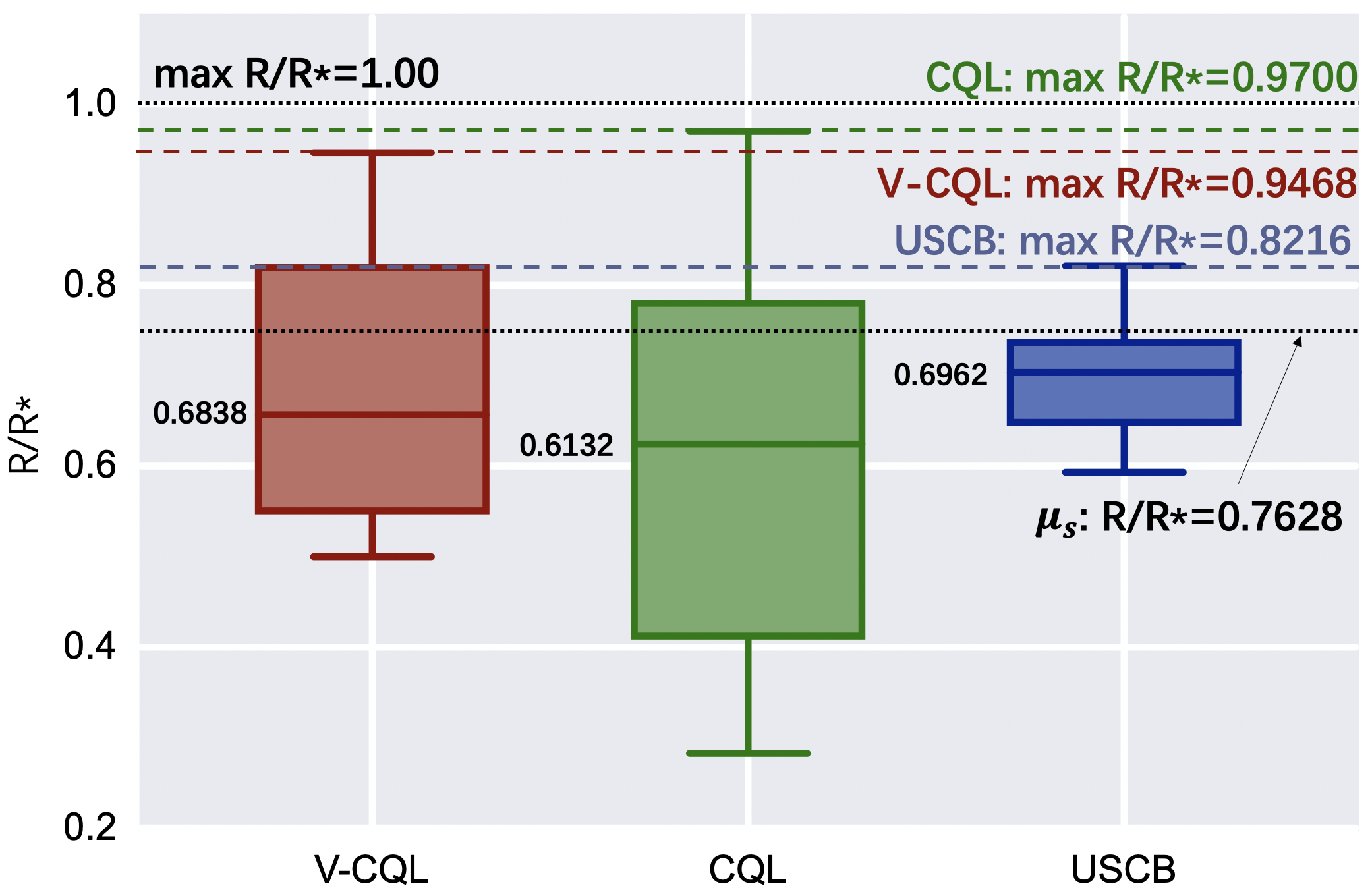}
		}
		\label{fig:res_3}
		\caption{The performance of different methods under 100 random seeds.}
	\end{center}
\label{fig:large_variance}
	\vspace{-6mm}
\end{figure*}

\vspace{-4mm}
\section{Conclusions}
\label{sec:conclusion}
\vspace{-2mm}
In this paper, we study the auto-bidding problem in online advertisings. Firstly, we define the IBOO challenge in the  auto-bidding problem and systematically analyze its causes and influences. Then, to avoid the IBOO, 
we propose the SORL framework that can directly learn the auto-bidding policy with the RAS. Specifically, the SORL framework contains two main algorithms, including a safe and efficient online exploration policy, the SER policy, and an effective and stable offline training method, the V-CQL method. 
The whole SORL framework works in an iterative manner, alternating between online explorations and offline training.
Both simulated and real-world experiments validate the superiority of the whole SORL framework over the state-of-the-art auto-bidding algorithm. Moreover, the ablation study shows that the SER policy can guarantee the safety of explorations even under the guide of the auto-bidding policy with bad performance. The ablation study also verifies the stability of the V-CQL method under different random seeds.

\vspace{-3mm}

	\begin{ack}
		\vspace{-2mm}
	This work is supported by Alibaba Research Intern Program. 
	The authors would like to thank Mingyuan Cheng, Guan Wang, Zongtao Liu, Zhaoqing Peng, Lvyin Niu, Miao Xu, and Tianyu Wang for their valuable feedbacks and insightful discussions.

	\end{ack}
	\section*{Checklist}

	The checklist follows the references.  Please
	read the checklist guidelines carefully for information on how to answer these
	questions.  For each question, change the default \answerTODO{} to \answerYes{},
	\answerNo{}, or \answerNA{}.  You are strongly encouraged to include a {\bf
		justification to your answer}, either by referencing the appropriate section of
	your paper or providing a brief inline description.  For example:
	\begin{itemize}
		\item Did you include the license to the code and datasets? \answerYes{See Section~\ref{gen_inst}.}
		\item Did you include the license to the code and datasets? \answerNo{The code and the data are proprietary.}
		\item Did you include the license to the code and datasets? \answerNA{}
	\end{itemize}
	Please do not modify the questions and only use the provided macros for your
	answers.  Note that the Checklist section does not count towards the page
	limit.  In your paper, please delete this instructions block and only keep the
	Checklist section heading above along with the questions/answers below.

	\begin{enumerate}

		\item For all authors...
		\begin{enumerate}
			\item Do the main claims made in the abstract and introduction accurately reflect the paper's contributions and scope?
			\answerYes{}
			\item Did you describe the limitations of your work?
			\answerYes{See Section \ref{sec:conclusion} and social impact in Appendix \ref{app:social_impact}.}
			\item Did you discuss any potential negative societal impacts of your work?
			\answerYes{We believe there is no negative societal impacts.}
			\item Have you read the ethics review guidelines and ensured that your paper conforms to them?
			\answerYes{ }
		\end{enumerate}

		\item If you are including theoretical results...
		\begin{enumerate}
			\item Did you state the full set of assumptions of all theoretical results?
			\answerYes{See Section \ref{sec:theory} and Appendix \ref{app:assumption}.}
			\item Did you include complete proofs of all theoretical results?
			\answerYes{See Appendix \ref{app:proposition} and Appendix \ref{app:theorem}.}
		\end{enumerate}

		\item If you ran experiments...
		\begin{enumerate}
			\item Did you include the code, data, and instructions needed to reproduce the main experimental results (either in the supplemental material or as a URL)?
			\answerYes{We include the codes for simulated experiments in Section \ref{sec:experiment} in the paper. The codes for real-world experiments are proprietary.}
			\item Did you specify all the training details (e.g., data splits, hyperparameters, how they were chosen)?
			\answerYes{}
			\item Did you report error bars (e.g., with respect to the random seed after running experiments multiple times)?
			\answerYes{}
			\item Did you include the total amount of compute and the type of resources used (e.g., type of GPUs, internal cluster, or cloud provider)?
			\answerYes{}
		\end{enumerate}

		\item If you are using existing assets (e.g., code, data, models) or curating/releasing new assets...
		\begin{enumerate}
			\item If your work uses existing assets, did you cite the creators?
			\answerNA{Our work uses own data, and the data used in the simulation experiments and real-world experiments is released in supplementarymaterial.}
			\item Did you mention the license of the assets?
			\answerNA{}
			\item Did you include any new assets either in the supplemental material or as a URL?
			\answerNA{}
			\item Did you discuss whether and how consent was obtained from people whose data you're using/curating?
			\answerNA{}
			\item Did you discuss whether the data you are using/curating contains personally identifiable information or offensive content?
			\answerNA{}
		\end{enumerate}

		\item If you used crowdsourcing or conducted research with human subjects...
		\begin{enumerate}
			\item Did you include the full text of instructions given to participants and screenshots, if applicable?
			\answerNA{}
			\item Did you describe any potential participant risks, with links to Institutional Review Board (IRB) approvals, if applicable?
			\answerNA{}
			\item Did you include the estimated hourly wage paid to participants and the total amount spent on participant compensation?
			\answerNA{}
		\end{enumerate}

	\end{enumerate}

	\newpage
	\setcounter{footnote}{0}
	\section*{Outline of the Supplementary Material}
	The Appendix is mainly organized as follows\footnote{Figures, tables and formulas in the appendix have all been renumbered. Unless specifically stated "in the manuscript", Fig. x, Table. x and (x) refer to the figure, table and formula with number x in the appendix, respectively. Nonetheless, the numbers of assumptions, propositions and theorems in the appendix are the same as in the manuscript. The references in the appendix refers to the references presented in the manuscript.}.
	\begin{itemize}
		\item Appendix \ref{app:motivations} presents more information on the backgrounds, as well as a further discussion on the motivation of this paper, i.e., IBOO problem in the auto-bidding.
		\begin{itemize}
			\item[-] Appendix \ref{app:ras_vas} presents the detailed structures of the RAS and VAS.
			\item[-] Appendix \ref{app:concept_as}  presents the details of the simulated advertising system experiments shown in Fig. \ref{fig:IBOO} in the manuscript.
			\item [-] Appendix \ref{app:ill_IBOO} presents the figures illustrating the causes and influence and IBOO.
			\item [-] Appendix \ref{app:ras} discusses the importance and universality of the IBOO problem.
			\end{itemize}
		\item Appendix \ref{app:related_work} represents the additional related works.
	\item Appendix \ref{app:assumption}, \ref{app:proposition} and \ref{app:theorem} provides the theoretical supports for the Lipschitz smooth property of the safety function.
\begin{itemize}
	\item[-] Appendix \ref{app:assumption_1} and \ref{app:assumption_2} provide rationality explanations on Assumption \ref{assumption:uniform} and \ref{assumption:bounded_Q}, respectively.
	\item[-] Appendix \ref{app:proposition_1} provide proofs of Proposition  \ref{proposition:form}, respectively.
	\item[-] Appendix \ref{app:theorem_1}, \ref{app:theorem_2} and \ref{app:theorem_3} provide proofs of Theorem \ref{thm:r_smooth}, \ref{thm:Q_smooth} and \ref{thm:lower_bound}, respectively.
\end{itemize}
\item Appendix \ref{app:SORL} presents the detailed interpretations, implementations and derivations of the SORL framework.
		\begin{itemize}
		
		\item[-] Appendix \ref{app:SORL_SER}  presents the derivations and practical implementations on the SER policy.
		\begin{itemize}
			\item [- -] Appendix \ref{app:exploration_derivation} provides the derivations of the SER policy.
			\item[- -] Appendix \ref{app:practical_implementations} provides an implementation of the SER policy in practice.
			\item[- -] Appendix \ref{app:safety_explain} describes more on the safety requirement in  auto-bidding.
		\end{itemize}
	\item[-] Appendix \ref{app:V-CQL} presents the  motivations of the design on the V-CQL, and the  complete implementation on the V-CQL, as well as its  interpretation and relations to previous works.
		\begin{itemize}
		\item [- -] Appendix \ref{app:V-CQL_1} presents the motivations of the design on the V-CQL.
		\item[- -] Appendix \ref{app:V-CQL_2} presents the complete implementation on the V-CQL, as well as its  interpretation and relations to previous works.
	\end{itemize}
		\item[-] Appendix \ref{app:sorl} shows the pseudocode of the SORL framework.
	\end{itemize}
\item Appendix \ref{app:experiments} presents the additional settings and results of the experiments.
\begin{itemize}
	\item [-] Appendix \ref{app:experiment_params} shows the experiment setup.
	\item [-] Appendix \ref{app:additional_results} presents the additional experiment results to validate the effectiveness of the proposed SORL algorithm.
	\begin{itemize}
		\item [- -] Appendix \ref{app:q1} compares our V-CQL method to more popular offline RL algorithms, including, BCQ, CQL($\rho$), etc., which can act as an ablation study.
		\item[- -] Appendix \ref{app:q2} shows the effect of hyper-parameters $\sigma$ and $\lambda$ on the SER policy.
		\item[- -] Appendix \ref{app:q3} shows the detailed experiment data on the A/B test of the SORL framework.
		\item[-] Appendix \ref{app:q4} shows the comparison between our approach and the multi-agent auto-bidding algorithm.
	\end{itemize}
\end{itemize}
\item Appendix \ref{app:social_impact} presents the broader impact of this paper.
	\end{itemize}

	\newpage
		\sectionfont{\huge}
	\section*{Appendix}
	\appendix
	\sectionfont{\Large}
	\setcounter{equation}{0}
	\setcounter{theorem}{0}
\setcounter{assumption}{0}
\setcounter{figure}{0}
\setcounter{table}{0}
		
	\section{Backgrounds and Motivations}
	\label{app:motivations}
	In this section, we provide more information on the application backgrounds, including the detailed structures of the RAS and VAS, the structures of the simulated advertising system. We also discuss the importance and universality of the IBOO problem in auto-bidding, which acts as the motivation of this work.
	\subsection{Detailed Structures of the RAS and VAS}
	\label{app:ras_vas}
	Fig. \ref{fig:ras_vas} shows the detailed structures of the RAS and VAS. Particularly, the VAS is built based on the historical data of advertisers during bidding in stage 2 of the RAS. The VAS can interact with any auto-bidding policies, while the RAS cannot due to safety concerns.
	
		\begin{figure}[h]
		\centering
		\includegraphics[width=13.5cm]{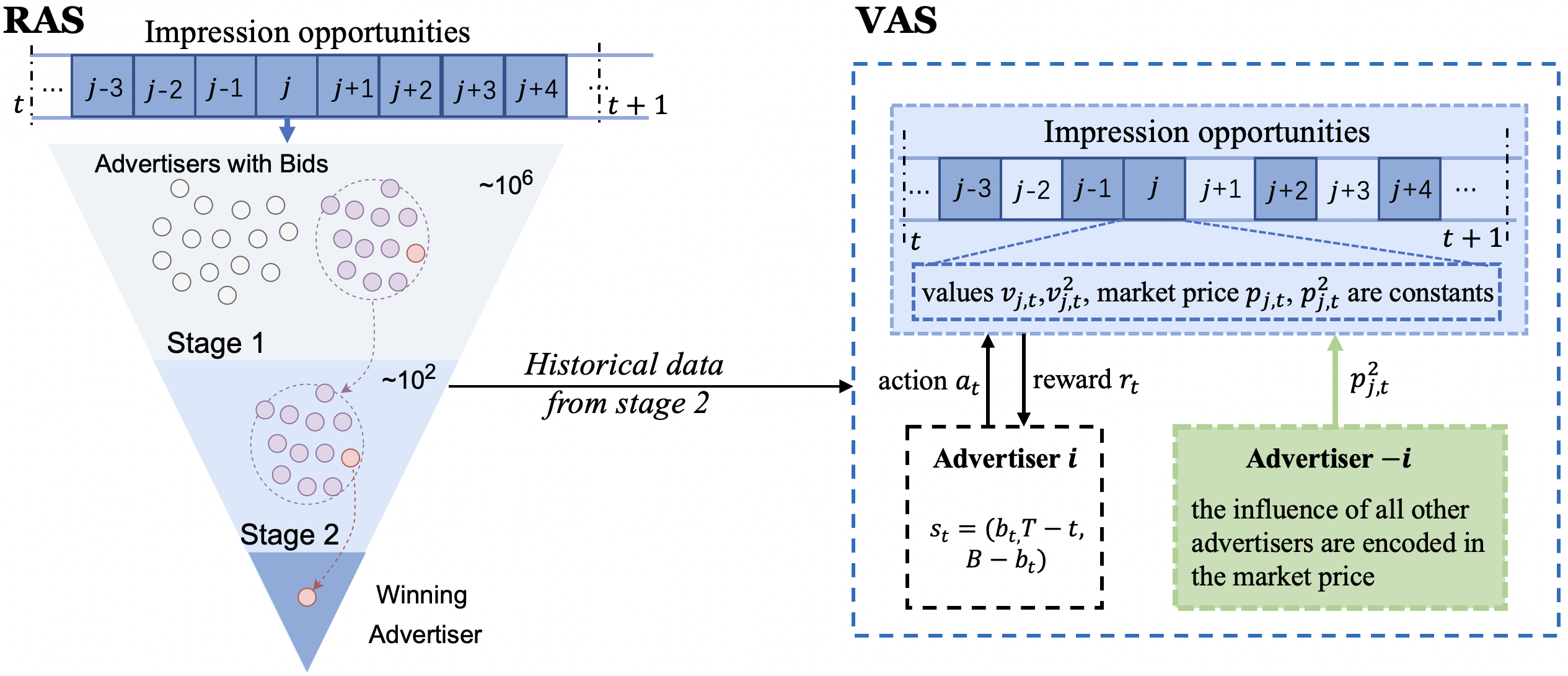}
		\caption{The RAS and VAS detailed structures, where the auction of an impression opportunity is completed in two stages in the RAS, and the advertiser in the VAS competes for each impression opportunity stored in the historical data from stage 2 with the stored market price.}
		\label{fig:ras_vas}
	\end{figure}

	\subsubsection{Structures of RAS} 
	\label{app:ras_vas_ras}
	In the RAS, the auction of each impression opportunity is completed through two stages. Specifically, consider an impression opportunity $j$ coming between time step $t$ and $t+1$ with value $v_{j,t}$. All advertisers give the bids $a_{i,t}, i\in\mathbf{I}$ based on their current states, where $\mathbf{I}$ denotes the set of all advertisers. At stage 1, the RAS roughly calculates the value $v^1_{i,j,t}$ of impression opportunity $j$ with respect to each advertiser $i$, and compare the \emph{effective cost per mile } (eCPM) value of all advertisers. The eCPM value of advertiser $i$ is defined as $a_{i,t}v^1_{i,j,t}$. A market price of impression opportunity $j$ is given by the RAS $p^1_{j,t}$ in stage 1, and advertisers with eCPM larger than $p^1_{j,t}$ can successfully enter the stage 2. Denote the set of advertisers entering stage 2 as $\mathbf{I}_{j}^2$.
	At stage 2, the RAS accurately evaluates the value $v^2_{i',j,t}$ of impression opportunity $j$ with respect to each advertiser $i'\in\mathbf{I}_{j}^2$. The advertiser with the largest eCPM $a_{i',t}v^2_{i',j,t}$ wins the impression $j$. Define the market price in stage 2 as the second highest eCPM among all advertisers, i.e.,
	\begin{align}
		p^2_{j,t}\triangleq\max_{i'\in\mathbf{I}^2_j,i'\neq\arg\max_{k}a_{k,t}v^2_{k,j,t}}a_{i',t}v^2_{i',j,t}.
	\end{align}
The winning advertiser earns the true value $v_{j,t}$ of impression $j$ and pays its market price given by the RAS $p_{j,t}$.

Note that we study the auto-bidding problem from the perspective of a single advertiser. Hence, in the manuscript (especially in Proposition \ref{proposition:form}), we omit the subscript $i$ and $i'$ in the values of impression opportunity $j$ in stage 1, $v_{j,t}^1$, and stage 2, $v^2_{j,t}$, respectively. We also omit the subscript $i$ in action $a_t$ in the manuscript.
	\subsubsection{Structures of VAS}
	\label{app:ras_vas_vas}
	The VAS is built only based on the historical data of advertisers during bidding in stage 2 of the RAS. {This is because the amount of data generated in stage 1 of the RAS is very large (about $10^2$ to $10^4$ impression opportunities coming to stage 1 at every moment, and about $10^6$ advertisers bidding for each impression), and it is computationally infeasible to build the VAS based on data in stage 1 of the RAS.}
	
	There exists some impression opportunities that are not in the VAS, such as impression opportunities $j-2$, $j+1$ and $j+3$ as shown in the VAS in Fig. \ref{fig:ras_vas}. This is because that the advertiser does not enter the stage 2 during the bidding of these impression opportunities in the RAS. In the RL training process, the advertiser at each time step $t$ bids $a_t$ based on the current state $s_t$, and wins the impression opportunity $j$ with value $v_{j,t}$ if $a_tv^2_{j,t}>p_{j,t}^2$, or loses otherwise. Once winning the impression opportunity $j$, the advertiser earns the value $v_{j,t}$ as the and pays the market price $p_{j,t}$. Note that the influence of all other advertisers are encoded in the market price $p^2_{j,t}$. However, the values $v^2_{j,t}$, $v_{j,t}$ and market prices $p_{j,t}^2$, $p_{j,t}$ are all constants during the RL training process.

	\subsection{Simulated Advertising System Experiments}
		\label{app:concept_as}
		We construct a simulated advertising system s-RAS and build a simulated virtual advertising system s-VAS based on it. Specifically, the s-RAS is composed of two consecutive stages, where the auction mechanisms resemble those in the RAS. We consider the bidding process in a day, where the episode is divided into 96 time steps. Thus, the duration between any two adjacent time steps $t$ and $t+1$ is 15 minutes. The number of impression opportunities between time step $t$ and $t+1$ fluctuates from $100$ to $500$.
		Detailed parameters in the s-RAS is shown in Table. \ref{table:config_IBOO}.
		
		\begin{table}[h]
			\caption{The parameters used in the s-RAS.}
			\centering
			\begin{tabular}{ll}
				\toprule
				\makecell[l]{Parameters}&\makecell[c]{Values}  \\   
				\midrule
				\makecell[l]{Number of advertisers}    &\makecell[c]{$100$}\\
			\makecell[l]{Time steps in an episode, $T$} &\makecell[c]{$96$} \\
			\makecell[l]{Minimum number of impression opportunities $N_t$} &\makecell[c]{$100$} \\
				\makecell[l]{Maximum number of impression opportunities $N_t$} &\makecell[c]{$500$} \\
					\makecell[l]{Minimum budget} &\makecell[c]{$31,000$ Yuan} \\
						\makecell[l]{Maximum budget} &\makecell[c]{$36,000$ Yuan} \\
	\makecell[l]{Value of impression opportunities in stage 1, $v_{j,t}^1$} &\makecell[c]{$0\sim 1$} \\			
		\makecell[l]{Value of impression opportunities in stage 2, $v_{j,t}^2$} &\makecell[c]{$0\sim 1$} \\			
			\makecell[l]{Minimum bidding price, $A_\text{min}$} &\makecell[c]{$0$ Yuan} \\
				\makecell[l]{Maximum bidding price, $A_\text{max}$} &\makecell[c]{$1,000$ Yuan} \\					
					\makecell[l]{Maximum value of impression opportunity, $v_M$} &\makecell[c]{$1$} \\
				\makecell[l]{Maximum market price, $p_M$} &\makecell[c]{$1,000$ Yuan} \\					
				\bottomrule
			\end{tabular}
			\label{table:config_IBOO}
		\end{table}
	
We adopt the standard RL algorithm, DDPG \cite{DDPG}, to train the auto-bidding policy of an advertiser in the s-RAS while keeping the policies of all other $99$ advertisers fixed. Obviously, all other advertisers are viewed as parts of the environment with respect to the training advertiser. Fixing other advertisers' policies makes the environment stationary.
The hyper-parameters used in the DDPG are shown in Table. \ref{table:config_IBOO_DDPG}. In addition, the s-VAS is built based on the historical data of the advertiser when bidding in stage 2 of the s-RAS, including the indexes of impression opportunities and the corresponding values and market prices. Moreover, an offline dataset is construct by collecting data directly from the s-RAS. We train the auto-bidding policy with the s-VAS and the offline dataset using the DDPG, where the hyper-parameters are the same as those in Table. \ref{table:config_IBOO_DDPG}. The differences in impression opportunities and market prices between the s-RAS and s-VAS are shown in Fig. \ref{fig:IBOO_pv_changing} in the manuscript and Fig. \ref{fig:IBOO_market_price_changing} in the manuscript, respectively. The RL rewards of these three settings are shown in Fig. \ref{fig:IBOO_rewards} in the manuscript.

	\begin{table}[h]
	\caption{The hyper-parameters of DDPG when training with the s-RAS, s-VAS and the offline dataset.}
	\centering
	\begin{tabular}{ll}
		\toprule
		\makecell[l]{Hyper-parameters}&\makecell[c]{Values}  \\   
		\midrule
		\makecell[l]{Optimizer}    &\makecell[c]{Adam}\\
		\makecell[l]{Learning rate for critic network}    &\makecell[c]{$1\times 10^{-4}$}\\
			\makecell[l]{Learning rate for actor network}    &\makecell[c]{$1\times 10^{-4}$}\\
				\makecell[l]{Soft updated rate }    &\makecell[c]{$0.01$}\\
					\makecell[l]{Buffer size }    &\makecell[c]{$1000$}\\
						\makecell[l]{Sampling size }    &\makecell[c]{$200$}\\
	\makecell[l]{Discounted factor $\gamma$ }    &\makecell[c]{$0.99$}\\
	\makecell[l]{Random seeds }    &\makecell[c]{$1\sim 16$}\\
		\makecell[l]{Exploration actions}    &\makecell[c]{Gaussian noise with variance $0.01$}\\
		\bottomrule
	\end{tabular}
	\label{table:config_IBOO_DDPG}
\end{table}
	
	\subsection{Illustrations of IBOO}
	\label{app:ill_IBOO}
	Fig. \ref{app:ill_IBOO} illustrates the dominated gaps and influence of the IBOO.
	\begin{figure}[t]
		\centering
		\subfigure[Impression opportunities.]
		{\label{fig:IBOO_pv_changing}
			\includegraphics[scale=0.11]{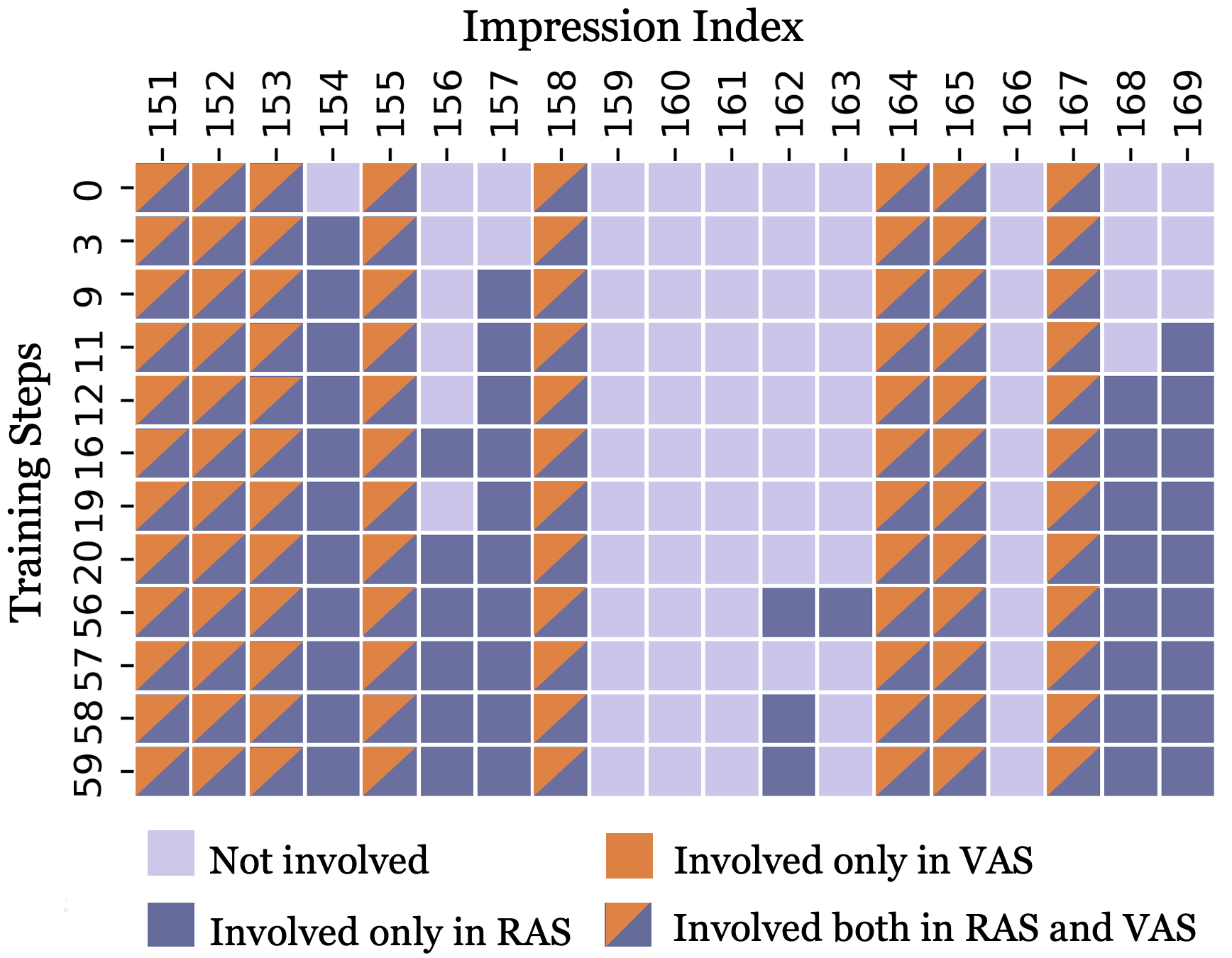}}
		\subfigure[Market prices.]
		{
			\label{fig:IBOO_market_price_changing}
			\includegraphics[scale=0.11]{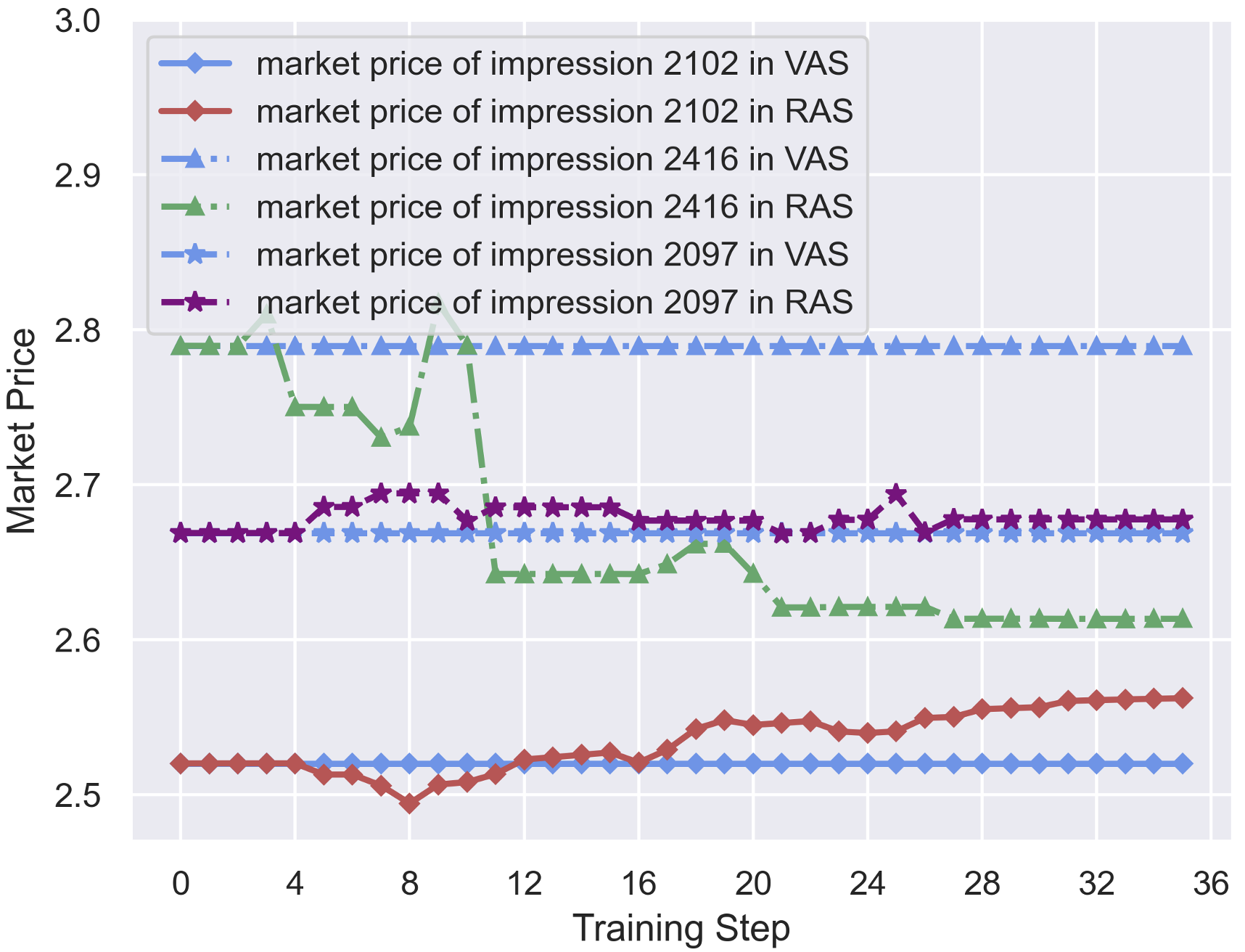}
		}
		\subfigure[RL rewards {evaluated in RAS}.]
		{
			\label{fig:IBOO_rewards}
			\includegraphics[scale=0.11]{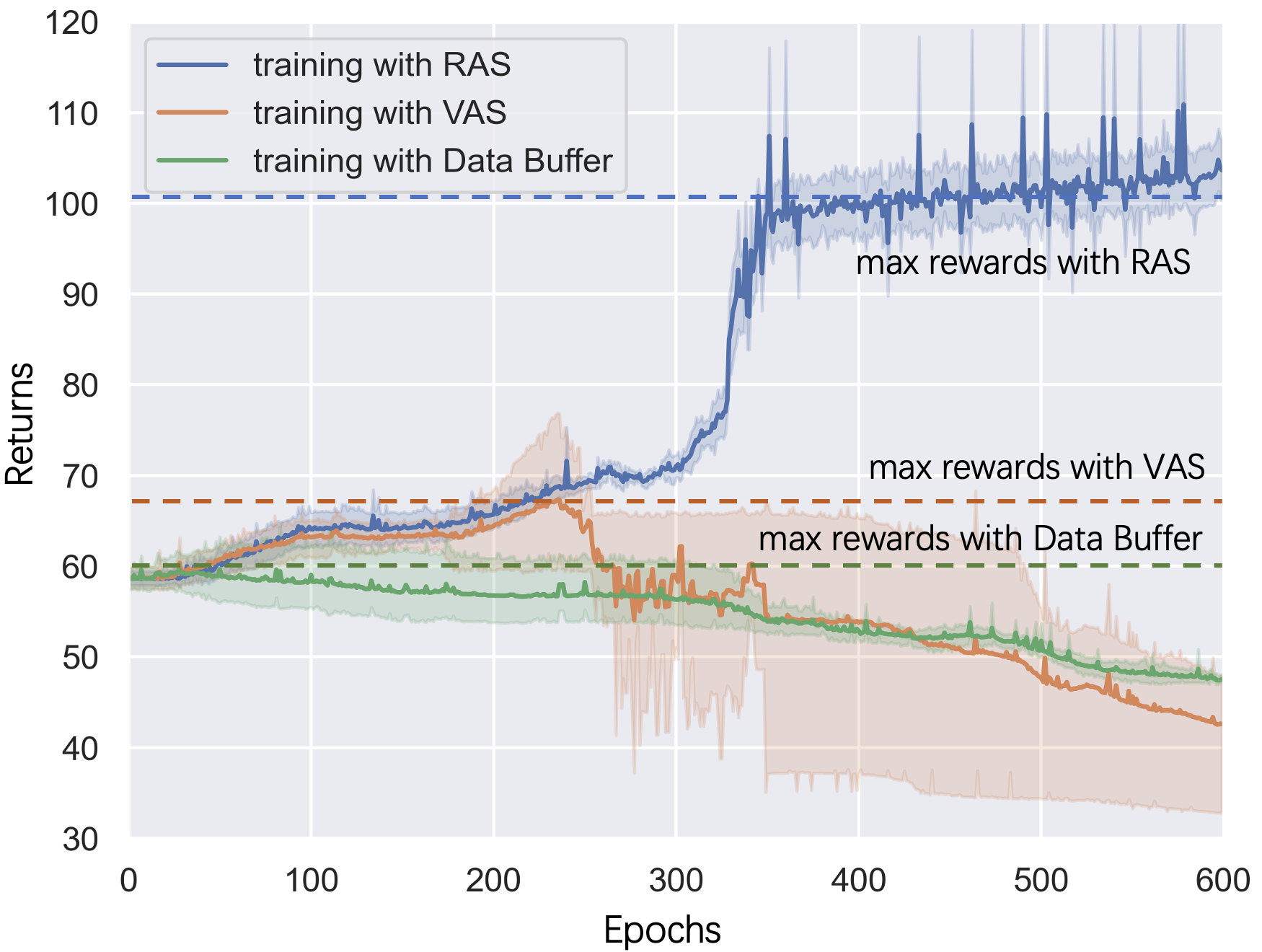}
		}
		\caption{Dominated gaps and IBOO influence. (a) shows that the impression opportunities within a certain period involving in the bidding of an advertiser remains the same in the VAS, but changes in {stage 2 of the RAS} during RL training process. (b) shows that the market prices remain the same in the VAS, but can rise, decrease or fluctuate around during RL training. (c) shows that auto-bidding policies {training with the RAS can achieve higher rewards than those training with the VAS and a fixed data buffer using traditional RL methods.}}
		\label{fig:IBOO}
	\end{figure}
		\subsection{Importance and Universality of IBOO Problem}
	\label{app:ras}
		As stated in the manuscript, it was previously believed that training auto-bidding policies directly in the RAS is nearly impossible due to safety concerns.
	State-of-the-art auto-bidding policies training in the VAS face IBOO problems which can largely degrade their performance in the RAS. Solving the IBOO problem in the auto-bidding acts as the motivation of our work. Here, we further discuss the importance and universality of this motivation.
	
	\textbf{Importance.} Online advertising business has become one of the main profit models for many companies, such as Google, Amazon, Alibaba, etc. In 2021, Google's online advertising revenue accounts for 82\% of the total, and online advertising revenue in Alibaba accounts for over 90\% of the total. At the same time, online advertising business also offers clients, acting as advertisers, a good chance to increase the ROI. Hence, online advertising business plays an important role for both companies and advertisers in the era of Internet. Recently, auto-bidding technique has become one of the most important tools for advertisers to lift up their ROI. However, state-of-the-art auto-bidding policies leveraging RL algorithms suffer from the IBOO problem. The IBOO can be significant when the gap between the optimal auto-bidding policy and the auto-bidding policy used for collecting data to construct the VAS is large. This is because many impression opportunities and corresponding information on values and market prices in the \emph{optimal VAS}\footnote{For convenience, we name the VAS built based on the data collected by the optimal auto-bidding policy as the optimal VAS.} are missing in the VAS. Hence, the auto-bidding policy cannot know how to behave on these unseen impression opportunities, and the improvement of the auto-bidding policy can be limited. 
	
	\textbf{Universality.}
	The IBOO problem does not exist only in the realm of auto-bidding. Actually, it exists in many other fields such as robotics \cite{sim2real_robot}, thermal power generating \cite{sim2real:thermal_power}, and even computer visions \cite{sim2real_cv}, where the real-world environment cannot be accessed during RL training process and a virtual environment is needed. 
	In these fields, the IBOO problem is usually known as the \emph{sim2real} problem. Although many algorithms have been proposed to mitigate the IBOO (or sim2real) problem, it still remains a major challenge in the RL applications.

	\section{Related Work}
	\label{app:related_work}

In addition, to avoid IBOO, one may consider training auto-bidding policies with traditional RL algorithms based on the data collected by {some} safe {policies} directly from the RAS. However, this {approach} will {suffer from extrapolation errors} that can seriously degrade the {policy's performance in RAS \cite{BCQ, offline_review}.} {As shown in Fig. \ref{fig:IBOO_rewards}, the expected cumulative rewards of traditional RL method \cite{DDPG} training with a fixed data buffer are lower than those of traditional RL method training with RAS and VAS.}
Though we can leverage offline RL techniques \cite{BCQ, CQL, BCQ_1, BCQ_2} to mitigate {this challenge}, 
{we cannot guarantee that the collected data contains sufficient transitions from high-reward regions \cite{offline_review}. This will strain the capacity of the offline RL algorithms to train near-optimal auto-bidding policies. Thus, extra data collections with different behavior policies (presumably better behavior policies) from the RAS are required. The policies trained by the offline RL methods cannot be directly used as behavior policies for data collections in the RAS, since}
the performance variance of the trained policies can be large \cite{parameter}.
Off policy evaluations (OPE{s}) {have been recently studied for selecting policies with good performance without applying them to real-world environments} \cite{OPE, parameter_selection}. However, {existing} OPE{s} in auto-bidding {are usually conducted in the VAS and can be inaccurate (see the last paragraph in this section).}
{Therefore}, we may still have little confidence to widely apply the policies {trained by offline RL methods} to advertisers for further online explorations in the RAS, even though they are proven to perform well by OPE.
Note that there exist {some} safe online RL methods for {safely} exploring in the environments  \cite{CSC,safeRL:query,learning-based,liyapnouv,icrl,irl}.
{However, they are either developed based on the constraints that are not suitable for the auto-bidding problem or designed for systems with specific assumptions.}
Recently, with the development of offline RL {methods}, many algorithms for efficient online explorations on the premise of having an offline dataset  \cite{AWAC, O2O} have emerged. However, they often focus on the efficiency of RL training process rather than the safety of explorations. 
	
	\textbf{Safe Online RL.}
	\cite{6} realize safe explorations by adding a safety layer at the end of the actor network. However, the safety layer needs to be trained by a prior dataset and can be inaccurate at states outside the dataset. \cite{7} uses offline safety tests to examine the safety of the latest policy and directly applies it to explore online if it passes the tests. However, as we stated in Appendix G.2, there is no such reliable offline safety tests in auto-bidding. Besides, \cite{liyapnouv} realizes safe explorations by gradually increasing the attraction regions of the initial safe policy. However, it leverages the assumption that the environment is a linear model, which is not suitable for auto-bidding. As for this paper, based on the proved Lipschitz property of Q functions, we design the exploration policy by offsetting the actions to the promising directions relative to an initial safe policy.

	\textbf{Extrapolation Error.}
	Extrapolation error means the misestimation of the states and actions outside the fixed dataset. A typical misestimation happens to the Q function in the standard RL algorithm. The fixed dataset cannot contain all the data from the environment, since the amount of all the data can usually be infinite. Hence, the trained Q function can only be accurate at the states and actions inside the dataset and can be inaccurate (usually overestimated) at those outside the dataset. This will make the actor network learn actions that extremely deviate from the behavioral actions and often bias towards bad actions. Hence, the policy performance can be seriously degraded. 
	Offline RL algorithms usually address this challenge {in three ways}, including \emph{policy constraint} methods, 
	where explicit or implicit constraints are directly imposed to policies,
	such as BCQ \cite{BCQ}, BEAR \cite{BCQ_1}, and \emph{conservative regularization} methods, 
	where penalties for out-of-distribution (OOD) actions are imposed to the Q function, 
	such as CQL \cite{CQL}, BRAC \cite{BRAC}, as well as \emph{modifications of imitation learning} method \cite{onestep} such as ABM \cite{BCQ_2}, CRR \cite{CRR}, BAIL\cite{BAIL}.
		
		\textbf{OPE in Auto-bidding.}
		Generally, the OPE used in auto-bidding is evaluate the auto-bidding policies in a VAS which is built based on the historical data of hundreds of advertisers. In the VAS, as we can know all the impression opportunities as well as their values and market prices in advance, we can calculate the optimal bids using linear programming \cite{RL:USCB}. Hence, the optimal accumulated rewards can be obtained. We define the ratio between the accumulated reward of the evaluated policy and the optimal accumulated rewards as $R/R^*$, which acts an important metric in the OPE of auto-bidding. The range of $R/R^*$ is $[0,1]$. The closer the value of $R/R^*$ to 1, the better the performance of the evaluated auto-bidding policy. However, due to the IBOO, this common OPE method is not very accurate. Specifically, a low value of $R/R^*$ (below $0.7$) can indicate a poor performance of the evaluated auto-bidding policy, while a large value of $R/R^*$ (above $0.8$) does not indicate that the evaluated auto-bidding can certainly perform well in the RAS. Nonetheless, auto-bidding policies with higher $R/R^*$ are more likely to perform well in the RAS than those with lower $R/R^*$.
	
	\section{Rationality Analysis of Assumptions}
	\label{app:assumption}
	\subsection{Rationality of Assumption 1}
		\label{app:assumption_1}
	\begin{assumption}[\textnormal{Bounded Impression Distributions}]
		Between time step $t$ and $t+1$, we assume the numbers of winning impressions with action $a_t$ in the first stage $n_{t,1}$ and the second stage $n_{t,2}$ can both be bounded by linear functions, i.e., $n_{t,1}\le k_1a_t, n_{t,2}\le k_2a_t$, where $k_1,k_2>0$ 
	are constants.
	\end{assumption}
In a stable RAS, the amount of increased (or decreased) winning impression opportunities for an advertiser when increasing (or reducing) the bids $a_t$ within any time step $t$ and $t+1$ in both stage 1 and 2 will not change dramatically. Otherwise advertisers can largely increase the number of winning opportunities by slightly raising the bids, which can make the RAS unstable. Hence, there exist linear functions that can bound the changes in the amount of winning impression opportunities in both stage 1 and 2, where the slopes $k_1$ and $k_2$ have limited values (that usually are not very large). Fig. \ref{fig:assumption_1} illustrates this assumption with the data generated in an bidding episode in the s-RAS.

\begin{figure}[h]
	\centering
	\subfigure[$\frac{p_{j,t}^1}{v_{j,t}^1}$ at stage 1.]
	{\label{fig:assumption_1_1}
		\includegraphics[scale=0.07]{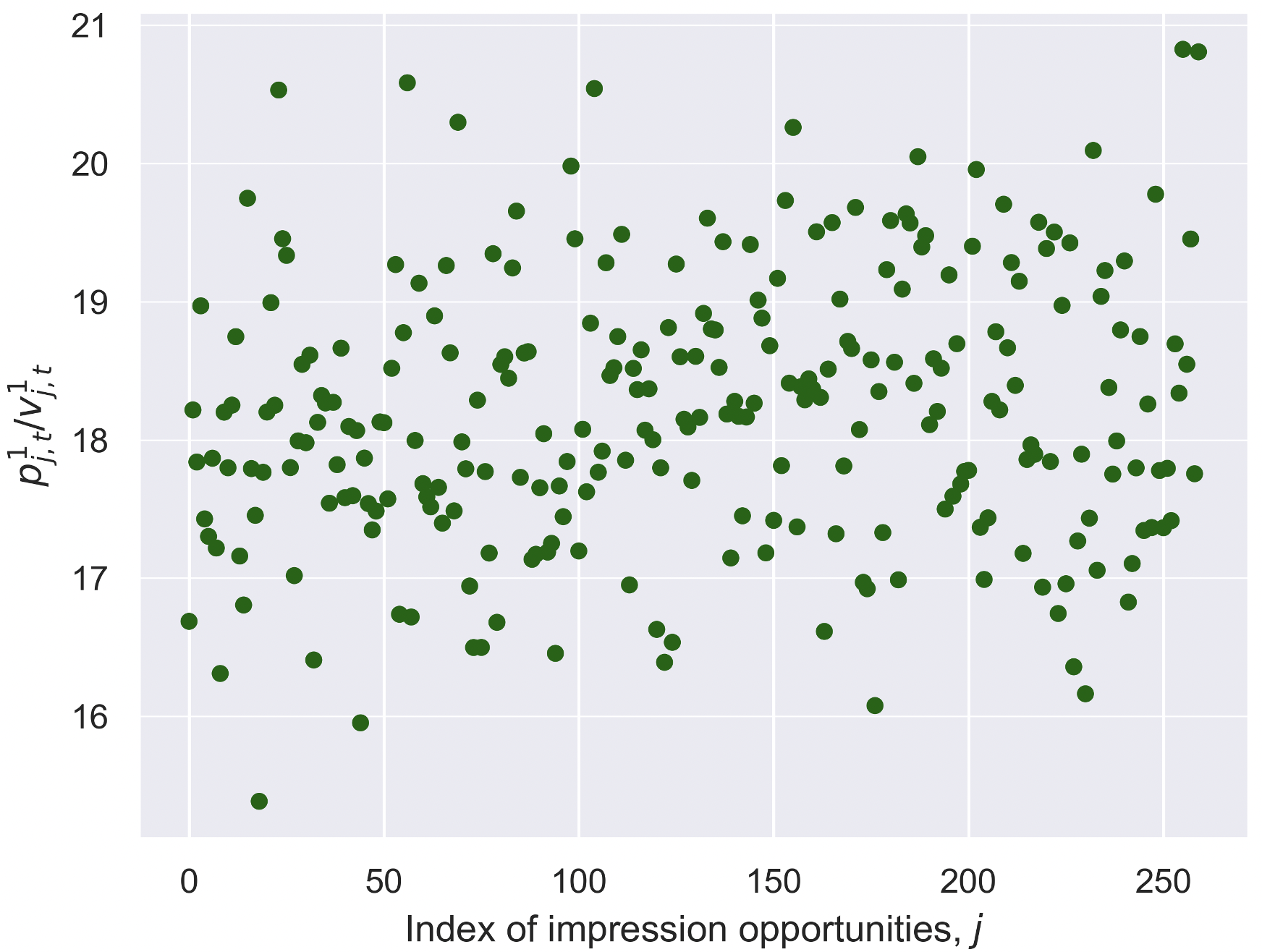}}
	\subfigure[$\frac{p_{j,t}^2}{v_{j,t}^2}$ at stage 2.]
	{
		\label{fig:assumption_1_2}
		\includegraphics[scale=0.07]{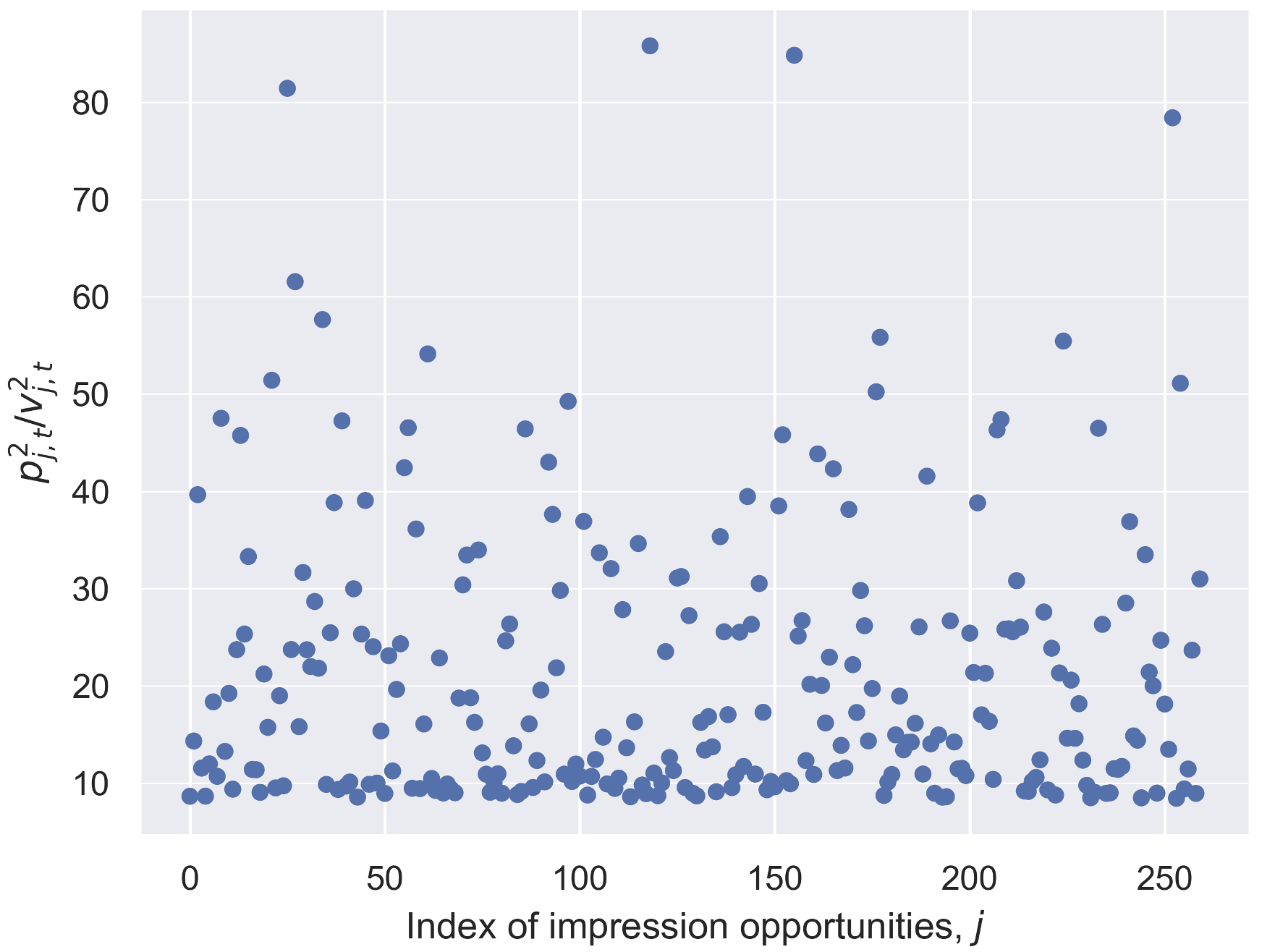}
	}
	\subfigure[$n_{t,1}$ and $n_{t,2}$ as functions of $a_t$]
	{
		\label{fig:assumption_1_3}
		\includegraphics[scale=0.07]{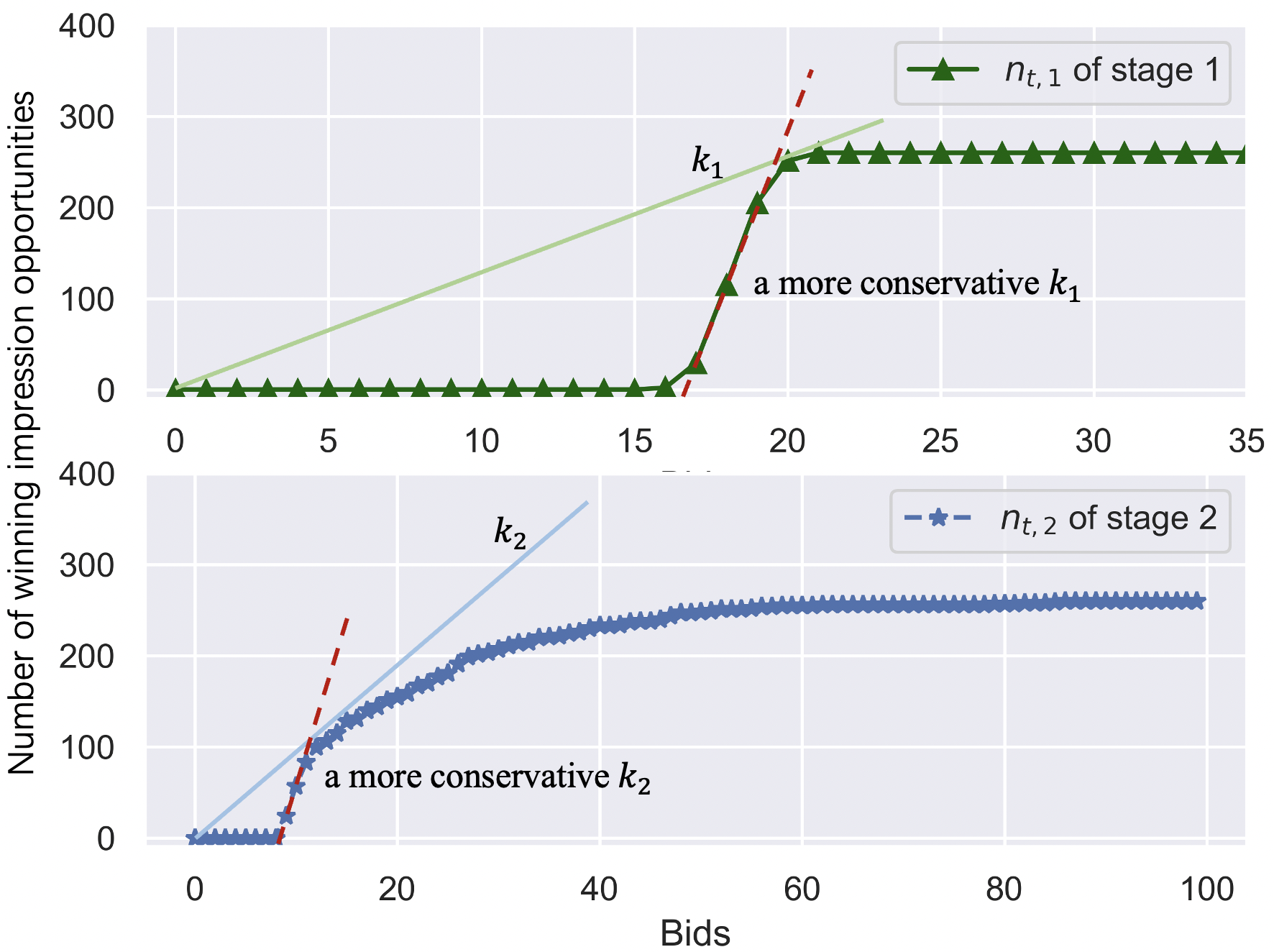}
	}

	\subfigure[$\frac{p_{j,t}^1}{v_{j,t}^1}$ at stage 1 in a day.]
{\label{fig:assumption_1_4}
	\includegraphics[scale=0.07]{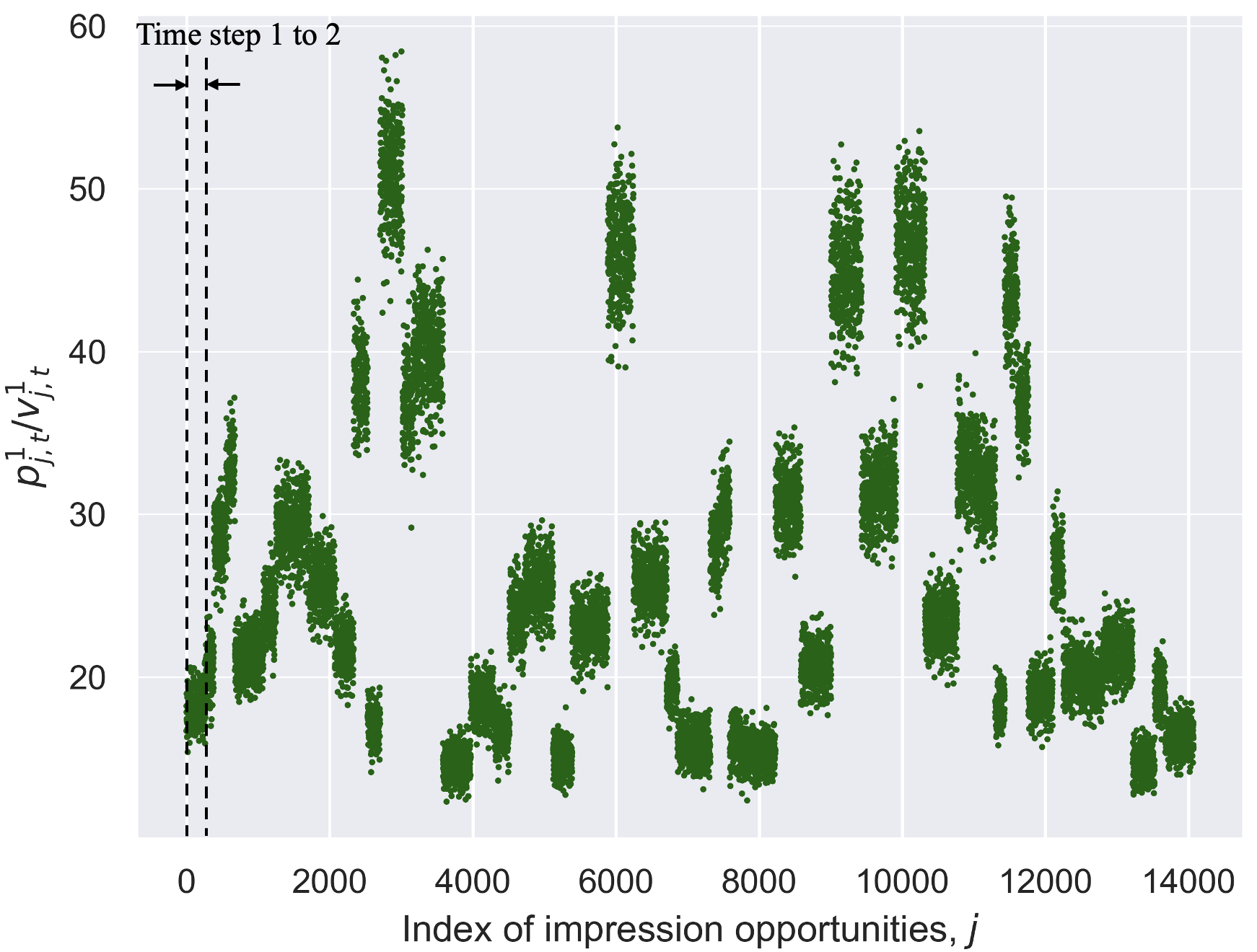}}
\subfigure[$\frac{p_{j,t}^2}{v_{j,t}^2}$ at stage 2 in a day.]
{
	\label{fig:assumption_1_5}
	\includegraphics[scale=0.07]{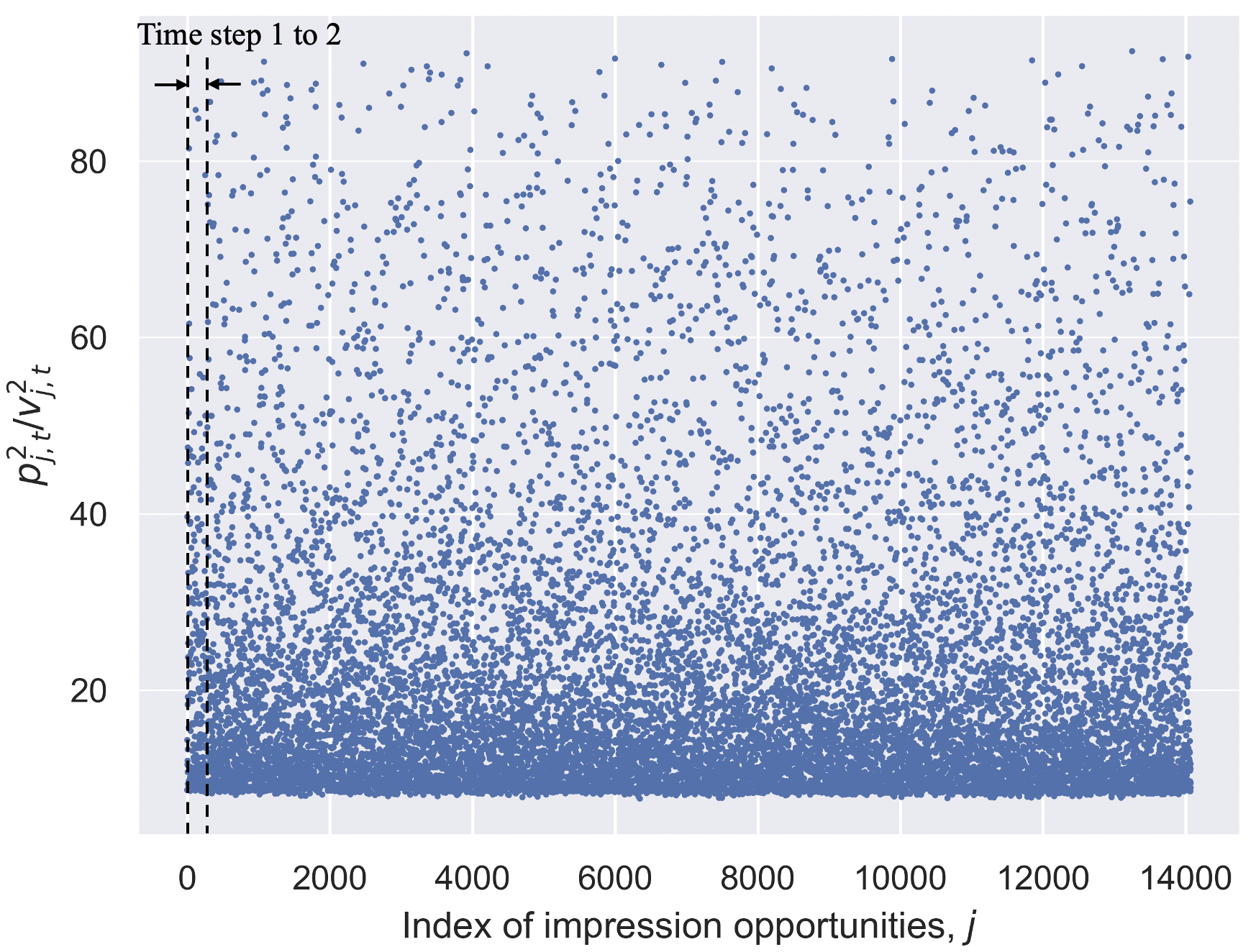}
}
\subfigure[$\sum_{t=1}^Tn_{t,1}$ and .$\sum_{t=1}^Tn_{t,2}$]
{
	\label{fig:assumption_1_6}
	\includegraphics[scale=0.07]{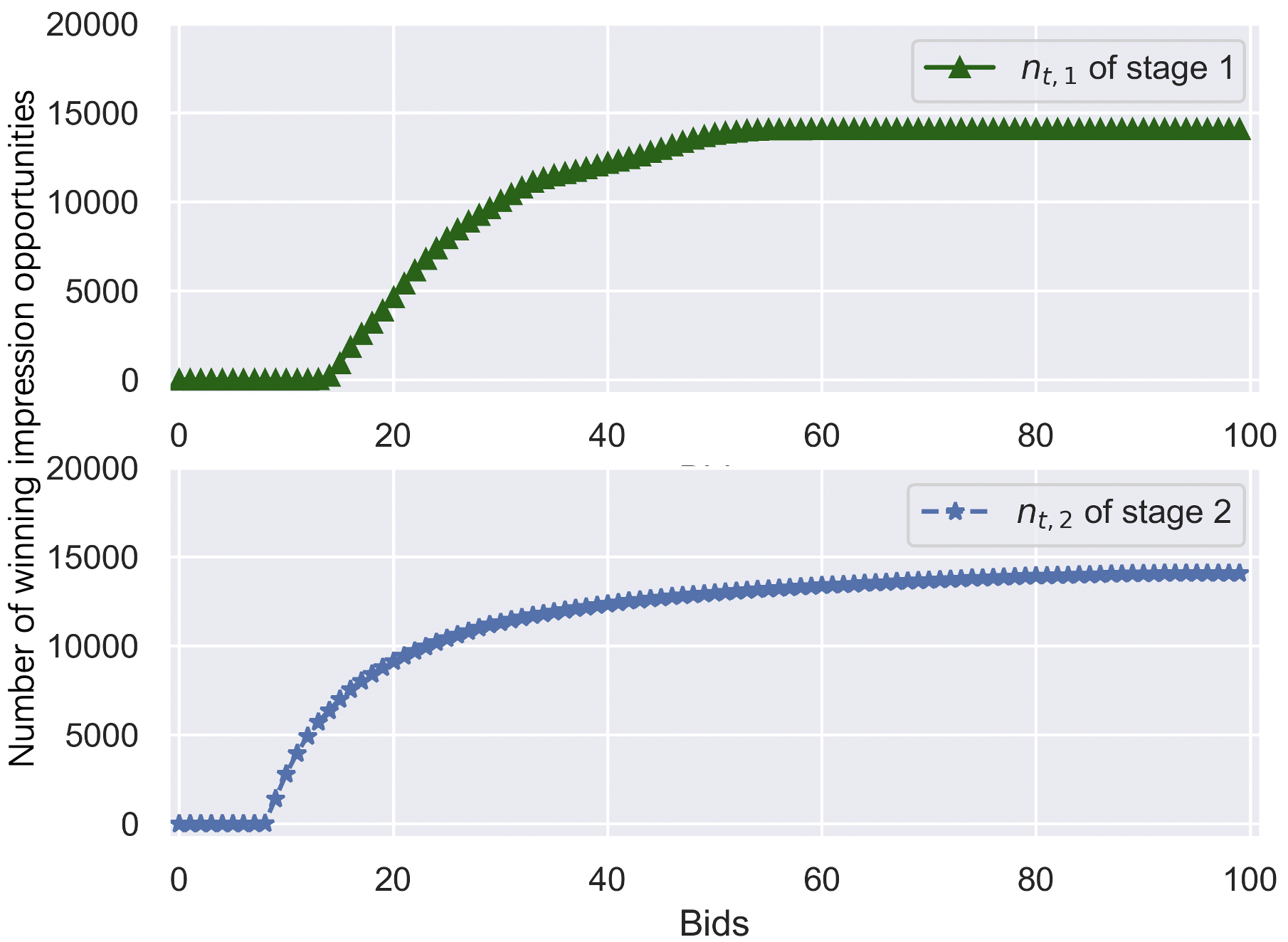}
}
	\caption{The fractions between the market price and value of impression opportunities as well as the number of winning impression opportunities changing with the bids $a_t$. (a) to (c) present the data between time step $1$ and $2$, and (d) to (e) displays the data of the whole episode (a day).}
	\label{fig:assumption_1}
\end{figure}

	\subsection{Rationality of Assumption 2}
		\label{app:assumption_2}
	\begin{assumption}[\textnormal{Bounded Partial Derivations of $Q^\mu(s_t,a_t)$}]
		We assume that the partial derivation of $Q^\mu(s_t,a_t)$ with respect to $s_t(1)$ and $s_t(3)$ is bounded, i.e., $\big\vert\frac{\partial Q^\mu(s_t,a_t)}{\partial s_t(1)}\big\vert\le k_3$ and $\big\vert\frac{\partial Q^\mu(s_t,a_t)}{\partial s_t(3)}\big\vert\le k_4$, where $k_3, k_4>0$ are constants.
	\end{assumption}
Recall that the value of the Q function at state $s_t$ and action $a_t$ represents the total value of winning impression opportunities starting from state $s_t$, bidding with $a_t$ and following policy $\mu$ afterwards. In a stable RAS, this total value of winning impression opportunities will not increase (or drop) dramatically when the advertiser slightly increases (or reduces) the budget. Similarly, $Q^\mu(s_t,a_t)$ will not extremely change if the advertiser spends a little more budget before time step $t$. Hence, the absolute values of the partial derivatives of $Q^\mu(s_t,a_t)$ with respect to the budget left $s_t(1)$ and the consumed budget $s_t(2)$ have limited values (that are usually not very large). We denote the upper bounds of $\big\vert\frac{\partial Q^\mu(s_t,a_t)}{\partial s_t(1)}\big\vert$ and $\big\vert\frac{\partial Q^\mu(s_t,a_t)}{\partial s_t(3)}\big\vert$ as $k_3$ and $k_4$, respectively.

	\section{Proofs of Propositions}
	\label{app:proposition}
	\setcounter{proposition}{0}


	\subsection{Proofs of Proposition 1}
		\label{app:proposition_1}
			\setcounter{proposition}{0}
			\begin{proposition}[\textnormal{Analytical expressions of $R$, $C$ and $\mathbb{P}$}] 
			Based on the characteristic of the two-stage cascaded auction in the RAS, we can formulate the reward function $R$ as $r_t(s_t,a_t)=\sum_{j}\mathbbm{1}\{a_tv^1_{j,t}\ge p^1_{j,t},a_tv^2_{j,t}\ge p^2_{j,t}\}v_{j,t}$, the constraint function $C$ as $c_t(s_t,a_t)=\sum_{j}\mathbbm{1}\{a_tv^1_{j,t}\ge p^1_{j,t},a_tv^2_{j,t}\ge p^2_{j,t}\}p_{j,t}$, and the state transition rule $\mathbb{P}$ as $s_{t+1}=s_{t}+[\triangle s_t(1),\triangle s_t(2),\triangle s_t(3)]$, where $\triangle s_t(1)=-\triangle s_t(3)=-\sum_{j}\mathbbm{1}\{a_tv^1_{j,t}\ge p^1_{j,t},a_tv^2_{j,t}\ge p^2_{j,t}\}p_{j,t}$ and $\triangle s_t(2)=-1$.
			Note that $p^1_{j,t}$ and $v^1_{j,t}$ denote the market price and rough value of impression $j$ in stage 1, and $p^2_{j,t}$  and $v^2_{j,t}$ denote the market price and accurate value of impression $j$ in stage 2.
		\end{proposition}
		\begin{proof}
As stated in Appendix A.1, the auction is completed in two cascaded stages. The condition to win the impression opportunity $j$ for an advertiser is that

\begin{itemize}
	\item successfully passing the stage 1, i.e., $a_tv^1_{j,t}\ge p^1_{j,t}$, and
	\item bidding the highest in stage 2, i.e., $a_tv^2_{j,t}\ge p^2_{j,t}$
\end{itemize} 
hold at the same time. Hence, the reward function which is the total value of winning impression opportunities between time step $t$ and $t+1$ can be expressed as 
	\begin{align}
	r_t(s_t,a_t)=\sum_{j}\mathbbm{1}\bigg\{a_tv^1_{j,t}\ge p^1_{j,t},a_tv^2_{j,t}\ge p^2_{j,t}\bigg\}v_{j,t},
\end{align}
and the cost function can be expressed as 
	\begin{align}
	c_t(s_t,a_t)=\sum_{j}\mathbbm{1}\bigg\{a_tv^1_{j,t}\ge p^1_{j,t},a_tv^2_{j,t}\ge p^2_{j,t}\bigg\}p_{j,t}.
\end{align}
The amount of the consumed budget between time step $t$ and $t+1$ can be expressed as
			\begin{align}
				b_t-b_{t+1}=-\triangle s_t(1)=\triangle s_t(3)=\sum_{j}\mathbbm{1}\{a_tv^1_{j,t}\ge p^1_{j,t},a_tv^2_{j,t}\ge p^2_{j,t}\}p_{j,t}.
			\end{align} 
			Besides, $\triangle s_t(2)=T-(t+1)-T+t=-1$.
		\end{proof}

		\section{Proofs of Theorems}
	\label{app:theorem}
	\subsection{Proof of Theorem 1}
		\label{app:theorem_1}
	\begin{theorem}[\textnormal{Lipschitz Smooth of $r_t(s_t,a_t)$}]
		Under Assumption \ref{assumption:uniform}, the reward function $r_t(s_t,a_t)$ is $L_r$-Lipschitz smooth with respect to actions $a_t$ at any given state $s_t$, where $L_r=(k_1+k_2)v_M$.
	\end{theorem}
	\begin{proof}
		Recall that the reward function $r_t(s_t,a_t)$ can be expressed as 
		\begin{align}
			r_t(s_t,a_t)=\sum_{j=1}^{N_t}\mathbbm{1}\big\{a_tv_{j,t}^1\ge p_{j,t}^1,a_tv_{j,t}^2\ge p_{j,t}^2 \big\}v_{j,t}=\sum_{j=1}^{N_t}\mathbbm{1}\bigg\{a_t\ge \frac{p_{j,t}^1}{v_{j,t}^1},a_t\ge \frac{p_{j,t}^2}{v_{j,t}^2}\bigg\}v_{j,t}.
		\end{align}	
		Hence, $\forall s_t\in\mathcal{S}$ and $\forall a_1,a_2\in\mathcal{A}, a_1\neq a_2$, we have 
		\begin{align}
			\label{equ:prim}
			{|r_t(s_t,a_1)-r_t(s_t,a_2)|}&={\bigg\vert\sum_{j=1}^{N_t}\bigg[\mathbbm{1}\bigg\{a_1\ge \frac{p_{j,t}^1}{v_{j,t}^1},a_1\ge \frac{p_{j,t}^2}{v_{j,t}^2}\bigg\}-\mathbbm{1}\bigg\{a_2\ge \frac{p_{j,t}^1}{v_{j,t}^1},a_2\ge \frac{p_{j,t}^2}{v_{j,t}^2}\bigg\}\bigg] v_{j,t}\bigg\vert}.
		\end{align}
		Without loss of generality, we let $a_1>a_2$. Note that the advertiser can win any impression opportunity $j$ with bid price $a_1$ if it can win this impression opportunity with bid price $a_2$, which means
		\begin{align}
			\mathbbm{1}\bigg\{a_1\ge \frac{p_{j,t}^1}{v_{j,t}^1},a_1\ge \frac{p_{j,t}^2}{v_{j,t}^2}\bigg\}\ge \mathbbm{1}\bigg\{a_2\ge \frac{p_{j,t}^1}{v_{j,t}^1},a_2\ge \frac{p_{j,t}^2}{v_{j,t}^2}\bigg\}.
		\end{align}
		Thus, we can  {drop the absolute value sign in \eqref{equ:prim} and} obtain
		\begin{align}
			\label{equ:pre_ranking_ranking_decomposition}
			{|r_t(s_t,a_1)-r_t(s_t,a_2)|}&=\sum_{j=1}^{N_t}\bigg[\mathbbm{1}\bigg\{a_1\ge \frac{p_{j,t}^1}{v_{j,t}^1},a_1\ge \frac{p_{j,t}^2}{v_{j,t}^2}\bigg\}-\mathbbm{1}\bigg\{a_2\ge \frac{p_{j,t}^1}{v_{j,t}^1},a_2\ge \frac{p_{j,t}^2}{v_{j,t}^2}\bigg\}\bigg] v_{j,t}\notag\\
			&\le v_M \sum_{j=1}^{N_t}\bigg[\mathbbm{1}\bigg\{a_1\ge \frac{p_{j,t}^1}{v_{j,t}^1},a_1\ge \frac{p_{j,t}^2}{v_{j,t}^2}\bigg\}-\mathbbm{1}\bigg\{a_2\ge \frac{p_{j,t}^1}{v_{j,t}^1},a_2\ge \frac{p_{j,t}^2}{v_{j,t}^2}\bigg\}\bigg]\notag\\
			&=v_M\sum_{j=1}^{N_t}\bigg[\mathbbm{1}\bigg\{a_1\ge \frac{p_{j,t}^1}{v_{j,t}^1},a_1\ge \frac{p_{j,t}^2}{v_{j,t}^2}, \bigg(a_2< \frac{p_{j,t}^1}{v_{j,t}^1},\text{or}\,\frac{p_{j,t}^1}{v_{j,t}^1}\le a_2< \frac{p_{j,t}^2}{v_{j,t}^2}\bigg)\bigg\}\bigg]\notag\\
			&=v_M\sum_{j=1}^{N_t}\bigg[\mathbbm{1}\bigg\{a_1\ge \frac{p_{j,t}^1}{v_{j,t}^1},a_1\ge \frac{p_{j,t}^2}{v_{j,t}^2},a_2< \frac{p_{j,t}^1}{v_{j,t}^1}\bigg\} +\notag\\
			&\quad\quad\qquad\quad\; \mathbbm{1}\bigg\{a_1,a_2\ge \frac{p_{j,t}^1}{v_{j,t}^1},a_1\ge \frac{p_{j,t}^2}{v_{j,t}^2},a_2< \frac{p_{j,t}^2}{v_{j,t}^2}\bigg\}\bigg]
		\end{align}
		Note that \eqref{equ:pre_ranking_ranking_decomposition} use the fact that the  {additional} impression opportunities won by bid $a_1$ compared to bid $a_2$, i.e., $\sum_{j=1}^{N_t}\bigg[\mathbbm{1}\bigg\{a_1\ge \frac{p_{j,t}^1}{v_{j,t}^1},a_1\ge \frac{p_{j,t}^2}{v_{j,t}^2}\bigg\}-\mathbbm{1}\bigg\{a_2\ge \frac{p_{j,t}^1}{v_{j,t}^1},a_2\ge \frac{p_{j,t}^2}{v_{j,t}^2}\bigg\}\bigg]$, can be divided into two parts: 
		\begin{itemize}
			\item the first part are the impression opportunities that can be won with bid $a_1$ but cannot be won with bid $a_2$ even in stage 1, i.e., $\sum_{j=1}^{N_t}\mathbbm{1}\bigg\{a_1\ge \frac{p_{j,t}^1}{v_{j,t}^1},a_1\ge \frac{p_{j,t}^2}{v_{j,t}^2},a_2< \frac{p_{j,t}^1}{v_{j,t}^1}\bigg\}$;
			\item the second part are the impression opportunities that can be won in the stage 1 with both bids $a_1$ and $a_2$, but can only be won in stage 2 with bid $a_1$, not $a_2$, i.e., $\sum_{j=1}^{N_t}\mathbbm{1}\bigg\{a_1,a_2\ge \frac{p_{j,t}^1}{v_{j,t}^1},a_1\ge \frac{p_{j,t}^2}{v_{j,t}^2},a_2< \frac{p_{j,t}^2}{v_{j,t}^2}\bigg\}$.
		\end{itemize}
	
		\begin{figure}[h]
		\centering
		\subfigure[Bidding with $a_1$ in stage 1.]
		{\label{fig:a_1_pre_ranking}
			\includegraphics[scale=0.1]{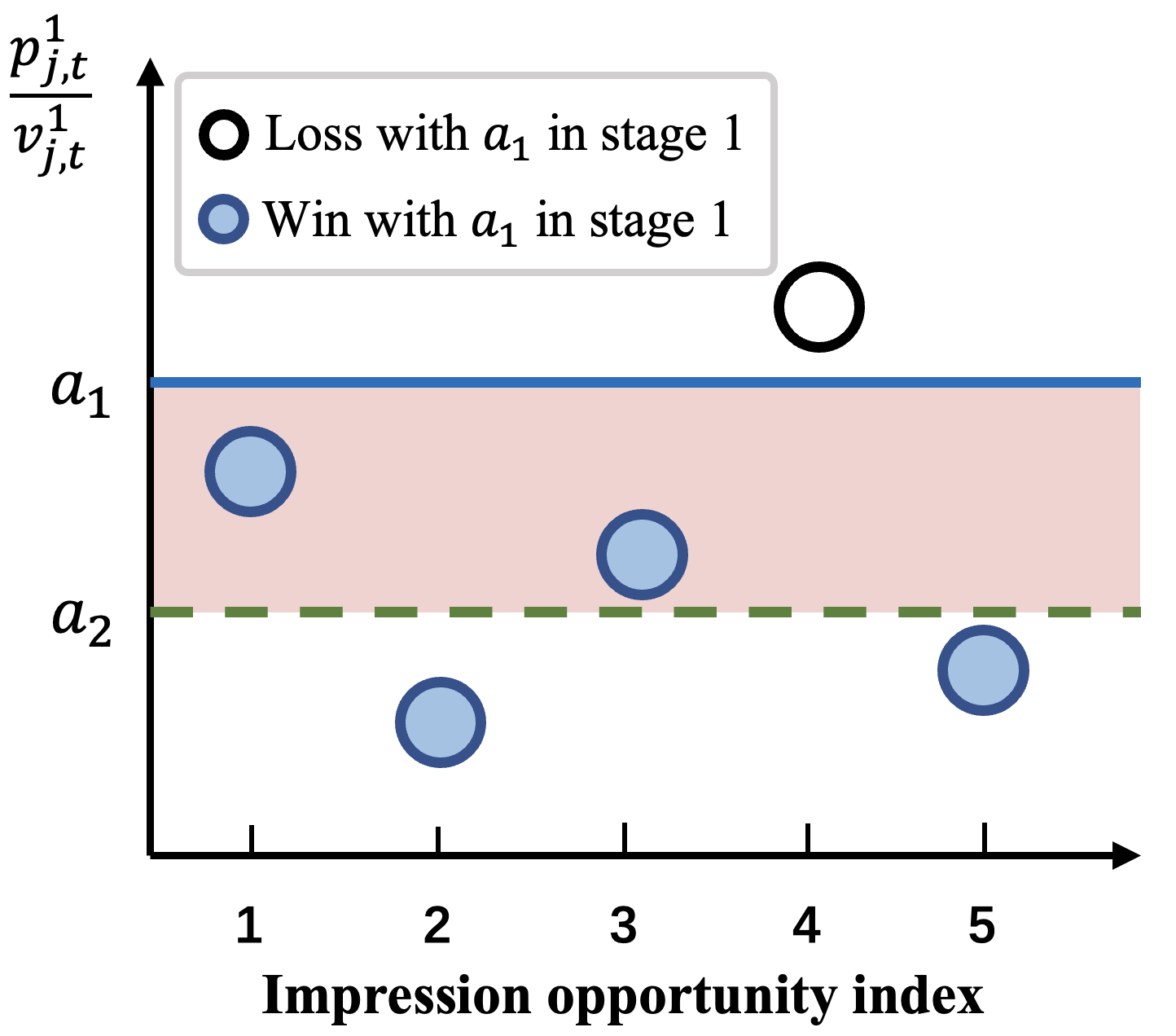}}
		\subfigure[Bidding with $a_1$ in stage 2.]
		{
			\label{fig:a_1_ranking}
			\includegraphics[scale=0.1]{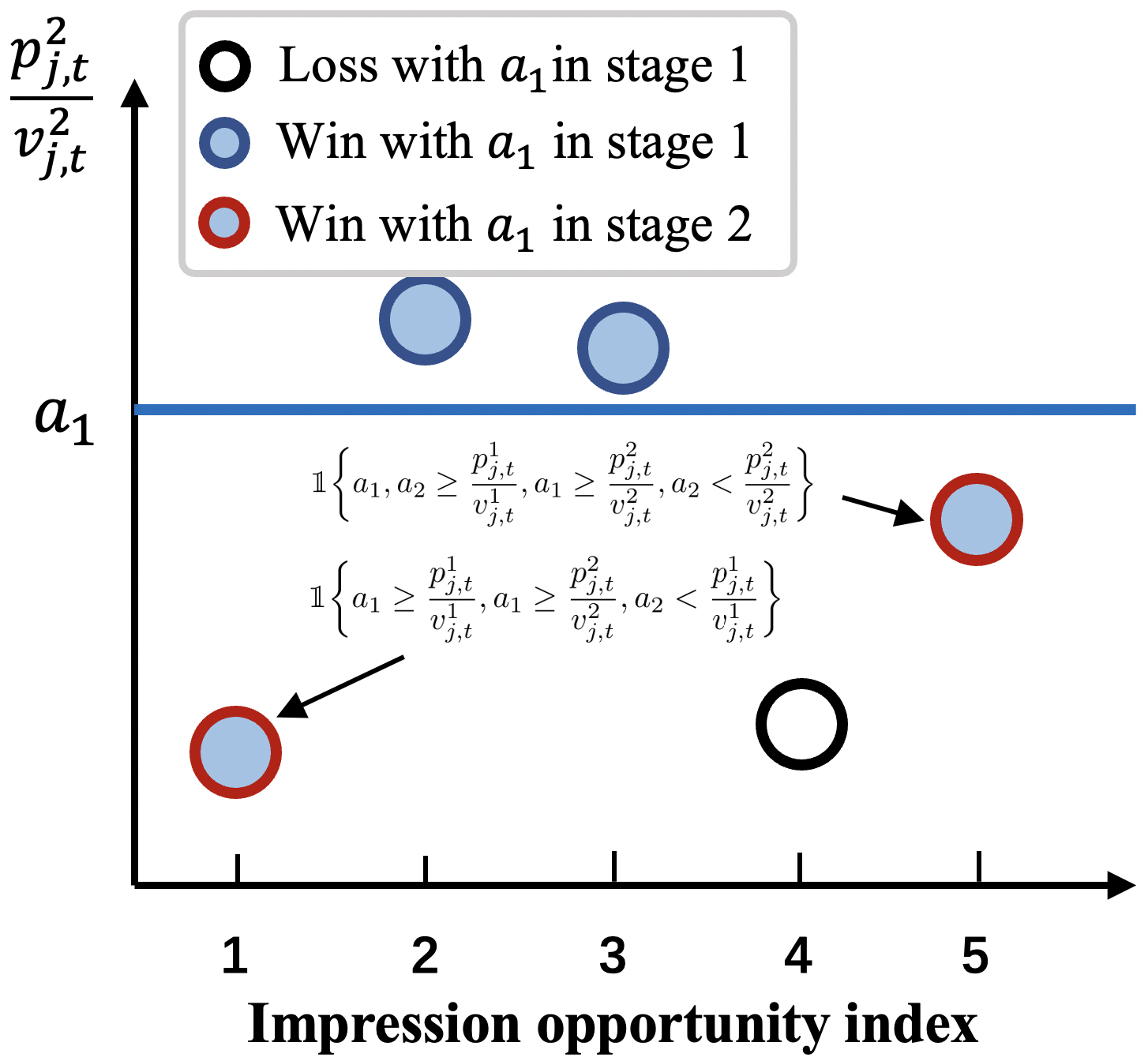}
		}
		
		\subfigure[Bidding with $a_2$ in stage 1.]
		{
			\label{fig:a_2_pre_ranking}
			\includegraphics[scale=0.1]{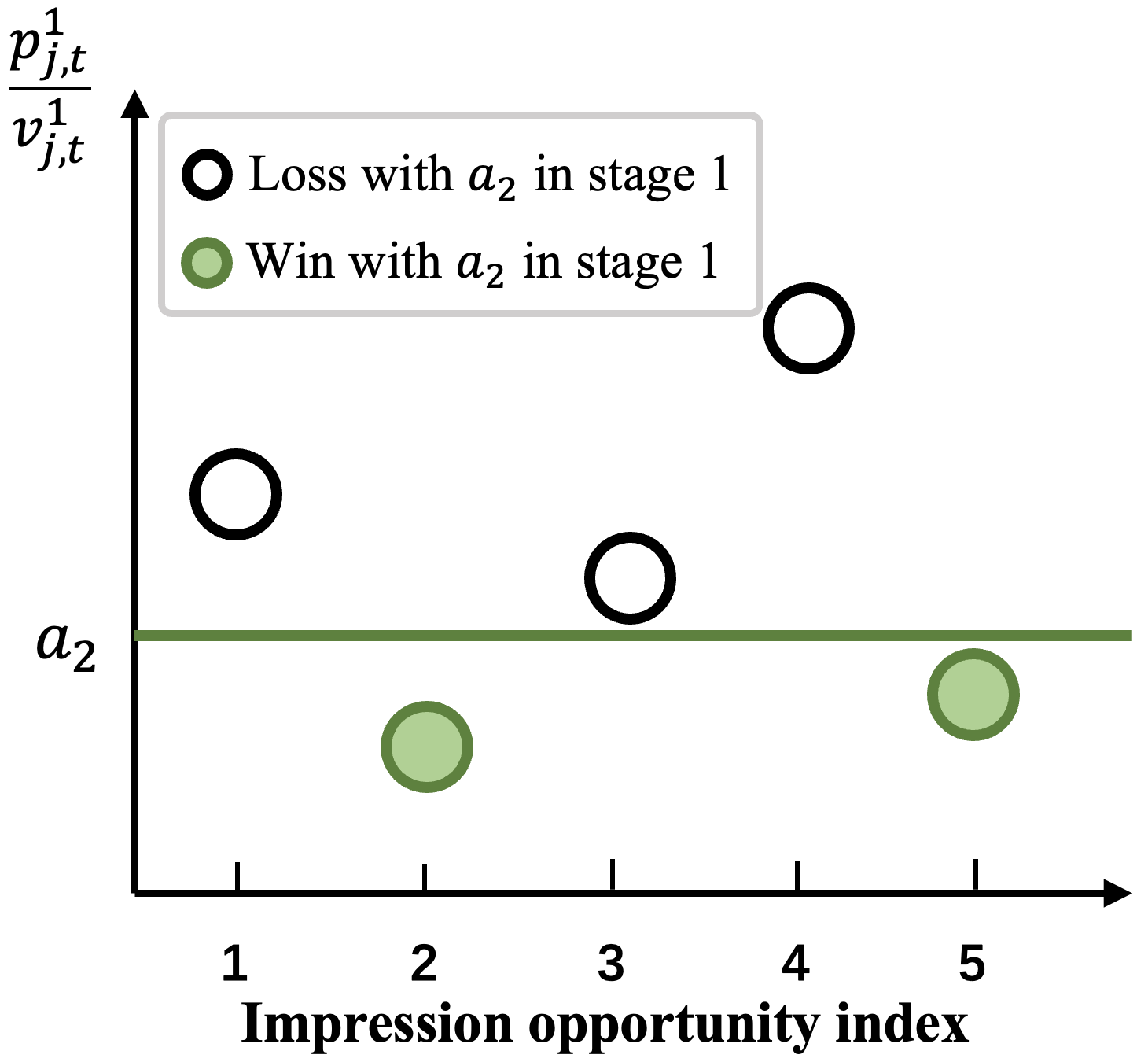}
		}
		\subfigure[Bidding with $a_2$ in stage 2.]
		{	\label{fig:a_2_ranking}
			\includegraphics[scale=0.1]{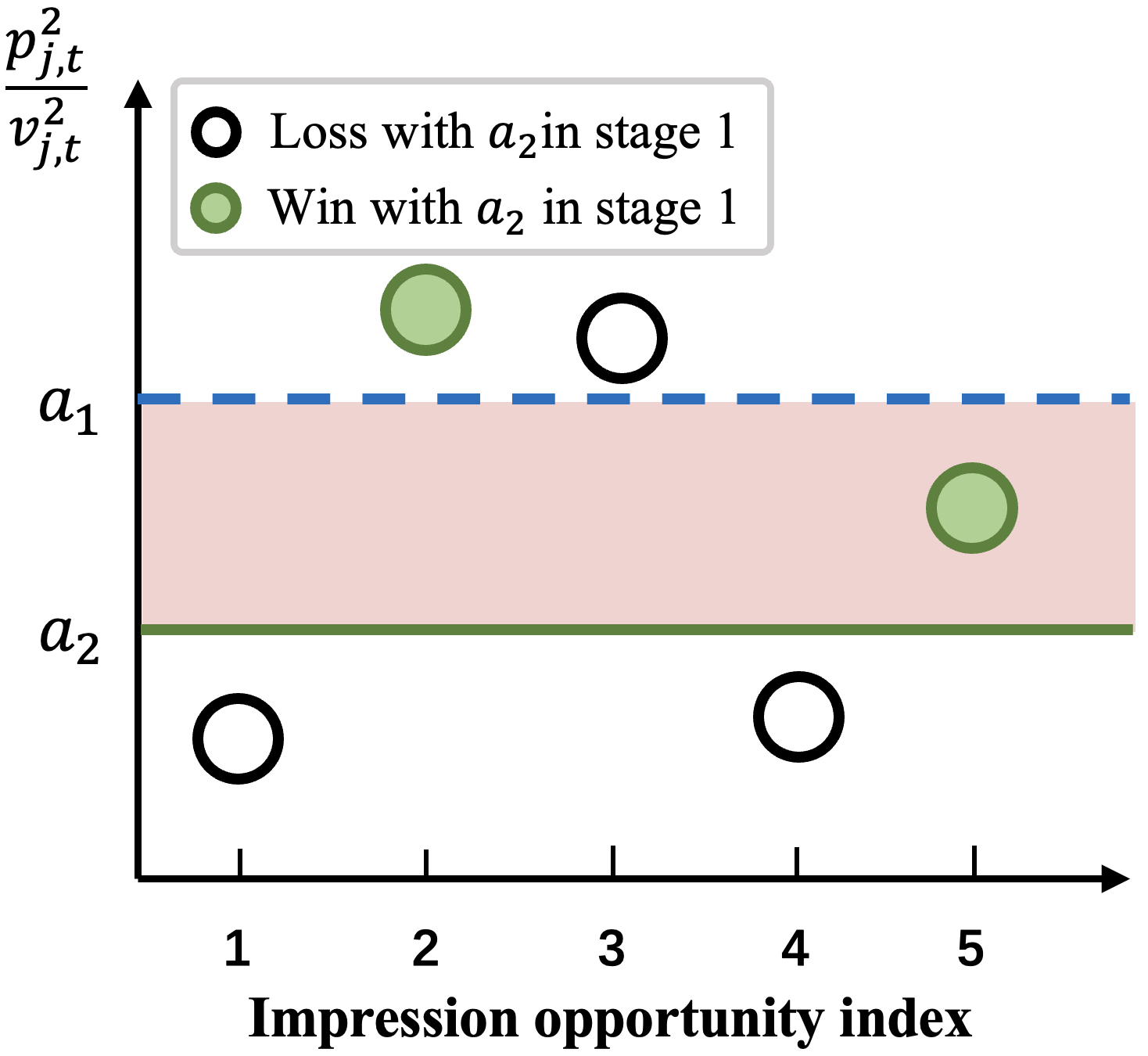}
		}
		
		\caption{Bidding with $a_1$ and $a_2$ in stage 1 and stage 2, where $a_1>a_2$. The extra impression opportunities won by bid $a_1$ compared to bid $a_2$ are impression opportunity $1$ that satisfies $a_1\ge \frac{p_{1,t}^1}{v_{1,t}^1},a_1\ge \frac{p_{1,t}^2}{v_{1,t}^2},a_2< \frac{p_{1,t}^2}{v_{1,t}^2}$, and impression $5$ that satisfies $a_1,a_2\ge \frac{p_{5,t}^1}{v_{5,t}^1},a_1\ge \frac{p_{5,t}^2}{v_{5,t}^2} \ge a_2$.}
		\label{fig:pre_ranking&ranking}
	\end{figure}

		To illustrates these two parts of impression opportunities, an example of bidding with $a_1$ and $a_2$ in stage 1 and stage 2 is shown in Fig. \ref{fig:pre_ranking&ranking}. The first part of impression opportunities can be bounded by:
		\begin{align}
			\sum_{j=1}^{N_t}\mathbbm{1}\bigg\{a_1\ge \frac{p_{j,t}^1}{v_{j,t}^1},a_1\ge \frac{p_{j,t}^2}{v_{j,t}^2},a_2< \frac{p_{j,t}^1}{v_{j,t}^1}\bigg\}\le \sum_{j=1}^{N_t}\mathbbm{1}\bigg\{a_1\ge \frac{p_{j,t}^1}{v_{j,t}^1}>a_2\bigg\},
		\end{align}
		which is represented by the red shaded area in Fig.  \ref{fig:a_1_pre_ranking}. Similarly, the second part of impression opportunities can be bounded by:
		\begin{align}
			\sum_{j=1}^{N_t}\mathbbm{1}\bigg\{a_1,a_2\ge \frac{p_{j,t}^1}{v_{j,t}^1},a_1\ge \frac{p_{j,t}^2}{v_{j,t}^2},a_2< \frac{p_{j,t}^2}{v_{j,t}^2}\bigg\}\le \sum_{j=1}^{N_t}\mathbbm{1}\bigg\{a_1\ge \frac{p_{j,t}^2}{v_{j,t}^2}>a_2\bigg\},
		\end{align}
		which is represented by the red shaded area in Fig. \ref{fig:a_2_ranking}.
		Hence,  {with Assumption \ref{assumption:uniform},} we have 
		\begin{align}
			\label{equ:r_smooth}
		\bigg\vert	r_t(s_t,a_1)-r_t(s_t,a_2)	\bigg\vert&\le v_M\sum_{j=1}^{N_t}\bigg[\mathbbm{1}\bigg\{a_1\ge \frac{p_{j,t}^1}{v_{j,t}^1}>a_2\bigg\}+\mathbbm{1}\bigg\{a_1\ge \frac{p_{j,t}^2}{v_{j,t}^2}>a_2\bigg\}\bigg]\notag\\
			&\le(k_1+k_2)v_M|a_1-a_2|.
		\end{align}
		The upper bound of the changing rate of the reward function $r_t(s_t,a_t)$ is
		\begin{align}
			\frac{	\bigg\vert r_t(s_t,a_1)-r_t(s_t,a_2)	\bigg\vert}{|a_1-a_2|}\le (k_1+k_2)v_M,
		\end{align}
		which indicates that $r_t(s_t,a_t)$ is $L_r$-Lipschitz smooth, $L_r \triangleq (k_1+k_2)v_M$.
	\end{proof}

	\subsection{Proof of Theorem 2}
		\label{app:theorem_2}
	\begin{theorem}[\textnormal{Lipschitz Smooth of $Q^\mu(s_t,a_t)$}]
		Under Assumption \ref{assumption:uniform} and \ref{assumption:bounded_Q}, the Q function $Q^\mu(s_t,a_t)$ is an $L_Q$-Lipschitz smooth function with respect to the actions $a_t$ at any given state $s_t$, where $L_Q=[v_M+(k_3+k_4)p_M](k_1+k_2)$.
	\end{theorem}
	\begin{proof}
		Recall that $Q^\mu(s_t,a_t)$ can be expressed as 
		\begin{align}
			Q^\mu(s_t,a_t)=r_t(s_t,a_t)+\gamma\mathbb{E}_{s_{t+1}\sim \mathbb{P}(\cdot\mid s_t,a_t)}Q^\mu(s_{t+1},\mu(s_{t+1})),
		\end{align}
		where $r_t(s_t,a_t)$ is Lipschitz smooth. Thus, we first focus on the characteristic of the second part $\gamma\mathbb{E}_{s_{t+1}\sim \mathbb{P}(\cdot\mid s_t,a_t)}Q^\mu(s_{t+1},\mu(s_{t+1}))$. According to Proposition \ref{proposition:form}, at any given state $s_t\in\mathcal{S}$, the next states $s_{t+1}^1$ and $s_{t+1}^2$ under bids $a_1$ and $a_2$ can be expressed as 
		\begin{align}
			s_{t+1}^1 = s_t+[\triangle s_1(1),\triangle s_1(2),\triangle s_1(3)], \;\;	s_{t+1}^2= s_t+[\triangle s_2(1),\triangle s_2(2),\triangle s_2(3)], 
		\end{align} 
		where 
		\begin{align}
			\triangle s_1(1)=-\triangle s_1(3)=-\sum_{j=1}^{N_t}\mathbbm{1}\bigg\{a_1\ge \frac{p_{j,t}^1}{v_{j,t}^1},a_1\ge \frac{p_{j,t}^2}{v_{j,t}^2}\bigg\}p_{j,t},
		\end{align}
		and 
		\begin{align}
			\triangle s_2(1)=-\triangle s_2(3)=-\sum_{j=1}^{N_t}\mathbbm{1}\bigg\{a_2\ge \frac{p_{j,t}^1}{v_{j,t}^1},a_2\ge \frac{p_{j,t}^2}{v_{j,t}^2}\bigg\}p_{j,t}.
		\end{align}
		and
		\begin{align}
			\triangle s_1(2)=\triangle s_2(2)=-1.
		\end{align}
		Hence, using Taylor expansion, we have 
		\begin{align}
			&\quad\; \bigg\vert\mathbb{E}_{s^1_{t+1}\sim \mathbb{P}(\cdot\mid s_t,a_1)}Q^\mu(s_{t+1}^1,\mu(s_{t+1}^1))-\mathbb{E}_{s_{t+1}^2\sim \mathbb{P}(\cdot\mid s_t,a_2)}Q^\mu(s_{t+1}^2,\mu(s_{t+1}^2))\bigg\vert\notag\\
			&=\bigg\vert\mathbb{E}_{s_{t+1}^1\sim \mathbb{P}(\cdot\mid s_t,a_1)}Q^\mu(s_{t+1}^1)-\mathbb{E}_{s_{t+1}^2\sim \mathbb{P}(\cdot\mid s_t,a_2)}Q^\mu(s_{t+1}^2)\bigg\vert\notag\\
			&=\bigg\vert Q^\mu(s_t+[\triangle s_1(1),\triangle s_1(2),\triangle s_1(3)])-Q^\mu(s_t+[\triangle s_2(1),\triangle s_2(2),\triangle s_2(3)])\bigg\vert\notag\\
			&\approx\bigg\vert Q^\mu(s_t)+\frac{\partial Q^\mu(s_t)}{\partial s_t(1)}\triangle s_1(1)+\frac{\partial Q^\mu(s_t)}{\partial s_t(2)}\triangle s_1(2)+\frac{\partial Q^\mu(s_t)}{\partial s_t(3)}\triangle s_1(3)-Q^\mu(s_t)-\frac{\partial Q^\mu(s_t)}{\partial s_t(1)}\triangle s_2(1)\notag\\
			&\quad\;\; -\frac{\partial Q^\mu(s_t)}{\partial s_t(2)}\triangle s_2(2)-\frac{\partial Q^\mu(s_t)}{\partial s_t(3)}\triangle s_2(3)
			\bigg\vert \notag\\
			&=\bigg\vert \frac{\partial Q^\mu(s_t)}{\partial s_t(1)}\triangle s_1(1)+\frac{\partial Q^\mu(s_t)}{\partial s_t(3)}\triangle s_1(3)-\frac{\partial Q^\mu(s_t)}{\partial s_t(1)}\triangle s_2(1)- \frac{\partial Q^\mu(s_t)}{\partial s_t(3)}\triangle s_2(3) \bigg\vert\notag\\
			&=\bigg\vert  \bigg(\frac{\partial Q^\mu(s_t)}{\partial s_t(1)}-\frac{\partial Q^\mu(s_t)}{\partial s_t(3)}\bigg)     \bigg( \sum_{j=1}^{N_t}\bigg[-\mathbbm{1}\bigg\{a_1\ge \frac{p_{j,t}^1}{v_{j,t}^1}, \frac{p_{j,t}^2}{v_{j,t}^2}\bigg\}+\mathbbm{1}\bigg\{a_2\ge \frac{p_{j,t}^1}{v_{j,t}^1}, \frac{p_{j,t}^2}{v_{j,t}^2}\bigg\}\bigg]p_{j,t}.\bigg)
			\bigg\vert \notag\\
			&\le \bigg\vert  \bigg(\frac{\partial Q^\mu(s_t)}{\partial s_t(1)}-\frac{\partial Q^\mu(s_t)}{\partial s_t(3)}\bigg)   \bigg\vert   \bigg\vert -r_t(s_t,a_1)+r_t(s_t,a_2)  \bigg\vert\frac{p_M}{v_M}  \notag\\
			&\le \bigg(\bigg\vert\frac{\partial Q^\mu(s_t)}{\partial s_t(1)}\bigg\vert+\bigg\vert\frac{\partial Q^\mu(s_t)}{\partial s_t(3)}\bigg\vert\bigg)(k_1+k_2)p_M|a_1-a_2|\notag\\
			&=(k_1+k_2)(k_3+k_4)p_M|a_1-a_2|.
		\end{align}
		Note that we use \eqref{equ:r_smooth} in Theorem \ref{thm:r_smooth}. Hence, we have
		\begin{align}
			\bigg\vert Q^\mu(s_t,a_1)-Q^\mu(s_t,a_2)	\bigg\vert&\le 	\bigg\vert r_t(s_t,a_1)-r_t(s_t,a_2) 	\bigg\vert +\notag\\
			&\quad\; \gamma	\bigg\vert\mathbb{E}_{s_{t+1}^1\sim \mathbb{P}(\cdot\mid s_t,a_1)}Q^\mu(s_{t+1}^1)-\mathbb{E}_{s_{t+1}^2\sim \mathbb{P}(\cdot\mid s_t,a_2)}Q^\mu(s_{t+1}^2)	\bigg\vert\notag\\
			&\le (k_1+k_2)v_M	\bigg\vert a_1-a_2	\bigg\vert +\gamma(k_1+k_2)(k_3+k_4)p_M	\bigg\vert a_1-a_2	\bigg\vert\notag\\
			&=\bigg[v_M+\gamma(k_3+k_4)p_M\bigg](k_1+k_2)	\bigg\vert a_1-a_2	\bigg\vert.
		\end{align}
		The upper bound of the absolute changing rate of the Q function $Q^\mu(s_t,a_t)$ is 
		\begin{align}
			\frac{	\bigg\vert Q^\mu(s_t,a_1)-Q^\mu(s_t,a_2)	\bigg\vert}{	\bigg\vert a_1-a_2	\bigg\vert} \le \bigg[v_M+\gamma(k_3+k_4)p_M\bigg](k_1+k_2),
		\end{align}
		which indicates that $Q^\mu(s_t,a_t)$ is $L_Q$-Lipschitz smooth, $L_Q\triangleq [v_M+\gamma(k_3+k_4)p_M](k_1+k_2)$.
	\end{proof}

	\subsection{Proof of Theorem 3}
\label{app:theorem_3}
\begin{theorem}[ {\textnormal{Upper Bound of $|V(\pi)-V(\mu_s)|$}}]
The expected accumulated reward $V(\pi_e)$ satisfies 
\begin{align}
	\bigg\vert V(\pi_e) - V(\mu_s)\bigg\vert\le \xi\gamma^{t_1}\bigg[v_M+\gamma\big(k_3+k_4\big)p_M\bigg]\big(k_1+k_2\big)\Delta T.
\end{align}
\end{theorem}
	\begin{figure}[h]
	\centering
	\includegraphics[width=14cm]{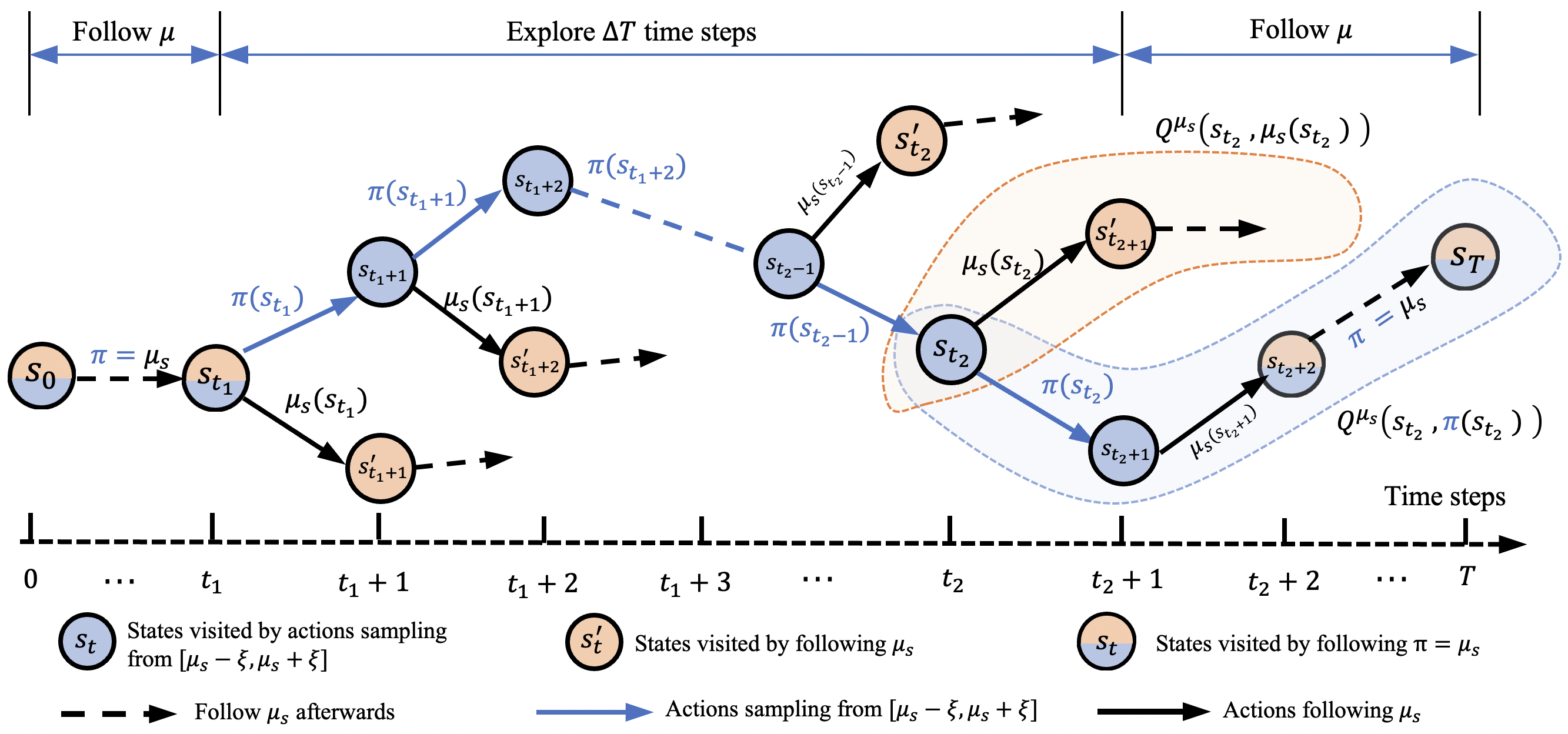}
	\caption{Practical exploration policy $\pi$ based on safe policy $\mu_s$.}
	\label{fig:theorem_3}
\end{figure}
\begin{proof}
	 { Fig. \ref{fig:theorem_3} shows the visited states during one episode using exploration policy $\pi$. The total value $V(\pi)$ can be expressed as \footnote{ {Note that the state transitions in all formulas follow the rule of $\mathbb{P}$, i.e., $s_{t+1}\sim\mathbb{P}(\cdot|s_{t},a_t), \forall\tau\in\{0,1,...,T-1\}$. Hence, for brevity, we omit this term in the subscript of the expectation operator $\mathbb{E}$ in the following formulas.}}: }
	\begin{align}
		V(\pi)&=\mathbb{E}_{a_t\sim\pi}\bigg[\sum_{t=0}^{T-1}\gamma^t r_t(s_t,a_t)\bigg\vert s_0\sim\rho_0\bigg]\notag\\
		&=\mathbb{E}_{a_t\sim\mu_s}\bigg[\sum_{t=0}^{t_1-1}\gamma^t r_t(s_t,a_t)\bigg\vert s_0\sim\rho_0\bigg]+\mathbb{E}_{a_t\sim\pi}\bigg[\sum_{t=t_1}^{T-1}\gamma^t r_t(s_t,a_t)\bigg\vert s_{t_1}\sim\rho_{t_1}\bigg],
	\end{align}
 {where $\rho_t$ denotes the state distribution at time step $t$ starting from $s_0\sim\rho_0$ and following $\pi$. The total value $V(\mu_s)$ is}
	\begin{align}
	V(\mu_s)=\mathbb{E}_{a_t\sim\mu_s}\bigg[\sum_{t=0}^{T-1}\gamma^t r_t(s_t,a_t)\bigg\vert s_0\sim\rho_0\bigg].
\end{align}
 {Notice that the accumulated rewards from time step $0$ to time step $t_1$ in both $V(\pi)$ and $V(\mu_s)$ are the same. Hence, the difference between $V(\pi)$ and $V(\mu_s)$ can be calculated as}
\begin{align}
	\label{equ:thm_3_diff}
	V(\pi)-V(\mu_s)&=\underbrace{\mathbb{E}_{a_t\sim\pi}\bigg[\sum_{t=t_1}^{T-1}\gamma^t r_t(s_t,a_t)\bigg\vert s_{t_1}\sim\rho_{t_1}\bigg]}_\text{\ding{172}}-\mathbb{E}_{a_t\sim\mu_s}\bigg[\sum_{t=t_1}^{T-1}\gamma^t r_t(s_t,a_t)\bigg\vert s_{t_1}\sim\rho_{t_1}\bigg].
\end{align}
 {The term \ding{172} can be further divided into three parts, including the accumulated rewards from time step $t$ to $t+\Delta T-1$ (part 1), the immediate reward at time step $t+\Delta T$ (part 2) and the accumulated rewards from time step $t+\Delta T+1$ to $T$ (part 3), i.e., }

\begin{align}
	\label{equ:thm_3_2}
	&\quad\;\text{\ding{172}}=\mathbb{E}_{a_t\sim\pi}\bigg[\sum_{t=t_1}^{T-1}\gamma^t r_t(s_t,a_t)\bigg\vert s_{t_1}\sim\rho_{t_1}\bigg]=\underbrace{\mathbb{E}_{a_t\sim\pi}\bigg[\sum_{t=t_1}^{t_2-1}\gamma^t r_t(s_t,a_t)\bigg\vert s_{t_1}\sim\rho_{t_1}\bigg]}_\text{part 1: accumulated rewards from $t_1$ to $t_2$ following $\pi$ }\notag\\
	&+\gamma^{t_2}\mathbb{E}_{ s_{t_2}\sim\rho_{t_2}}\bigg[\underbrace{\underbrace{r_{t_2}(s_{t_2},\pi(s_{t_2}))}_\text{part 2: immediate reward at time step $t_2$}+\gamma\underbrace{\mathbb{E}_{a_t\sim\mu_s,\forall t\ge t_2+1}\bigg[\sum_{t=t_2 + 1}^{ T-1}\gamma^{t-t_2-1} r_t(s_t,a_t)\bigg\vert s_{t_2+1}\sim\rho_{t_2+1}\bigg]}_\text{part 3: accumulated rewards from $t_2 +1$ to $T-1$ following $\mu_s$ }}_\text{part 2+part 3=$Q^{\mu_s}(s_{t_2}, \pi(s_{t_2}))$}\bigg]\notag\\
	&=\mathbb{E}_{a_t\sim\pi}\bigg[\sum_{t=t_1}^{t_2-1}\gamma^t r_t(s_t,a_t)\bigg\vert s_{t_1}\sim\rho_{t_1}\bigg] \notag\\
	&\quad+ \gamma^{t_2}\mathbb{E}_{ s_{t_2}\sim\rho_{t_2}}\bigg[Q^{\mu_s}(s_{t_2}, \pi(s_{t_2}))\underbrace{-Q^{\mu_s}(s_{t_2}, \mu_s(s_{t_2})) + Q^{\mu_s}(s_{t_2}, \mu_s(s_{t_2}))}_\text{trick: plus and minus $Q^{\mu_s}(s_{t_2},\mu_s(s_{t_2}))$}\bigg]\notag\\
	&=\underbrace{\mathbb{E}_{a_t\sim\pi}\bigg[\sum_{t=t_1}^{t_2-1}\gamma^t r_t(s_t,a_t)\bigg\vert s_{t_1}\sim\rho_{t_1}\bigg]}_\text{\ding{173}}+\gamma^{t_2}\mathbb{E}_{ s_{t_2}\sim\rho_{t_2}}\bigg[\underbrace{Q^{\mu_s}(s_{t_2}, \pi(s_{t_2}))-Q^{\mu_s}(s_{t_2}, \mu_s(s_{t_2}))}_\text{$\triangleq \Delta Q({t_2})$}\bigg]\notag\\
	&\quad+ \underbrace{\gamma^{t_2}\mathbb{E}_{ s_{t_2}\sim\rho_{t_2}}\bigg[Q^{\mu_s}(s_{t_2}, \mu_s(s_{t_2}))\bigg]}_\text{\ding{174}},
\end{align}
 {where we define $\Delta Q(t)\triangleq Q^{\mu_s}(s_{t}, \pi(s_{t}))-Q^{\mu_s}(s_{t}, \mu_s(s_{t}))$. Note that we can take term $\gamma^{t_2-1}r_{t_2-1}(s_{t_2-1},a_{t_2-1})$ from \ding{173} and combine it with term \ding{174} to obtain $Q^{\mu_s}(s_{t_2-1},\pi(s_{t_2-1}))$ . In fact,we have: $\forall \tau\in\{1,2,...,t_2-t_1\}$, }
\begin{align}
	&\;\;\;\;\mathbb{E}_{a_t\sim\pi}\bigg[\sum_{t=t_1}^{t_2-\tau}\gamma^t r_t(s_t,a_t)\bigg\vert s_{t_1}\sim\rho_{t_1}\bigg]+\gamma^{t_2-\tau+1}\mathbb{E}_{ s_{t_2-\tau+1}\sim\rho_{t_2-\tau+1}}\bigg[Q^{\mu_s}(s_{t_2-\tau+1}, \mu_s(s_{t_2-\tau+1}))\bigg]\notag\\
	&=\mathbb{E}_{a_t\sim\pi}\bigg[\sum_{t=t_1}^{t_2-\tau-1}\gamma^t r_t(s_t,a_t)\bigg\vert s_{t_1}\sim\rho_{t_1}\bigg]+\gamma^{t_2-\tau}\bigg\{\mathbb{E}_{s_{t_2-\tau}\sim\rho_{t_2-\tau}}\bigg[r_{t_2-\tau}(s_{t_2-\tau},\pi(s_{t_2-\tau}))\bigg]\notag\\
	&\quad+\gamma\mathbb{E}_{a_t\sim\mu_s,\forall t\ge t_2-\tau+1}\bigg[\sum_{t=t_2-\tau+1}^{T-1}\gamma^{t-t_2+\tau-1}r_t(s_t,a_t)\bigg\vert s_{t_2-\tau+1}\sim\rho_{t_2-\tau+1}\bigg]\bigg\}\notag\\
	&=\underbrace{\mathbb{E}_{a_t\sim\pi}\bigg[\sum_{t=t_1}^{t_2-\tau-1}\gamma^t r_t(s_t,a_t)\bigg\vert s_{t_1}\sim\rho_{t_1}\bigg]}_\text{\ding{173}'}+\gamma^{t_2-\tau}\mathbb{E}_{s_{t_2-\tau}\sim\rho_{t_2-\tau}}\bigg[Q^{\mu_s}(s_{t_2-\tau},\pi(s_{t_2-\tau}))\bigg].
\end{align}
 {We can continue to use \emph{plus and minus} trick as we did in \eqref{equ:thm_3_2} to further break down term \ding{173}'. Hence, term \ding{172} can be calculated as}
\begin{align}
	\label{equ:term_1}
	&\;\;\text{\ding{172}}=\mathbb{E}_{a_t\sim\pi}\bigg[\sum_{t=t_1}^{t_2-2}\gamma^t r_t(s_t,a_t)\bigg\vert s_{t_1}\sim\rho_{t_1}\bigg]+\gamma^{t_2-1}\mathbb{E}_{s_{t_2-1}\sim\rho_{t_2-1}}\bigg[Q^{\mu_s}(s_{t_2-1},\mu_s(s_{t_2-1}))\bigg]\notag\\
	&\quad+\gamma^{t_2-1}\mathbb{E}_{ s_{t_2-1}\sim\rho_{t_2-1}}\bigg[\Delta Q(t_2-1)\bigg]+\gamma^{t_2}\mathbb{E}_{ s_{t_2}\sim\rho_{t_2}}\bigg[\Delta Q(t_2)\bigg]\notag\\
	&=\mathbb{E}_{a_t\sim\pi}\bigg[\sum_{t=t_1}^{t_2-3}\gamma^t r_t(s_t,a_t)\bigg\vert s_{t_1}\sim\rho_{t_1}\bigg]+\gamma^{t_2-2}\mathbb{E}_{s_{t_2-2}\sim\rho_{t_2-2}}\bigg[Q^{\mu_s}(s_{t_2-2},\mu_s(s_{t_2-2}))\bigg]\notag\\
	&\quad+\gamma^{t_2-2}\mathbb{E}_{ s_{t_2-2}\sim\rho_{t_2-2}}\bigg[\Delta Q(t_2-2)\bigg]+\gamma^{t_2-1}\mathbb{E}_{ s_{t_2-1}\sim\rho_{t_2-1}}\bigg[\Delta Q(t_2-1)\bigg]+\gamma^{t_2}\mathbb{E}_{ s_{t_2}\sim\rho_{t_2}}\bigg[\Delta Q(t_2)\bigg]\notag\\
	&=\cdots\cdots\cdots\notag\\
	&=\mathbb{E}_{a_t\sim\pi}\bigg[\gamma^{t_1}r_{t_1}(s_{t_1},a_{t_1})\bigg\vert s_{t_1}\sim\rho_{t_1}\bigg] + \gamma^{t_1+1}\mathbb{E}_{s_{t_1+1}\sim\rho_{t_1+1}}\bigg[Q^{\mu_s}(s_{t_1+1},\mu_s(s_{t_1+1}))\bigg]\notag\\
	&\quad+\sum_{t=t_1+1}^{t_2}\gamma^t\mathbb{E}_{s_t\sim\rho_t}\bigg[\Delta Q(t)\bigg]\notag\\
	&=\gamma^{t_1}\mathbb{E}_{s_{t_1}\sim\rho_{t_1}}\bigg\{r_{t_1}(s_{t_1},\pi(s_{t_1}))+\gamma\mathbb{E}_{a_t\sim\mu_s,\forall t\ge t_1+1}\bigg[\sum_{t=t_1+1}^{T-1}\gamma^{t-t_1-1}r_t(s_t,a_t)\bigg\vert s_{t_1+1}\sim\rho_{t_1+1}\bigg]\bigg\}\notag\\
	&\quad+\sum_{t=t_1+1}^{t_2}\gamma^t\mathbb{E}_{s_t\sim\rho_t}\bigg[\Delta Q(t)\bigg]\notag\\
	&=\gamma^{t_1}\mathbb{E}_{s_{t_1}\sim\rho_{t_1}}\bigg[Q^{\mu_s}(s_{t_1},\pi(s_{t_1}))\bigg] + \sum_{t=t_1+1}^{t_2}\gamma^t\mathbb{E}_{s_t\sim\rho_t}\bigg[\Delta Q(t)\bigg].
\end{align}
 {Substitute term \ding{172} in \eqref{equ:thm_3_diff} by \eqref{equ:term_1}, we can obtain }
\begin{align}
	V(\pi) - V(\mu_s) &= \gamma^{t_1}\mathbb{E}_{s_{t_1}\sim\rho_{t_1}}\bigg[Q^{\mu_s}(s_{t_1},\pi(s_{t_1}))\bigg] -\mathbb{E}_{a_t\sim\mu_s}\bigg[\sum_{t=t_1}^{T-1}\gamma^t r_t(s_t,a_t)\bigg\vert s_{t_1}\sim\rho_{t_1}\bigg]\notag\\
	&\quad+ \sum_{t=t_1+1}^{t_2}\gamma^t\mathbb{E}_{s_t\sim\rho_t}\bigg[\Delta Q(t)\bigg]\notag\\
	&=\gamma^{t_1}\mathbb{E}_{s_{t_1}\sim\rho_{t_1}}\bigg[Q^{\mu_s}(s_{t_1},\pi(s_{t_1}))\bigg]-\gamma^{t_1}\mathbb{E}_{a_t\sim\mu_s}\bigg[\sum_{t=t_1}^{T-1}\gamma^{t-t_1} r_t(s_t,a_t)\bigg\vert s_{t_1}\sim\rho_{t_1}\bigg]\notag\\
	&\quad+ \sum_{t=t_1+1}^{t_2}\gamma^t\mathbb{E}_{s_t\sim\rho_t}\bigg[\Delta Q(t)\bigg]\notag\\
	&=\gamma^{t_1}\mathbb{E}_{s_{t_1}\sim\rho_{t_1}}\bigg[Q^{\mu_s}(s_{t_1},\pi(s_{t_1}))-Q^{\mu_s}(s_{t_1},\mu_s(s_{t_1}))\bigg]+ \sum_{t=t_1+1}^{t_2}\gamma^t\mathbb{E}_{s_t\sim\rho_t}\bigg[\Delta Q(t)\bigg]\notag\\
	&= \sum_{t=t_1}^{t_2}\gamma^t\mathbb{E}_{s_t\sim\rho_t}\bigg[\Delta Q(t)\bigg].
\end{align}
 {From Theorem \ref{thm:Q_smooth}, we know that}
\begin{align}
	\bigg\vert\Delta Q(t)\bigg\vert&=\bigg\vert Q^{\mu_s}(s_{t},\pi(s_{t}))-Q^{\mu_s}(s_{t},\mu_s(s_{t}))\bigg\vert\notag\\
	&\le \bigg[v_M+\gamma\big(k_3+k_4\big)p_M\bigg]\big(k_1+k_2\big)\bigg\vert\pi(s_t)-\mu_s(s_t)\bigg\vert\notag\\
	&\le \bigg[v_M+\gamma\big(k_3+k_4\big)p_M\bigg](k_1+k_2)\xi,
\end{align}
 {where we use $|\pi(s_t)-\mu_s(s_t)|\le\xi$. Hence, we have}
\begin{align}
	\label{equ:thm_3}
	\bigg\vert V(\pi) - V(\mu_s)\bigg\vert &=\bigg\vert\sum_{t=t_1}^{t_2}\gamma^t\mathbb{E}_{s_t\sim\rho_t}\bigg[\Delta Q(t)\bigg]\bigg\vert\notag\\
	&\le \sum_{t=t_1}^{t_2}\gamma^t\mathbb{E}_{s_t\sim\rho_t}\bigg[\bigg\vert\Delta Q(t)\bigg\vert\bigg]\notag\\
	&\le \sum_{t=t_1}^{t_2}\gamma^t\mathbb{E}_{s_t\sim\rho_t}\bigg[[v_M+\gamma(k_3+k_4)p_M](k_1+k_2)\xi\bigg]\notag\\
	&=\sum_{t=t_1}^{t_2}\gamma^t\bigg[v_M+\gamma\big(k_3+k_4\big)p_M\bigg]\big(k_1+k_2\big)\xi\notag\\
	&\le \xi\gamma^{t_1}\bigg[v_M+\gamma\big(k_3+k_4\big)p_M\bigg]\big(k_1+k_2\big)\Delta T.
\end{align}
 {So far, we have proved the theorem. In addition, from \eqref{equ:thm_3}, we can obtain the following two conclusions:}
\begin{itemize}
	\item  {As $\gamma<1$, we can see that the later we start explorations (i.e., the larger $t_1$ is), the smaller $|V(\pi)-V(\mu_s)|$ is. }
	\item   {The larger the exploration time steps $\Delta T$ is, the bigger $|V(\pi)-V(\mu_s)|$ is.}
\end{itemize}

\end{proof}

\section{Our Approach: SORL Framework}
\label{app:SORL}
\subsection{Additional Details on SER Policy}
\label{app:SORL_SER}

\subsubsection{Derivations of the SER Policy $\pi_e$}
\label{app:exploration_derivation}
Recall that the functional optimization problem of the SER policy $\pi_e$ is
\begin{align}
	\max_{\pi_e,\forall s}\quad&\mathbb{E}_{a_t\sim\pi_e(\cdot|s_t)}\widehat{Q}(s_t,a_t)\label{app_optimization}\\
	\mathrm{ s.t. } \,\quad&D(\pi_e,\pi_{e,\mathcal{N}})\le \epsilon_e,\tag{\ref{app_optimization}{a}}\label{app_constraint_1}
\end{align}
and the Lagrange function is
\begin{align}
	\mathcal{L}(\pi_e,\lambda)&=-\mathbb{E}_{a_t\sim\pi_e(\cdot|s_t)}\widehat{Q}(s_t,a_t)+\lambda(KL(\pi_e,\pi_{e,\mathcal{N}})-\epsilon_e)\notag\\
	&=\int_{a_t}-\pi_e(a_t|s_t)\widehat{Q}(s_t,a_t)\mathrm{d}a_t+\lambda\int_{a_t}\pi_e\log\frac{\pi_e}{\pi_{e,\mathcal{N}}}\mathrm{d}a_t-\lambda\epsilon_e\notag\\
	&=\int_{a_t}\bigg[-\pi_e(a_t|s_t)\widehat{Q}(s_t,a_t)+\lambda\pi_e\log\frac{\pi_e}{\pi_{e,\mathcal{N}}}\bigg]\mathrm{d}a_t-\lambda\epsilon_e\notag\\
	&=\int_{a_t}F[\pi_e(a_t)]\mathrm{d}a_t-\lambda\epsilon_e,
\end{align}
where $F[\pi_e(a_t)]=-\pi_e(a_t|s_t)\widehat{Q}(s_t,a_t)+\lambda\pi_e\log\frac{\pi_e}{\pi_{e,\mathcal{N}}}$. According to Euler equation, a necessary condition of the optimal solution to \eqref{app_optimization} satisfies:
\begin{align}
	\delta \mathcal{L}(\pi_e,\lambda)=\int_{a_t}\frac{\partial F[\pi_e(a_t)]}{\partial \pi_e}\delta\pi_e\mathrm{d}a_t=0,\quad\lambda\ge 0.
\end{align}  
Due to the arbitrariness of $\delta\pi_e$, we have $\frac{\partial F[\pi_e(a_t)]}{\partial \pi_e}=0$, i.e.,
\begin{align}
	-\widehat{Q}(s_t,a_t)+\lambda\log\frac{\pi_e}{\pi_{e,\mathcal{N}}}+\lambda{\pi_e}\frac{\pi_{e,\mathcal{N}}}{\pi_e}\frac{1}{\pi_{e,\mathcal{N}}}=0\notag\\
	\Rightarrow\quad\quad-\widehat{Q}(s_t,a_t)+\lambda\log\frac{\pi_e}{\pi_{e,\mathcal{N}}}+\lambda=0\notag\\
	\Rightarrow\quad\quad\quad\,\;\,\,\,\exp\{\frac{\widehat{Q}(s_t,a_t)}{\lambda}-1\}=\frac{\pi_e}{\pi_{e,\mathcal{N}}}\notag\\
	\Rightarrow\quad\qquad\,\;\; \pi_e=\frac{\pi_{e,\mathcal{N}}}{e}\exp\bigg\{\frac{1}{\lambda}\widehat{Q}(s_t,a_t)\bigg\}.&
\end{align}
To ensure that $\pi_e$ is a distribution over actions, we modify it to
\begin{align}	\pi_e=\frac{1}{C(s_t)}\pi_{e,\mathcal{N}}\exp\bigg\{\frac{1}{\lambda}\widehat{Q}(s_t,a_t)\bigg\},
\end{align}
where $C(s_t)=\int_{a_t}\frac{1}{\sqrt{2\pi\sigma^2}}\exp\{-\frac{(a_t-\mu_s(s_t))^2}{2\sigma^2}+\frac{1}{\lambda}\widehat{Q}(s_t,a_t)\}\mathrm{d}a_t$ acts as the normalization factor. Note that the KL divergence $KL(\pi_e,\pi_{e,\mathcal{N}})$ we used in the derivations is formulated as $\int_{a_t}\pi_e\log\frac{\pi_e}{\pi_{e,\mathcal{N}}}\mathrm{d}a_t$ rather than $\int_{a_t}\pi_{e,\mathcal{N}}\log\frac{\pi_{e,\mathcal{N}}}{\pi_e}\mathrm{d}a_t$. The reason is that: the exploration policy will be calculated as $\pi_e=\pi_{e,\mathcal{N}}\frac{\widehat{Q}(s_t,a_t)}{\lambda}$ if we use the latter KL divergence, which cannot guarantee the non-negative property of $\pi_e$. We also note that the form of SER policy $\pi_e$ here resembles the results derived in \cite{dir_1, dir_2, AWAC}. Nonetheless, we view the problem as a functional optimization problem and utilize the Euler equation to obtain the results.

\subsubsection{Practical Implementations of $\pi_e$}
\label{app:practical_implementations}
As the Q function $\widehat{Q}$ is a neural network, we cannot directly sample actions from $\pi_e$. Nonetheless, we can obtain the value of $\pi_e$ of each action $a_t$ given a state $s_t$. Hence, we uniformly sample $M\in\mathbb{N}_+$ actions $\{a_t^m\}_{m=1}^M$ within the safety zone $[\mu_s(s_t)-\xi,\mu_s(s_t)+\xi]$, and the possibility of selecting action $a_t^m$ is calculated as $\pi_e(a_t^m|s_t)/\sum_{m=1}^{M}\pi_e(a_t^m|s_t)$. Then we sample the actions from $\{a_t^m\}_{m=1}^M$ for explorations.
\subsubsection{More on the Safety Requirement}
\label{app:safety_explain}
In other realms, such as robotics, it is possible to construct 
an \emph{immediate-evaluated} safety function to evaluate the safety of current state-action pairs, which is merely related to the values of them. For example, a robot can instantly be in danger due to an action at a state, for example, dashing against the wall. However, it is not appropriate to construct such kind of safety functions in auto-bidding.
Generally, there are usually two main kinds of  dangerous situations in auto-bidding:
\begin{itemize}
	\item the first situation is the extremely quick burns of budgets with high cost-per-action (CPA) values, which is probably caused by continuously bidding at very high prices;
	\item the second situation is the extremely slow consumptions of budgets, which is probably caused by continuously bidding at very low prices. 
\end{itemize}
Both of these two dangerous situations cannot be attributed to a specific state-action pair, but to a long-term auto-bidding policy.  In fact, any action (bids) in any state would be safe as long as the total subsequent rewards maintains at a high level. For example, bidding an oddly high price in time step $t$, but bidding at reasonable prices afterwards and the overall reward at the end of the episode is at a high level would be acceptable. On the contrary, bidding reasonably at present moment but continuously bidding at oddly high prices afterwards, resulting in a low accumulated reward at the end of the episode, would harm the interests of advertisers and not be safe in auto-bidding.

	\subsection{Additional Details on V-CQL}
	\label{app:V-CQL}
	\subsubsection{Nearly Quadratic Form of Q Functions}
	\label{app:V-CQL_1}
	Fig. \ref{fig:optimal_Q} shows the optimal Q functions in the simulated experiments, and Fig. \ref{fig:USCB_Q} shows the Q function of the state-of-the-art Q functions. 
	We can see that the Q functions are all in quadratic forms. Hence, based on the proved Lipschitz smooth property of Q functions, we can reasonably assume that the optimal Q function is nearly quadratic. 
	
	\begin{figure}[h]
		\centering
		\subfigure[at state $(20, 0.9, 80)$.]
		{\label{fig:1}
			\includegraphics[scale=0.82]{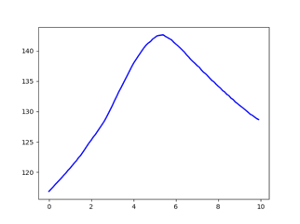}}
		\subfigure[at state $(50, 0.9, 50)$]
		{
			\label{fig:2}
			\includegraphics[scale=0.82]{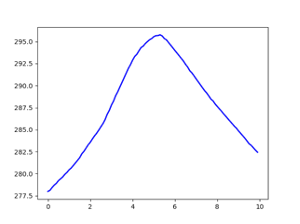}
		}
		\subfigure[at state $(80, 0.9, 20)$]
		{
			\label{fig:3}
			\includegraphics[scale=0.82]{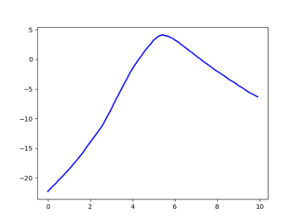}
		}
		
		\subfigure[at state $(20, 0.6, 80)$]
		{\label{fig:4}
			\includegraphics[scale=0.82]{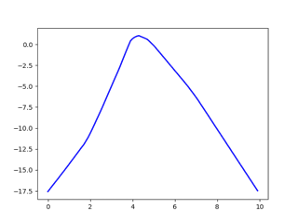}}
		\subfigure[at state $(50, 0.6, 50)$]
		{
			\label{fig:5}
			\includegraphics[scale=0.82]{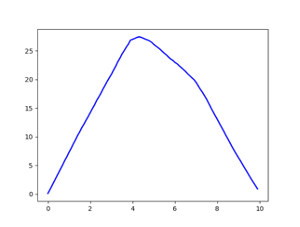}
		}
		\subfigure[at state $(80, 0.6, 20)$]
		{
			\label{fig:6}
			\includegraphics[scale=0.82]{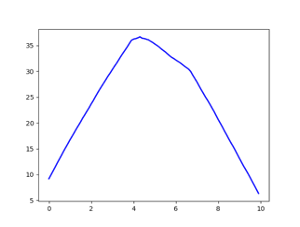}
		}
	
		\subfigure[at state $(20, 0.3, 80)$]
	{\label{fig:7}
		\includegraphics[scale=0.82]{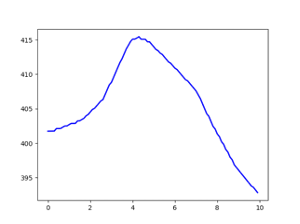}}
	\subfigure[at state $(50, 0.3, 50)$.]
	{
		\label{fig:8}
		\includegraphics[scale=0.82]{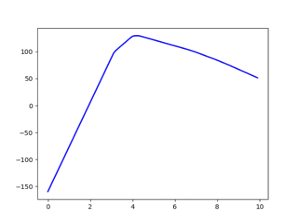}
	}
	\subfigure[at state $(50, 0.3, 50)$.]
	{
		\label{fig:9}
		\includegraphics[scale=0.82]{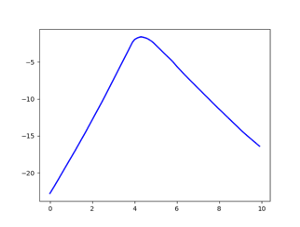}
	}

		\caption{The form of optimal Q functions in the simulated experiments are all nearly in quadratic form. In this example, the total budget is 100, and we choose time left to be $0.9, 0.6, 0.3$, respectively. }
		\label{fig:optimal_Q}
	\end{figure}

	\begin{figure}[h]
	\centering
	\subfigure[at state $(0.872, 0.9653, -1)$.]
	{\label{fig:1_}
		\includegraphics[scale=0.25]{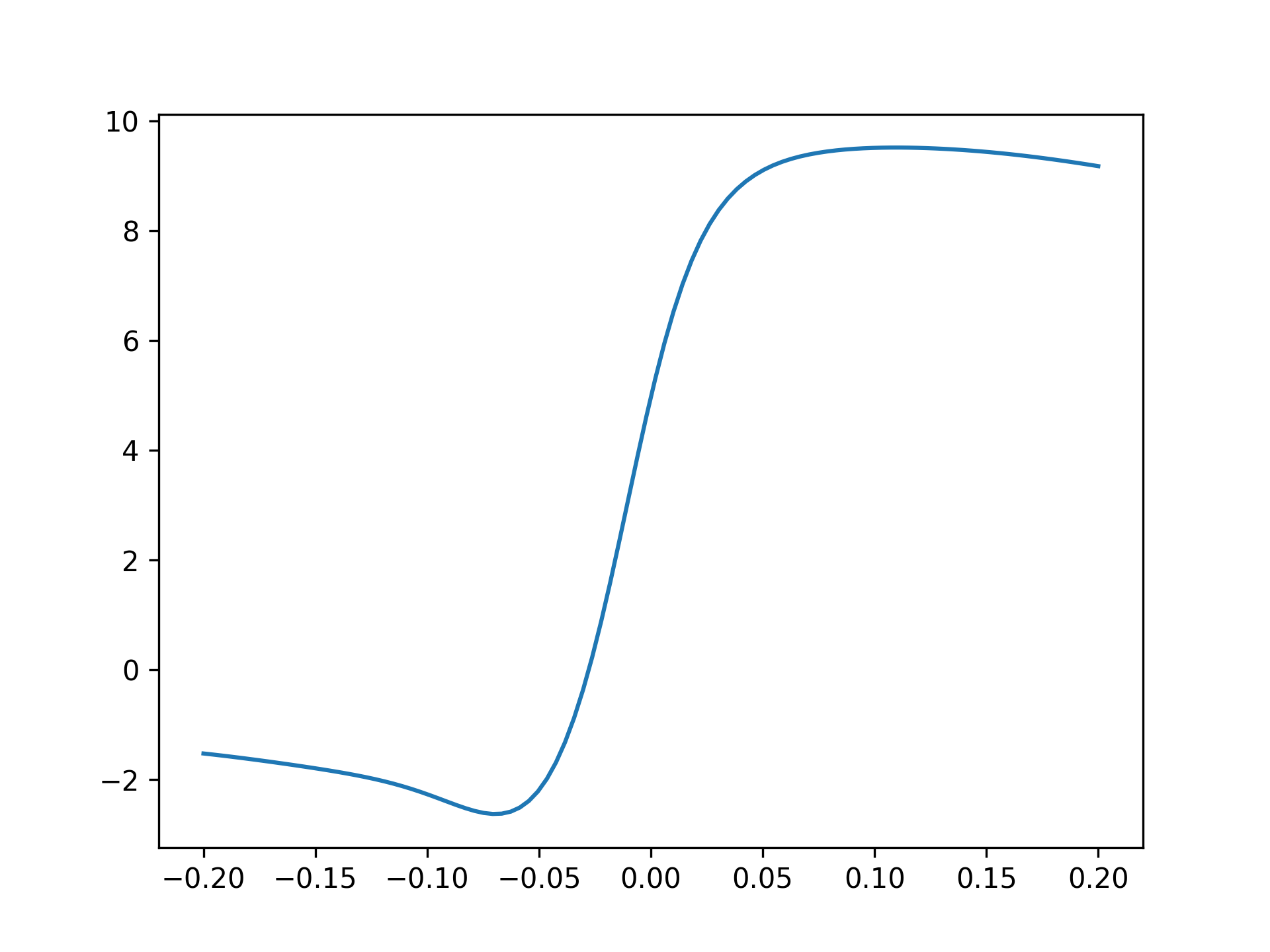}}
	\subfigure[at state $(0.5579, 0.9653, 0)$]
	{
		\label{fig:2_}
		\includegraphics[scale=0.25]{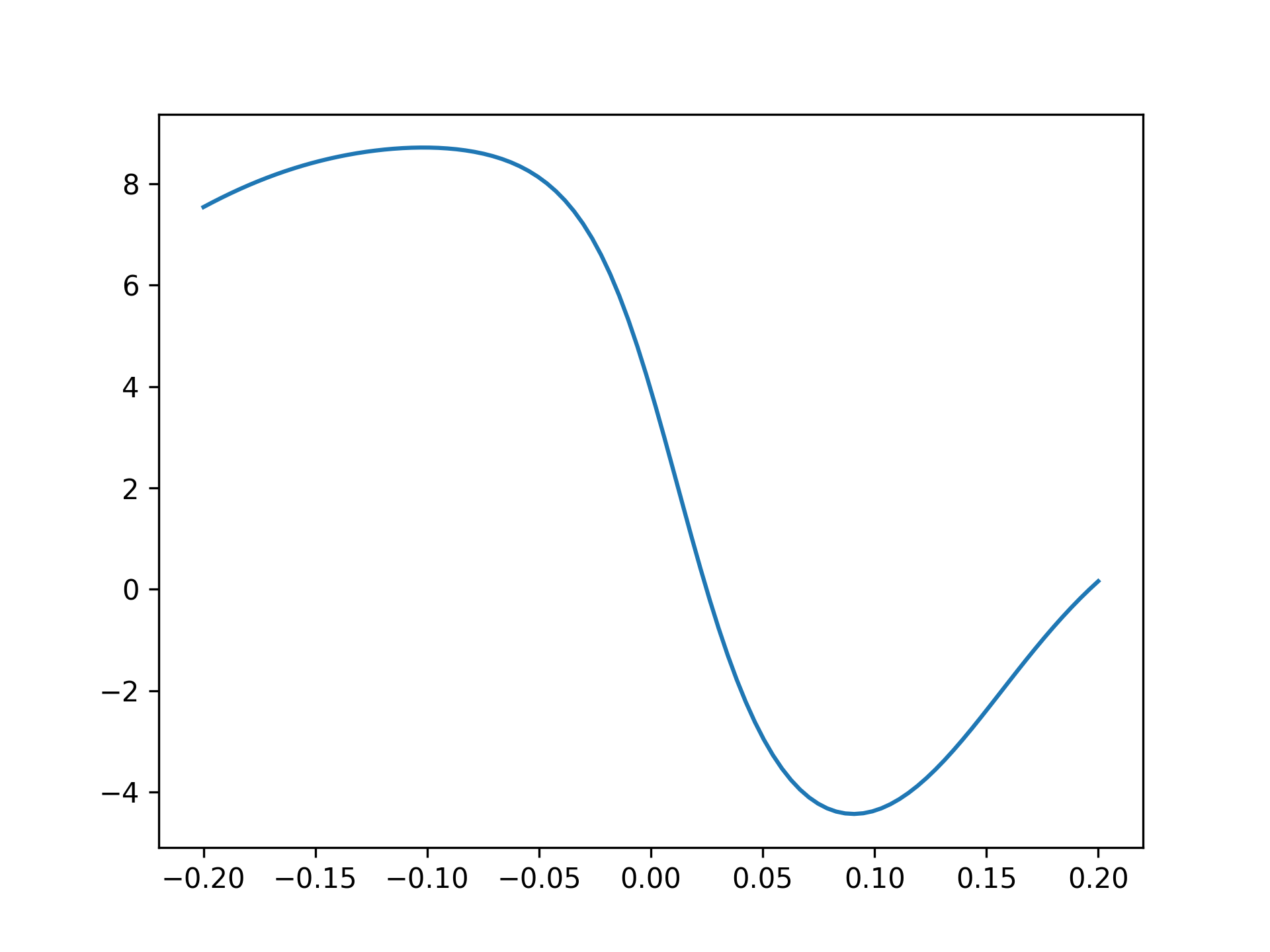}
	}
	\subfigure[at state $(0.1423, 0.9653, 1)$]
	{
		\label{fig:3_}
		\includegraphics[scale=0.25]{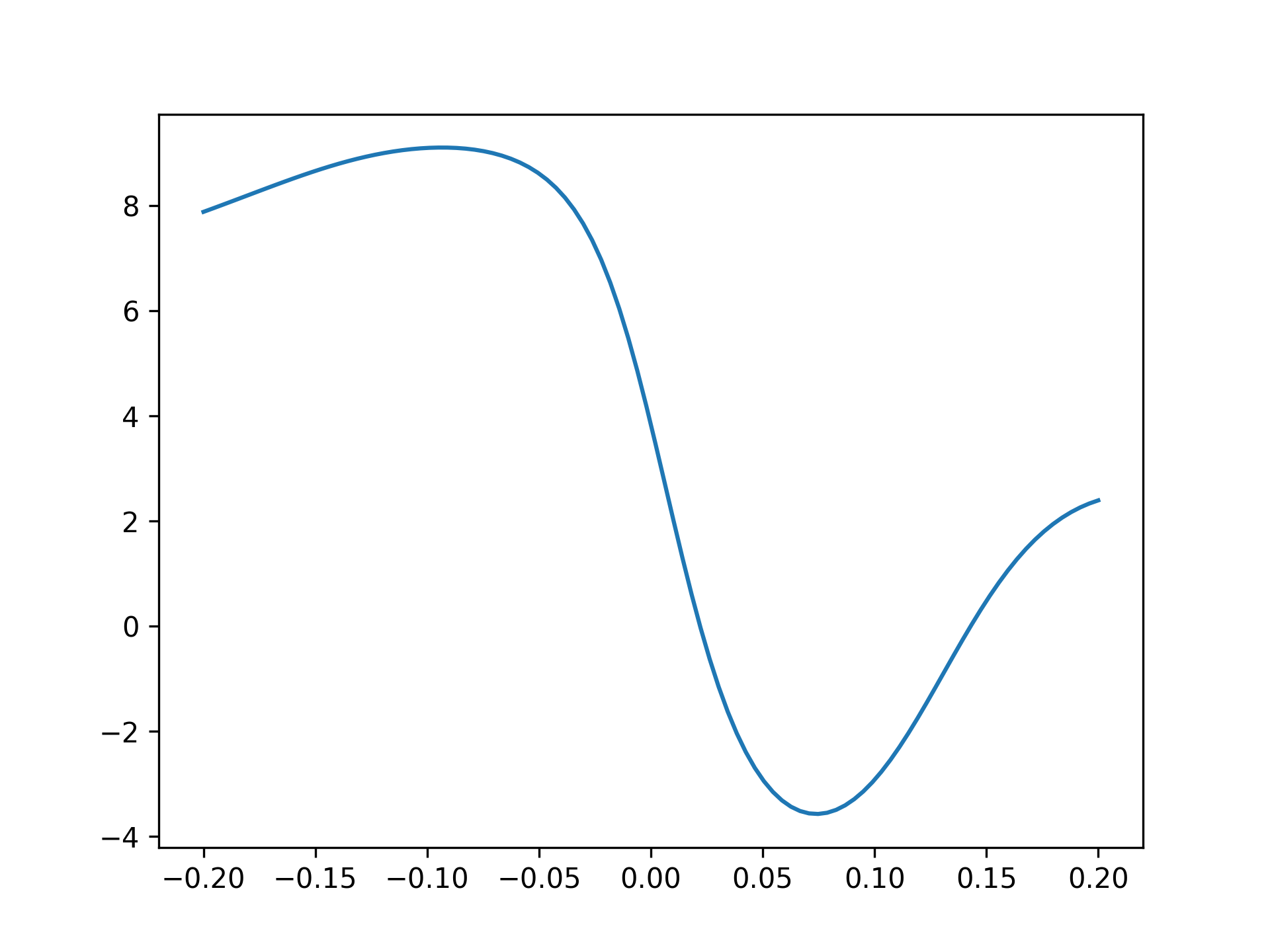}
	}
	
	\subfigure[at state $(0.872, 0.6458, 1)$]
	{\label{fig:4_}
		\includegraphics[scale=0.25]{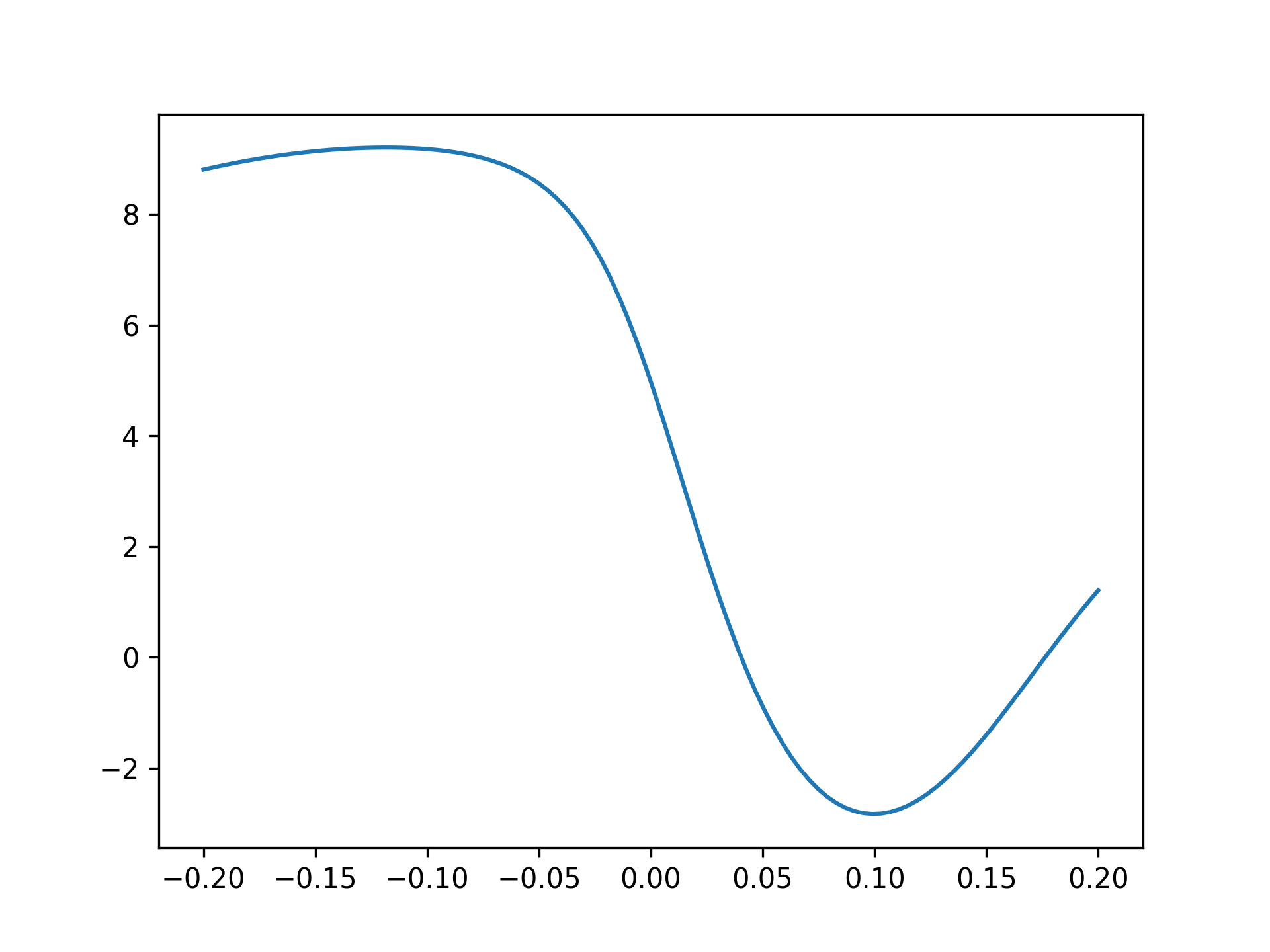}}
	\subfigure[at state $(0.5579, 0.6458, 0)$]
	{
		\label{fig:5_}
		\includegraphics[scale=0.25]{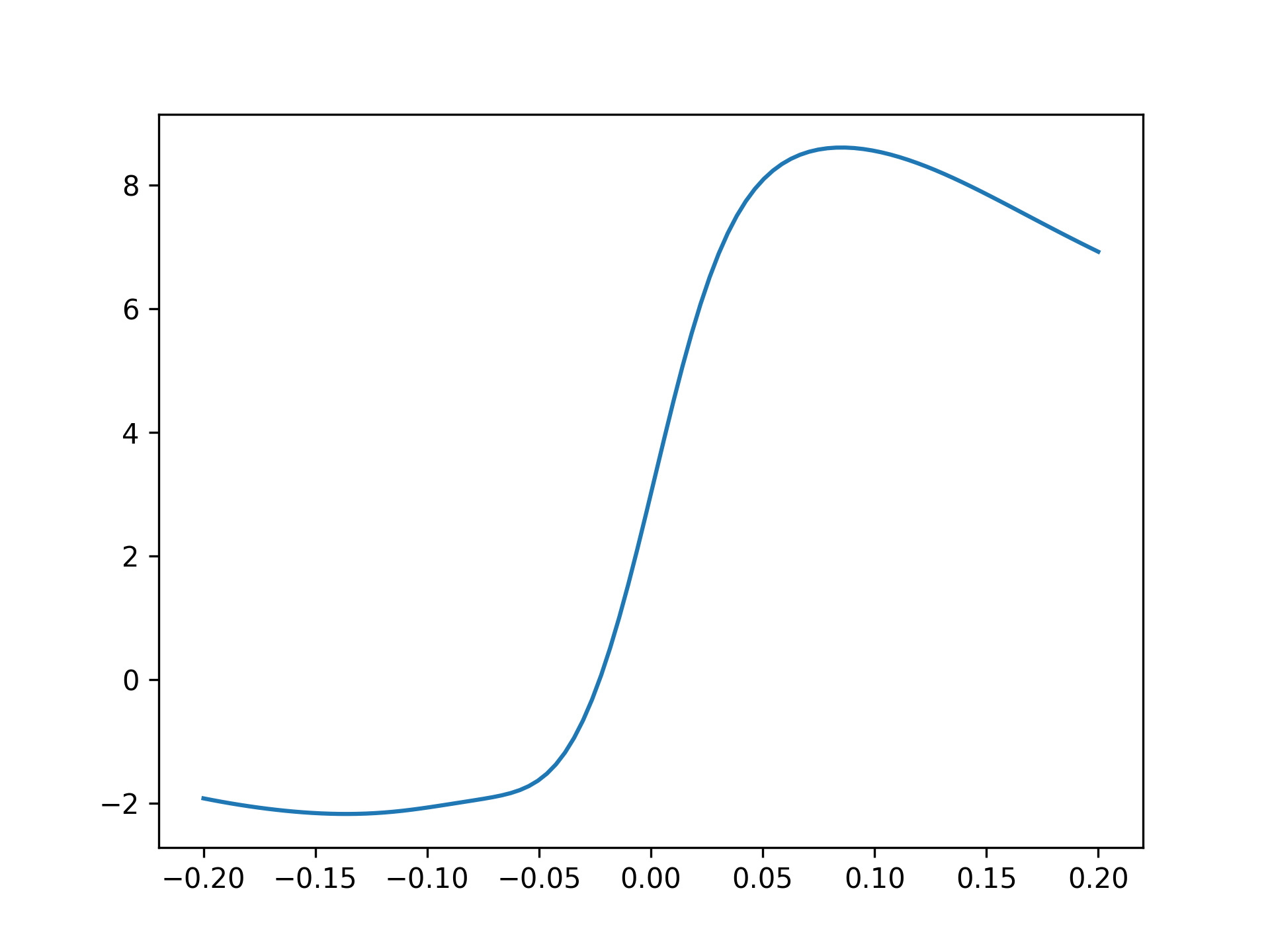}
	}
	\subfigure[at state $(0.1423, 0.6458, -1)$]
	{
		\label{fig:6_}
		\includegraphics[scale=0.25]{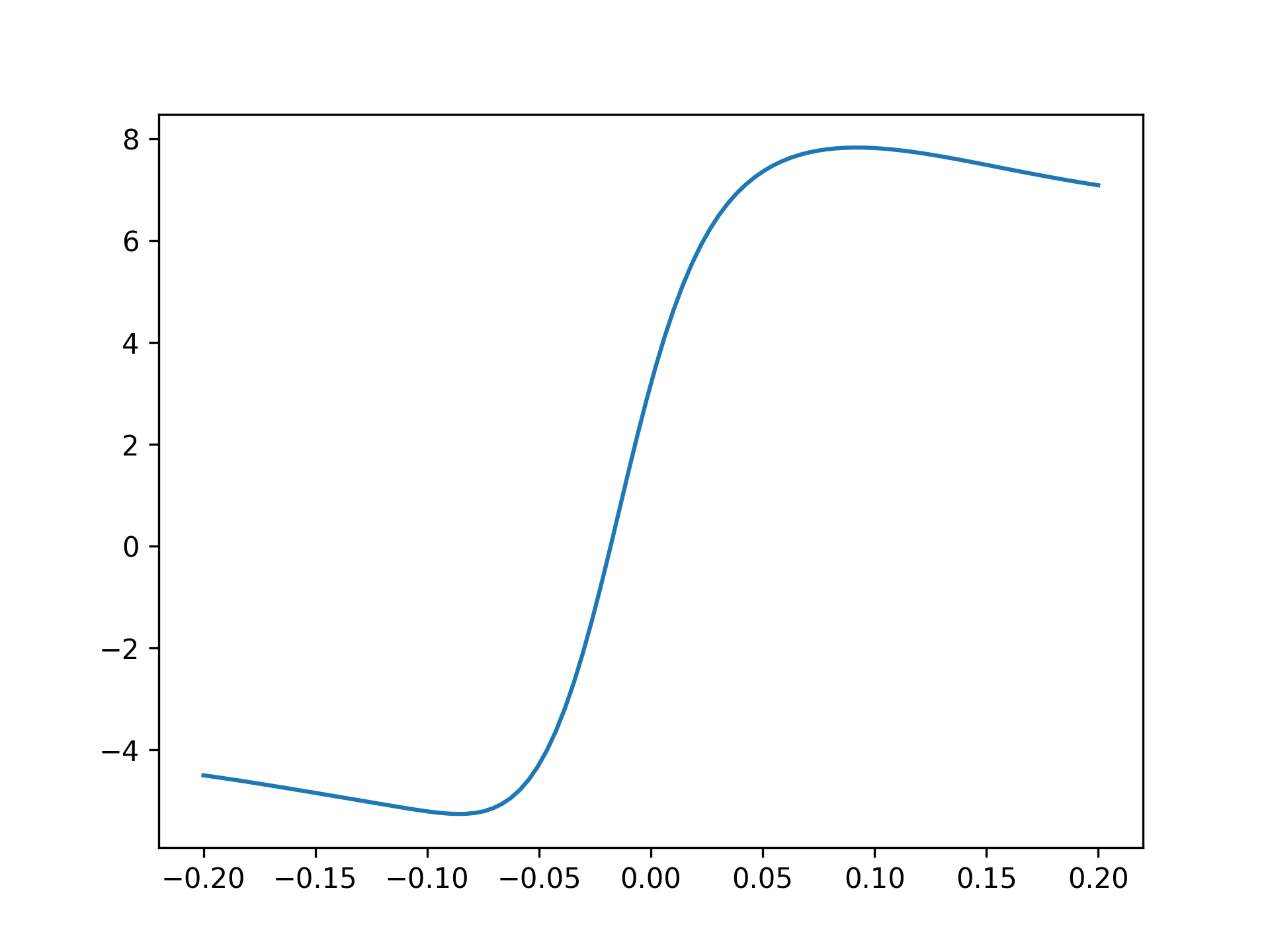}
	}
	
	\subfigure[at state $(0.872, 0.3156, 1)$]
	{\label{fig:7_}
		\includegraphics[scale=0.25]{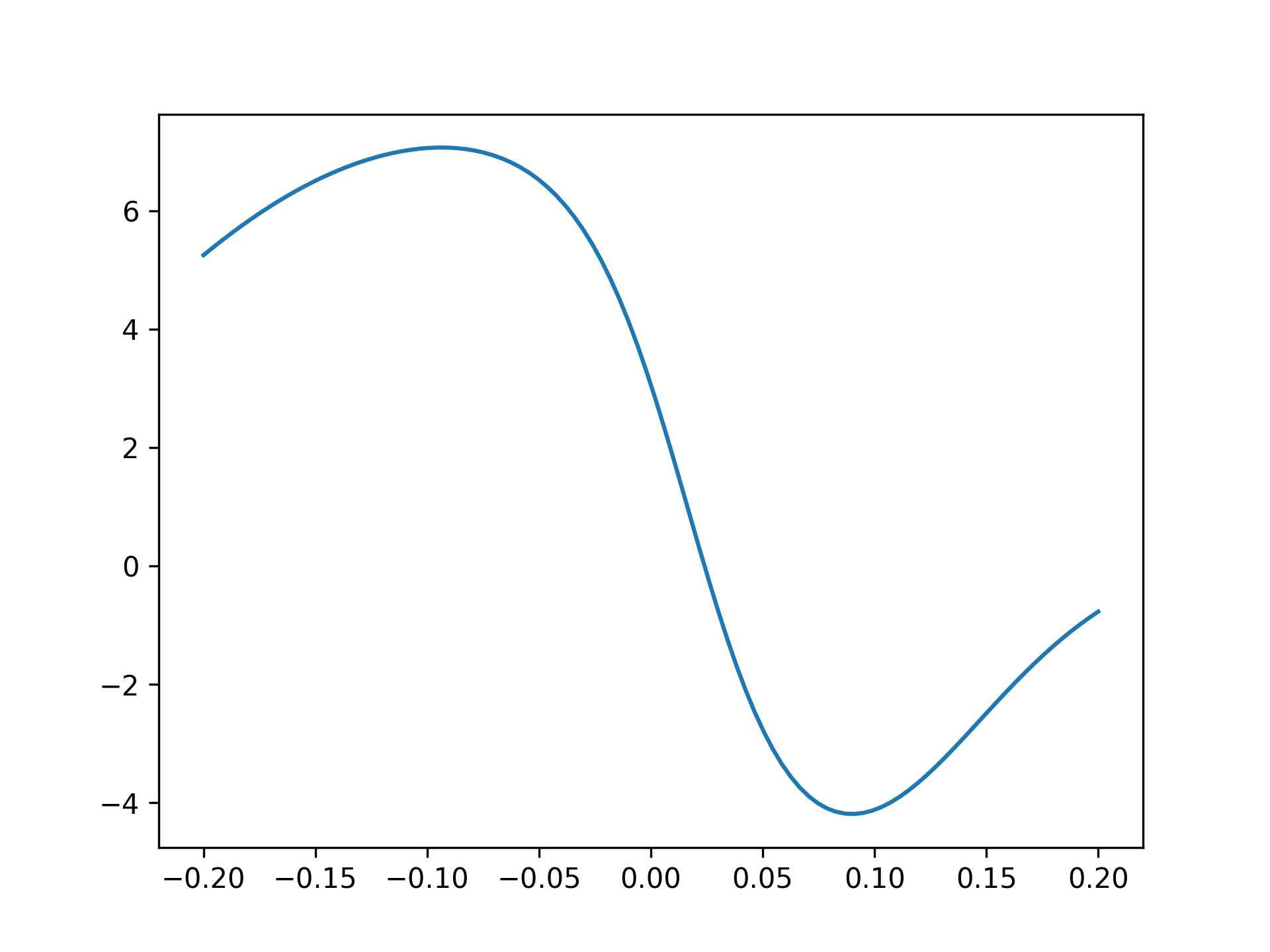}}
	\subfigure[at state $(0.5579, 0.3156, 0)$.]
	{
		\label{fig:8_}
		\includegraphics[scale=0.25]{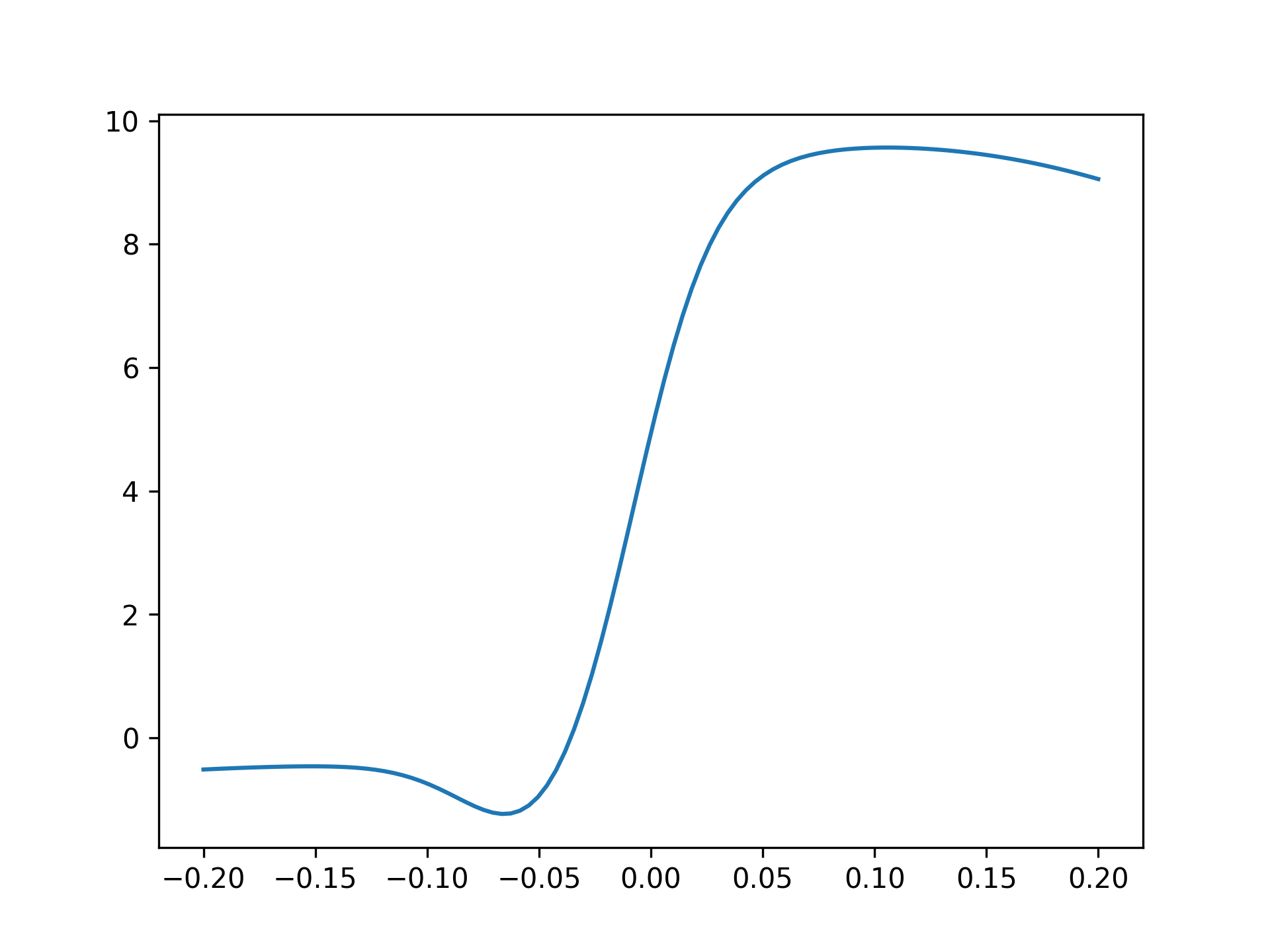}
	}
	\subfigure[at state $(0.1423, 0.3156, -1)$.]
	{
		\label{fig:9_}
		\includegraphics[scale=0.25]{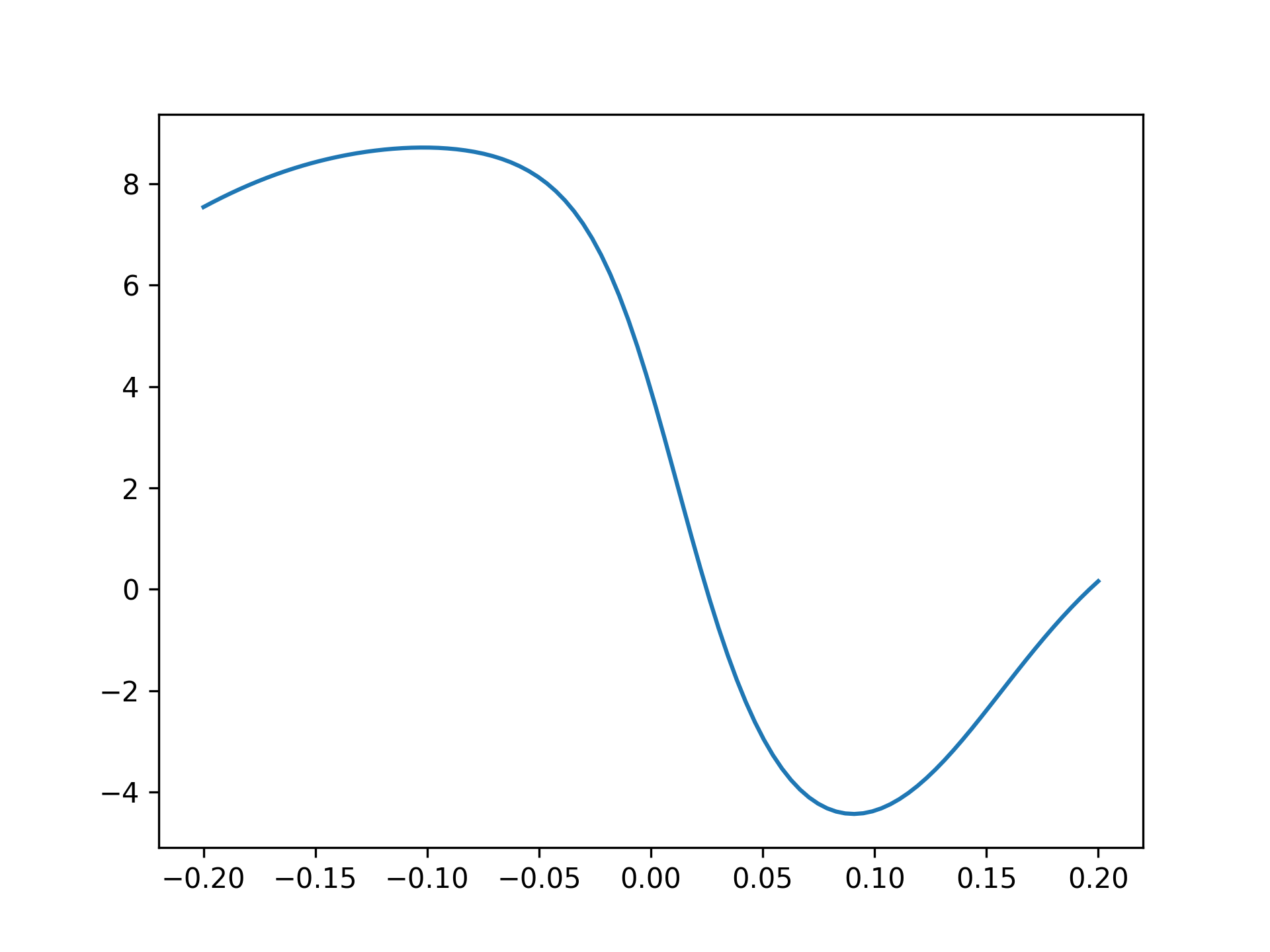}
	}
	
	\caption{The form of Q functions of the auto-bidding policy trained by USCB based on real-world dataset are all nearly in quadratic form. Note that the state have already been normalized, and we choose time left to be $0.9653, 0.6458, 0.3156$, respectively. }
	\label{fig:USCB_Q}
\end{figure}

	\subsubsection{Complete V-CQL Method}
	\label{app:V-CQL_2}
	In this subsection, we specify the novelties of the proposed V-CQL algorithm and analyze its advantages and relations to previous offline RL methods.
	
	\textbf{CQL and its variants.}
	Recall that general CQL \cite{CQL} algorithm (i.e., CQL($\mathcal{R}$)) can be expressed as 
	\begin{align}
		\label{equ:CQL}
		\min_Q\;&-\alpha\underbrace{\mathbb{E}_{s_k\sim\mathcal{D},a_k\sim\widehat{\pi}_\beta}[\widehat{Q}(s_k,a_k)]}_\text{Reward the in-distribution actions}+
		\frac{1}{2}\underbrace{\mathbb{E}_{s_k,a_k,s'_k\sim\mathcal{D}}\bigg[\bigg(\widehat{Q}(s_k,a_k)-\widehat{\bar{B}}\bar{Q}(s_k,a_k)\bigg)^2\bigg]}_\text{\textbf{Bellman error:} minimizing TD error}+\notag\\
		&\underbrace{\max_\mu\;\alpha\mathbb{E}_{s_k\sim\mathcal{D},a\sim\mu}[\widehat{Q}(s_k,a)]+\underbrace{\mathcal{R}(\mu)}_\text{regularizer}}_\text{choose $\mu$ to maximize the current Q-function},
	\end{align}
where $\mathcal{D}$ denotes the offline dataset, $\widehat{\pi}_\beta\triangleq\frac{\sum_{s_k,a_k\sim\mathcal{D}}\mathbf{1}[s=s_k,a=a_k]}{\sum_{s_k\sim\mathcal{D}}\mathbf{1}[s=s_k]}$ is the estimated behavior policy based on $\mathcal{D}$, and $\bar{\mathcal{B}}$ represents the Bellman operator. 
There are two popular variants of CQL, including CQL($\mathcal{H}$) and CQL($\rho$). They both implement the regularizer as a KL-divergence between $\mu$ and a prior distribution $\rho$. Specifically, 
\begin{itemize}
	\item CQL($\mathcal{H}$) chooses $\rho$ to be a uniform policy, i.e., $\mathcal{R}(\mu)=-D_\text{KL}(\mu,\text{Unif}(a))$. Hence, it turns the third term in \eqref{equ:CQL} into a \emph{conservative penalty}.
	\item CQL($\mathcal{\rho}$) chooses $\rho$ to be the previous policy $\widehat{\pi}^{k-1}$, i.e., $\mathcal{R}(\mu)=-D_\text{KL}(\mu,\widehat{\pi}^{k-1})$. Hence, it turns the third term in  \eqref{equ:CQL} into both a \emph{policy constraint} and a \emph{conservative penalty}.
\end{itemize}

\textbf{V-CQL.}
The novelties of the proposed V-CQL algorithm are in three-fold. Firstly, 
as we know the exact formulations of behavior policies generating the data in the offline dataset $\mathcal{D}$ (i.e., the data in $\mathcal{D}_s$ is generated by $\mu_s$, and the data in $\mathcal{D}_{on,\tau}$ is generated by $\pi_{e,\tau}$), we can substitute the $\widehat{\pi}_\beta$ in \eqref{equ:CQL} directly by the behavior policies. This cuts down the estimations process of behavior policy $\widehat{\pi}_\beta$.
For convenience, we uniformly denote the behavior policies as $\mu_b$, where 
\begin{align}
	\mu_b=\left\{\begin{aligned}
	&\mu_s, \;\,\text{for data in $\mathcal{D}_s$},\\
	&\pi_{e,\tau},  \text{for data in $\mathcal{D}_{on,\tau}$}.
	\end{aligned}\right.
\end{align}
Secondly, we adapt the policy constraint in  CQL($\rho$) to a constraint on the Q function. Specifically, we devise the regularizer $\mathcal{R}(\mu)$ as \eqref{equ:V-CQL} in the manuscript. Note that, as $\mathcal{R}(\mu)$ is not a function of $\mu$, the maximizing operation in the third term of \eqref{equ:CQL} is not needed. This way of policy constraint can reduce the performance variance compared to CQL($\mathcal{H}$), and has more flexibilities than the policy constraint in CQL($\rho$) as well as other form of policy constraints direct on policies (such as BCQ). Thirdly, the policy $\rho$ in CQL($\rho$) utilizes the policy in the previous training iterations. The $\widehat{Q}_\text{old}$ in $\mathcal{R}(\mu)$ of the V-CQL also leverages a previous Q function. Nonetheless, $\widehat{Q}_\text{old}$ does not change during the whole training process at iteration $\tau$ and keeps $\widehat{Q}_{\tau-1}$ until the next iteration. Fig. \ref{fig:V-CQL} shows the difference between the V-CQL and CQL($\rho$).
Besides, we adopt the conservative penalty in the CQL($\mathcal{H}$) in the V-CQL method.
Overall,
the V-CQL algorithm can be expressed as
\begin{align}
	\min_Q\quad &\alpha_1\underbrace{\mathbb{E}_{s_k\sim\mathcal{D}}\bigg[\log\sum_{a\sim\text{Unif}(\mathcal{A})}\exp(\widehat{Q}(s_k,a))\bigg]}_{\text{\textbf{conservative penalty:} punishing all actions }}-\alpha_2\underbrace{\mathbb{E}_{s_k\sim\mathcal{D}}\bigg[\widehat{Q}(s_k,\mu_b(s_k))\bigg]}_\text{Reward the in-distribution actions}\notag\\
	&+\frac{1}{2}\underbrace{\mathbb{E}_{s_k,a_k,s'_k\sim\mathcal{D}}\bigg[\bigg(\widehat{Q}(s_k,a_k)-\bar{\mathcal{B}}\bar{Q}(s_k,a_k)\bigg)^2\bigg]}_{\text{\textbf{Bellman error:} minimizing TD error}}\notag\\
	&+\beta\underbrace{\mathbb{E}_{s_k\sim\mathcal{D}}\bigg[D_\text{KL}\bigg(\frac{\exp(\widehat{Q}(s_k,\cdot))}{\sum_{a\sim\text{Unif}(\mathcal{A})}\exp(\widehat{Q}(s_k,a))}\;,\;\frac{\exp(\widehat{Q}_\text{qua}(s_k,\cdot))}{\sum_{a\sim\text{Unif}(\mathcal{A})}\exp(\widehat{Q}_\text{qua}(s_k,a))}\bigg)\bigg]}_{\text{\textbf{policy constraint:} constraining the distribution shifts of the Q function}},
\end{align}
where $\alpha_1, \alpha_2,\beta>0$ are constants, $\bar{\mathcal{B}}$ denotes the Bellman operator, and $\bar{Q}$ is the target Q function. Note that we also randomly sample the actions from the whole action space to calculate the KL-divergence between the old and new Q functions. At iteration $\tau$, the Q function $\widehat{Q}\leftarrow\widehat{Q}_{\tau}$ is trained based on $\widehat{Q}_\text{qua}\leftarrow\widehat{Q}_{\tau-1}$. In practice, the V-CQL can be applied to either Q learning RL algorithms with implicit policies, such as DQN, or actor-critic RL algorithms with explicit policies, such as DDPG. In the SORL, we leverage the DDPG method to train explicit auto-bidding policies, where the Q functions are trained by the V-CQL.

	\begin{figure}[h]
	\centering
	\includegraphics[width=13.5cm]{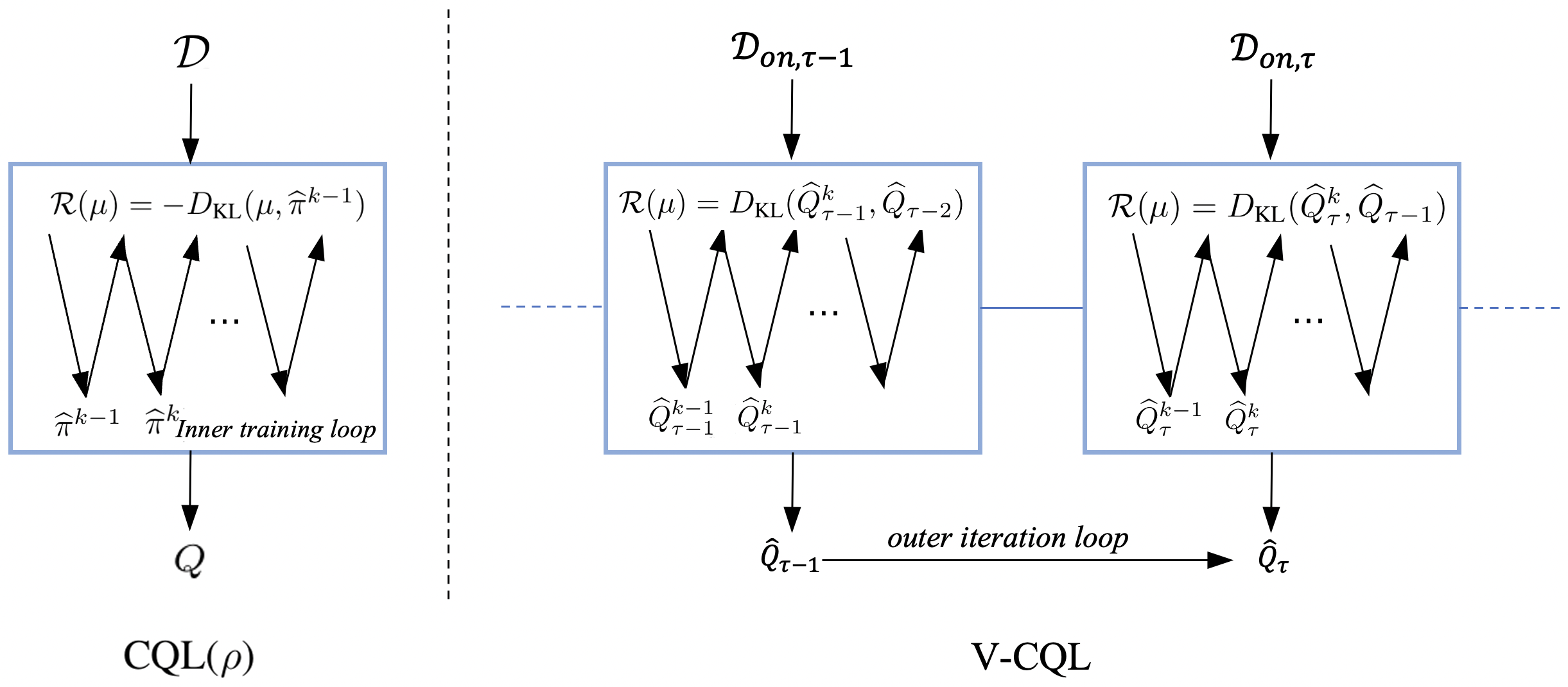}
	\caption{The differences between the V-CQL and CQL($\rho$).}
	\label{fig:V-CQL}
\end{figure}

\textbf{V-CQL combines the advantages of conservative penalty and policy constraint.}
As stated in the Appendix \ref{app:related_work}, there are two main ways to mitigate the extrapolation errors in offline RL methods, including the conservative penalty where explicit punishments are imposed to Q functions (a typical method is the CQL($\mathcal{H}$) \cite{CQL}), and the policy constraints where the KL-divergence between the trained policy and the original policy is limited within a certain range (BCQ \cite{BCQ} imposes constraint directly on the policy, while CQL($\rho$) imposes the constraint on the Q function). The first way has the potential to train policies with high performance, but can have high performance variance. The second way can have low performance variance since it directly imposes the constraints on the policies. However, it generally cannot achieve the  performance as good as the first way \cite{BCQ_CQL}. Besides, imposing constraints directly on policies (as BCQ does) in the SORL framework will face many challenges. For example, the behavior policy of the offline data $\mathcal{D}$ is mixed policy since $\mathcal{D}$ is composed of data collected by different policies. One needs to train a new perturbation model $\xi_{\phi}$ at each iteration $\tau$ for each exploration policy $\pi_{e,\tau}$ and cannot utilize the dataset in previous iterations. As the exploration policy $\pi_{e,\tau}$ does not equal to the auto-bidding policy $\mu_{\tau-1}$ in the previous iteration $\tau-1$, the auto-bidding policy cannot be iteratively improved. Nonetheless, 
the proposed V-CQL combines the advantages of both conservative penalty and policy constraint methods: the V-CQL can reduce the auto-bidding policy's performance variance and iteratively improves the auto-bidding policy in an elegant way.


	\subsection{SORL Framework Pseudocode}
	\label{app:sorl}
	The overall SORL framework algorithm is presented in Algorithm \ref{algorithm:SORL}. {Specifically, the SORL works in an iterative manner. In each iteration $\tau$, the SORL collects data directly from the RAS with the proposed exploration policy $\pi_{e,\tau}$ and use the V-CQL method to train the auto-bidding policy $\mu_\tau$ with the newly collected data. Note that in the first (i.e., $0$-th) iteration, we need a known policy to start the data collection process, and thus, boot the SORL. As the policy will directly interact with the RAS, it should be safe. Hence, we leverage the state-of-the-art auto-bidding policies, for example, USCB \cite{RL:USCB} that has already been deployed to the RAS in practice, to make a warm booting. As proved in the manuscript, the subsequent exploration policies $\pi_{e,\tau}$ is guaranteed to be safe. Due to the constantly feedback from the collected data, the auto-bidding policy $\mu_\tau$ will be improved. }
	
		\begin{algorithm}[h]
		\normalsize\caption{{SORL Framework} }
		\label{algorithm:SORL}
		{\bf Inputs:} The initial safety policy $\mu_s$.\\
		{\bf Outputs:} The auto-bidding policy $\mu^*$ and its Q function $Q^*$.\\
		{\bf Warm Booting:} Collect data $\mathcal{D}_s$ from the RAS with the safe policy $\mu_s$, and train the auto-bidding policy $\mu_{0}$ and $\widehat{Q}_0$. Let $\tau\leftarrow 1$.\\
		{\bf Iteration Process:}
		\begin{algorithmic}[1]
			\normalsize 
			\While{$\widehat{Q}_\tau$ not convergence}
			\State Construct the SER policy $\pi_{e,\tau}$ based on $\mu_s$ and $\widehat{Q}_\tau(s_t,a_t)$;
			\State Explore in the RAS with $\pi_{e,\tau}$ and collect the dataset $\mathcal{D}_{on, \tau}$.
			\State Train the new auto-bidding policy $\mu_\tau$ and its Q function $\widehat{Q}_\tau$ with V-CQL based on the collected data.
			\State $\tau\leftarrow\tau + 1$.
			\EndWhile
			\State Let $Q^*\leftarrow\widehat{Q}_\tau$, and $\mu^*\leftarrow\mu_\tau$.
		\end{algorithmic}
	\end{algorithm}

	\section{Experimental Results}
		\label{app:experiments}
		In this section, we present additional information on the experiment parameters and OPE method in auto-bidding. In addition, we conduct extra experiments to validate the effectiveness of our approach.
		\subsection{Experiment Setup}
		\label{app:experiment_params}
		
		\textbf{Simulated Advertising System.} We conduct experiments on the s-RAS mentioned in Appendix \ref{app:concept_as}.
		The parameters of the simulated advertising system are shown in Table. \ref{table:config_simulated_advertising_system}.
		We implement the V-CQL method by the actor-critic framework.
		The hyper-parameters used in the RL training are summarized in Table. \ref{table:config_sim_vcql_ser}.
		
			\begin{table}[h]
			\caption{The hyper-parameters of DDPG in experiments on the s-RAS.}
			\centering
			\begin{tabular}{ll}
				\toprule
				\makecell[l]{Hyper-parameters}&\makecell[c]{Values}  \\   
				\midrule
				\makecell[l]{Optimizer}    &\makecell[c]{Adam}\\
				\makecell[l]{Learning rate for critic network}    &\makecell[c]{$1\times 10^{-4}$}\\
				\makecell[l]{Learning rate for actor network}    &\makecell[c]{$1\times 10^{-4}$}\\
				\makecell[l]{Soft updated rate }    &\makecell[c]{$0.01$}\\
				\makecell[l]{Buffer size }    &\makecell[c]{$1000$}\\
				\makecell[l]{Sampling size }    &\makecell[c]{$200$}\\
				\makecell[l]{Discounted factor $\gamma$ }    &\makecell[c]{$0.99$}\\
				\makecell[l]{$\alpha_1,\alpha_2$ for V-CQL}    &\makecell[c]{$0.002$}\\
			\makecell[l]{$\beta$ for V-CQL}    &\makecell[c]{$0.001$}\\
				\makecell[l]{$\sigma$ in SER policy $\pi_\tau$}    &\makecell[c]{$1$}\\
					\makecell[l]{$\lambda$ in SER policy $\pi_\tau$}    &\makecell[c]{$0.1$}\\
						\makecell[l]{sample numbers of SER policy $\pi_\tau$}    &\makecell[c]{$1000$}\\
					\makecell[l]{$\xi$ sample range of SER policy $\pi_\tau$}    &\makecell[c]{$0.5$}\\
				\bottomrule
			\end{tabular}
			\label{table:config_sim_vcql_ser}
		\end{table}
	
		\textbf{Real-world Advertising System (RAS).} We conduct experiments on one the world's largest E-commerce platforms, TaoBao. We apply the SORL framework to thousands of real advertisers from April 28, 2022 to May 26, 2022 to validate the effectiveness of it. We implement the V-CQL method by the actor-critic framework, whose hyper-parameters are summarized in Table. \ref{table:config_RAS_DDPG}.
		
			\begin{table}[h]
			\caption{The parameters used in the s-RAS.}
			\centering
			\begin{tabular}{ll}
				\toprule
				\makecell[l]{Parameters}&\makecell[c]{Values}  \\   
				\midrule
				\makecell[l]{Number of advertisers}    &\makecell[c]{$100$}\\
				\makecell[l]{Time steps in an episode, $T$} &\makecell[c]{$96$} \\
				\makecell[l]{Minimum number of impression opportunities $N_t$} &\makecell[c]{$100$} \\
				\makecell[l]{Maximum number of impression opportunities $N_t$} &\makecell[c]{$500$} \\
				\makecell[l]{Minimum budget} &\makecell[c]{$100,000$ Yuan} \\
				\makecell[l]{Maximum budget} &\makecell[c]{$200,000$ Yuan} \\
				\makecell[l]{Value of impression opportunities in stage 1, $v_{j,t}^1$} &\makecell[c]{$0\sim 1$} \\			
				\makecell[l]{Value of impression opportunities in stage 2, $v_{j,t}^2$} &\makecell[c]{$0\sim 1$} \\			
				\makecell[l]{Minimum bidding price, $A_\text{min}$} &\makecell[c]{$0$ Yuan} \\
				\makecell[l]{Maximum bidding price, $A_\text{max}$} &\makecell[c]{$1,000$ Yuan} \\					
				\makecell[l]{Maximum value of impression opportunity, $v_M$} &\makecell[c]{$1$} \\
				\makecell[l]{Maximum market price, $p_M$} &\makecell[c]{$1,000$ Yuan} \\					
				\bottomrule
			\end{tabular}
			\label{table:config_simulated_advertising_system}
		\end{table}

			\begin{table}[h]
			\caption{The hyper-parameters of DDPG when applying the SORL to the RAS.}
			\centering
			\begin{tabular}{ll}
				\toprule
				\makecell[l]{Hyper-parameters}&\makecell[c]{Values}  \\   
				\midrule
				\makecell[l]{Optimizer}    &\makecell[c]{Adam}\\
				\makecell[l]{Learning rate for critic network}    &\makecell[c]{$2\times 10^{-5}$}\\
				\makecell[l]{Learning rate for actor network}    &\makecell[c]{$2\times 10^{-5}$}\\
				\makecell[l]{Soft updated rate }    &\makecell[c]{$0.01$}\\
				\makecell[l]{Buffer size }    &\makecell[c]{ $\sim 10000$}\\
				\makecell[l]{Sampling size }    &\makecell[c]{$64$}\\
				\makecell[l]{Discounted factor $\gamma$ }    &\makecell[c]{$0.999$}\\
				\makecell[l]{$\alpha_1,\alpha_2$ for V-CQL}    &\makecell[c]{$0.001$}\\
					\makecell[l]{$\beta$ for V-CQL}    &\makecell[c]{$0.002$}\\
						\makecell[l]{$\sigma$ in SER policy $\pi_\tau$}    &\makecell[c]{$0.15$}\\
					\makecell[l]{$\lambda$ in SER policy $\pi_\tau$}    &\makecell[c]{$0.3$}\\
					\makecell[l]{sample numbers of SER policy $\pi_\tau$}    &\makecell[c]{$50$}\\
					\makecell[l]{$\xi$ sample range of SER policy $\pi_\tau$}    &\makecell[c]{$0.1$}\\
					
				\bottomrule
			\end{tabular}
			\label{table:config_RAS_DDPG}
		\end{table}

		\subsection{Additional Results}
		\label{app:additional_results}
		\subsubsection{Ablation Study: Compare the V-CQL with BCQ, CQL($\mathcal{H}$) and CQL($\rho$)}
		\label{app:q1}
		In the manuscript, we compare the V-CQL with the CQL method, specifically CQL($\mathcal{H}$). Here, we compare the V-CQL with more offline RL methods in the RAS, which can serve as an ablation study.
		As shown in Table. \ref{table:ablation}, the V-CQL outperforms the BCQ, CQL($\mathcal{H}$) and CQL($\rho$) in all metrics in the RAS. Firstly, the comparison with BCQ and CQL($\mathcal{H}$) indicates that the V-CQL combines the advantages of conservative penalty and policy constraint. Secondly, the comparison with the CQL($\rho$) validates the effectiveness of the proposed form of policy constraint \eqref{equ:V-CQL} in the manuscript. 
		
			\begin{table}[h]
			\caption{Ablation study in the RAS: compare V-CQL with BCQ (policy constraint), CQL($\mathcal{H}$) (conservative penalty) and CQL($\rho$) (both policy constraint and conservative penalty).}
			\centering
			\begin{tabular}{lllll}
				\toprule
				\multirow{2}*{Methods}&\multicolumn{4}{c}{no conservative penalty: V-CQL vs. BCQ}     \\   
				\cmidrule(r){2-5}
				& BuyCnt &ROI &CPA&ConBdg       \\
				\midrule
				BCQ &8,746&3.78&23.85&208,576.53\\
				V-CQL&8,992&3.94&23.48&211,142.59\\
				variety&\textbf{+2.81\%}&\textbf{+4.23\%}&\textbf{-1.54\%}&\textbf{+1.23\%}\\
				\midrule
				\multirow{2}*{Methods}&\multicolumn{4}{c}{no policy constraint: V-CQL vs. CQL($\mathcal{H}$)}     \\   
				\cmidrule(r){2-5}
				&BuyCnt   &ROI&CPA&ConBdg       \\
				\midrule
			CQL($\mathcal{H}$)    &40,462    &3.87 &21.42 &845,621.15 \\
				V-CQL    &{42,236} &{3.95} &{20.29}&{856,913.14}\\
				variety&\textbf{+4.38\%}&\textbf{+2.07\%}&\textbf{-5.27\%}&\textbf{+1.33\%}\\
				\midrule
					\multirow{2}*{Methods}&\multicolumn{4}{c}{with different versions of policy constraint : V-CQL vs. CQL($\rho$)}     \\   
				\cmidrule(r){2-5}
				& BuyCnt &ROI &CPA&ConBdg       \\
				\midrule
				CQL($\rho$)    &9,523   & 4.04 &20.99 &199,873.22 \\
				V-CQL    &{9,867} &{4.20} &{20.30}&{200,291.00}\\
				variety&\textbf{+3.61\%}&\textbf{+3.96\%}&\textbf{-3.28\%}&\textbf{+0.21\%}\\

				\bottomrule
			\end{tabular}
			\label{table:ablation}
		\end{table}
		
		\subsubsection{Affects of Hyper-parameters $\sigma$ and $\lambda$ on SER Policy $\pi_\tau$}
		\label{app:q2}
		We apply the SER policy $\pi_e$ to the s-RAS with different hyper-parameters $\sigma$ and $\lambda$, and the total accumulated rewards (Q value)  are shown in Fig. \ref{fig:sigma_lambda}. Specifically, Fig. \ref{fig:phi_1} shows the accumulated rewards of $\pi_e$ constructed by the Q function $\widehat{Q}_{(1)}(s_t,a_t)$ that has poor performance (with ROI of $3.39$), while Fig. \ref{fig:phi_7} shows the accumulated rewards of $\pi_e$ constructed by the Q function $\widehat{Q}_{(7)}(s_t,a_t)$ that has good performance (with ROI of $3.82$). The initial safe policy $\mu_s$ has a total accumulated reward of $212.36$ and a ROI of $3.64$. We can see that the declines in Q values of $\pi_e$ under all hyper-parameters are within $5\%$ with respect to the Q value of $\mu_s$. This indicates the safety of the SER policy. Moreover, from Fig. \ref{fig:phi_1}, we can see that when the Q function of $\pi_e$ has poor performance, larger $\lambda$ and smaller $\sigma$ can make the SER policy $\pi_e$ safer. This is because $\widehat{Q}_{(1)}(s_t,a_t)$ does not lead the explorations to a good direction, and stick to the safe policy $\mu_s$ can be a safer choice. On the contrary, from Fig. \ref{fig:phi_7}, we can see that when the Q function has good performance, smaller $\lambda$ can make $\pi_e$ more safer, and $\sigma$ can be set to a larger value to increase the exploration efficiency.
		
		\begin{figure}[h]
			\centering
			\subfigure[Q value of $\pi_e$ with Q function $\widehat{Q}_{(1)}(s,a)$ that has poor performance.]
			{\label{fig:phi_1}
				\includegraphics[scale=0.14]{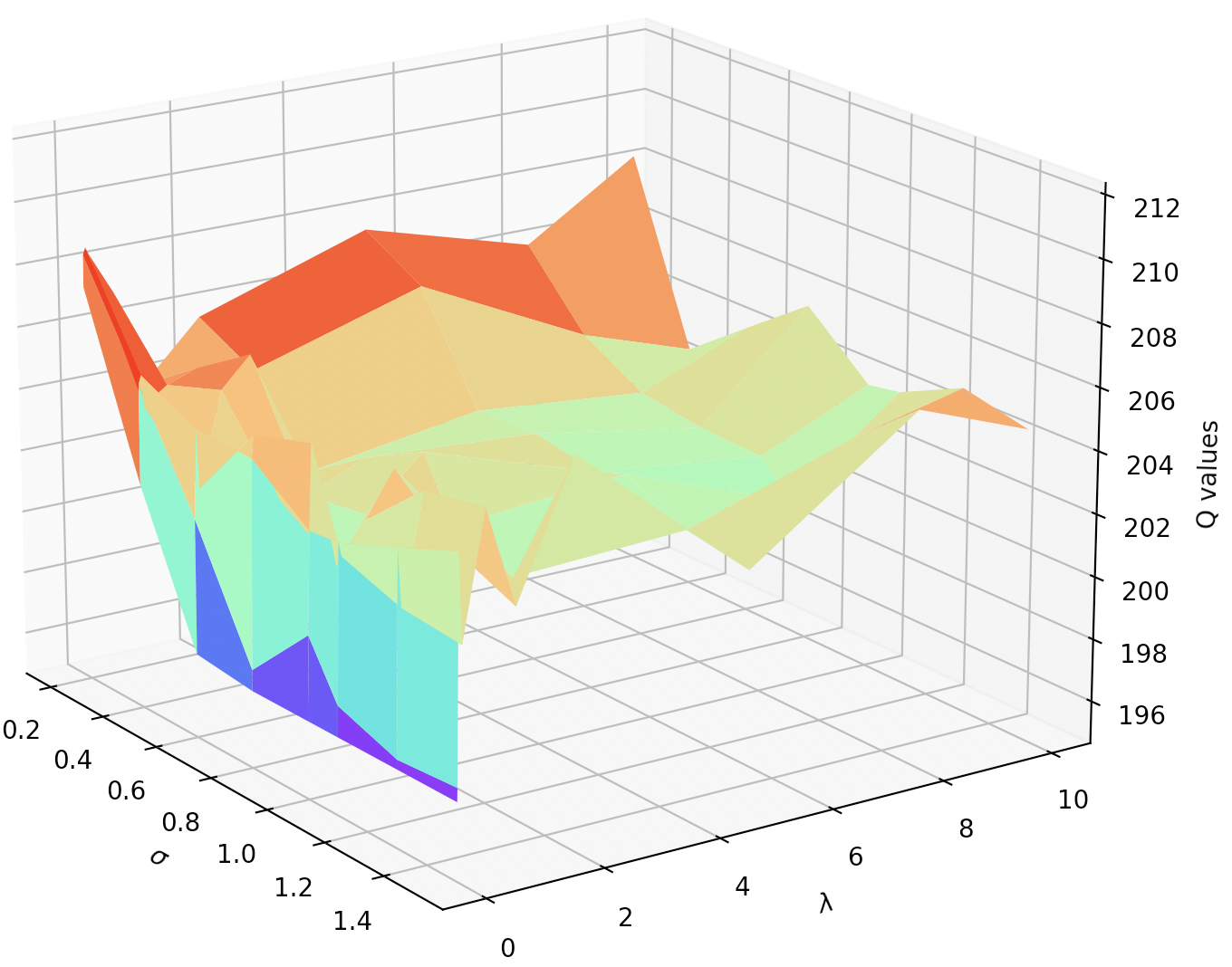}}
			\subfigure[Q value of $\pi_e$ with Q function  $\widehat{Q}_{(7)}(s,a)$ that has good performance.]
			{
				\label{fig:phi_7}
				\includegraphics[scale=0.14]{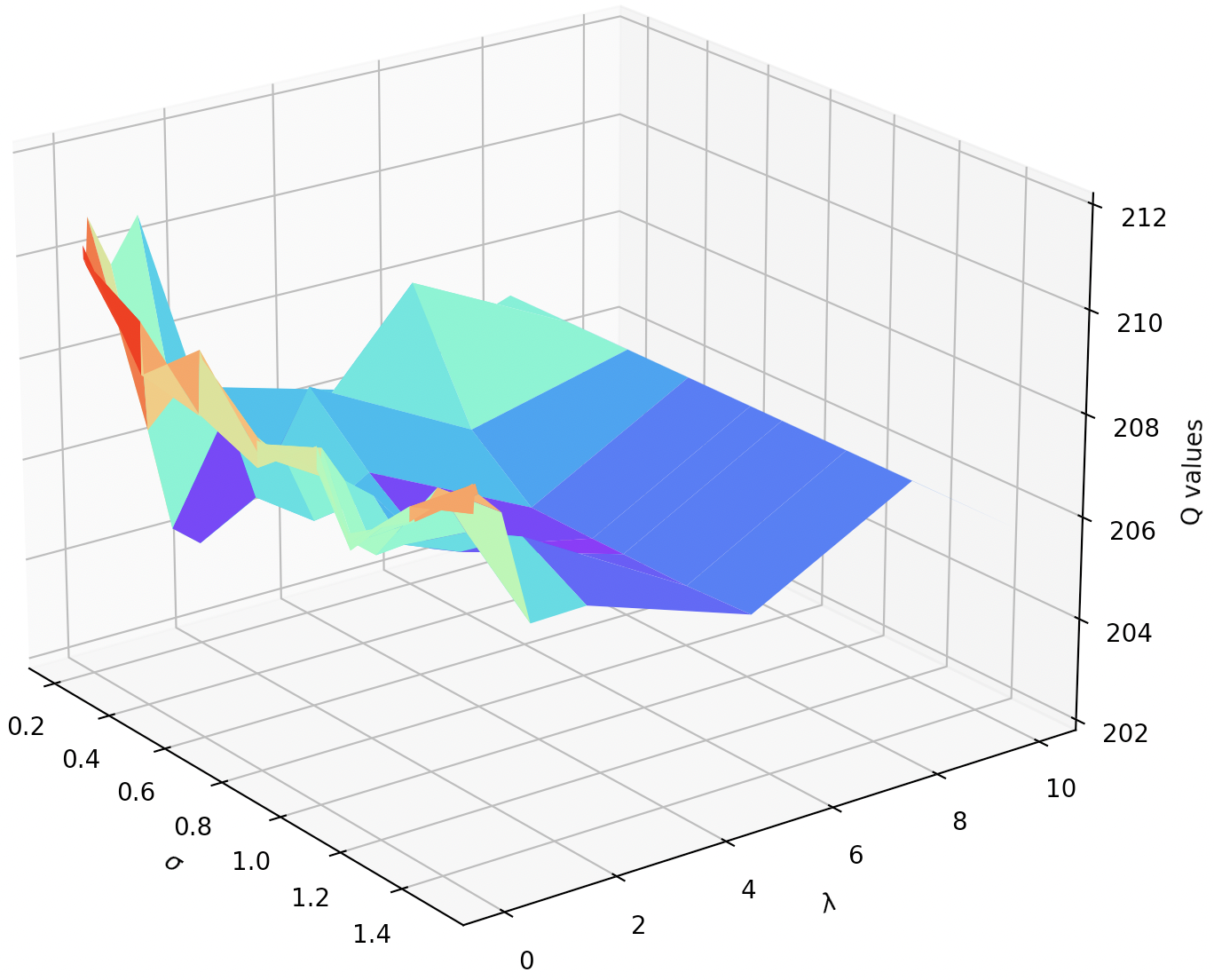}
			}
		
			\caption{The accumulated rewards of the SER policy $\pi_e$ constructed by Q functions with different performance level under different hyper-parameters $\sigma$ and $\lambda$.}
			\label{fig:sigma_lambda}
		\end{figure}
		
			\subsubsection{Complete Results of Table. \ref{table:SORL} in the Manuscript}
		\label{app:q3}
		The complete experiment results of Table. \ref{table:SORL} in the manuscript are shown in Table \ref{table:SORL_complete}. 
		
		\subsubsection{Compare With Multi-agent Auto-bidding Algorithms}
		\label{app:q4}
		Some researchers may consider the comparison between our approach and the multi-agent auto-bidding algorithms. However, we claim that
		the problem setting of this paper is different from that of the multi-agent algorithms \cite{multi-agent}.
		As we stated in Section \ref{sec:problem_setting}, the problem we considered is how to maximize the total value of a single advertiser, which is naturally a single-agent problem. What we argue in Fig. \ref{fig:IBOO} is that all other advertisers acting as a part of the environment are not correctly modeled in the VAS, which will cause the IBOO. In addition, our method can solve the IBOO, including the inaccurate market price changing issue. However, the multi-agent auto-bidding problem \cite{multi-agent}  studies how to realize multi-objective goals, involving the interests of advertisers and the platform. Hence, it is not very proper to compare our algorithm with the multi-agent algorithms \cite{multi-agent} that solves a different problem.
		
		Nonetheless, we conduct real-world A/B test between our approach and the multi-agent algorithm in \cite{multi-agent}, and the results are shown in Table. \ref{table:multi-agent}. We can see that the SORL largely outperforms than the multi-agent algorithm in the performance indexes considered in the single-agent auto-bidding problem in our paper.
			\begin{table}[h]
			\caption{Real-world A/B tests between the SORL and the multi-agent auto-bidding algorithm.}
			\centering
			\small
			\begin{tabular}{lllll}
				\toprule
			Algorithms &BuyCnt & ROI & CPA & ConBdg\\
				\midrule
				multi-agent method \cite{multi-agent} & $121,616$ & $2.79$ & $44.86$ &$5,455,507.35$ \\
			
				SORL    &$\textbf{139,599}$& $\textbf{3.15}$ &$\textbf{40.05}$&$\textbf{5,590,858.50}$  \\
					\midrule 
				varieties &$\textbf{+14.79\%}
				 $&$\textbf{+12.90\%}$&$\textbf{-10.72\%}$&$\textbf{+2.48\%}$\\
				\bottomrule
			\end{tabular}
			\label{table:multi-agent}
		\end{table}
	
		\begin{table}[h]
			\caption{The complete experiment results of SORL framework in the {RAS}.}
			\centering
			\small
			\begin{tabular}{llllll}
				\toprule
			\multirow{2}{*}{\makecell[c]{Metrics}}&\	\multirow{2}{*}{\makecell[c]{A/B Tests}}&\multicolumn{4}{c}{Auto-bidding policy $\mu_\tau$ derived in iteration $\tau$}\\
				\cmidrule(r){3-6}
				&&\makecell[c]{0-th: $\mu_{0}$}&\makecell[c]{1-th: $\mu_{1}$}&\makecell[c]{2-th: $\mu_{2}$}&\makecell[c]{3-th: $\mu_{3}$}\\
				\midrule
				\multirow{6}{*}{BuyCnt}&\multirow{3}{*}{auto-bidding policy $\mu_{\tau-1}$}&\makecell[c]{40,926.00}&\makecell[c]{7,982.00}&\makecell[c]{8,571.00}&\makecell[c]{9,207.00} \\
			\cmidrule(r){3-6}
		&&\makecell[c]{42,236.00}&\makecell[c]{8,034.00}&\makecell[c]{8,611.00}&\makecell[c]{9,295.00} \\
		\cmidrule(r){3-6}
		&&\makecell[c]{\textbf{+3.20\%}}&\makecell[c]{\textbf{+0.65\%}}&\makecell[c]{\textbf{+0.47\%}}&\makecell[c]{\textbf{+0.95\%}} \\
		\cmidrule(r){2-6}
		&\multirow{3}{*}{\makecell[c]{USCB\\ (state-of-the-art)}}&\makecell[c]{40,926.00}&\makecell[c]{6,358.00}&\makecell[c]{7,432.00}&\makecell[c]{9,358.00} \\
			\cmidrule(r){3-6}
		&&\makecell[c]{42,236.00}&\makecell[c]{6,575.00}&\makecell[c]{7,697.00}&\makecell[c]{9,709.00} \\
		\cmidrule(r){3-6}
		&&\makecell[c]{\textbf{+3.20\%}}&\makecell[c]{\textbf{+3.41\%}}&\makecell[c]{\textbf{+3.57\%}}&\makecell[c]{\textbf{+3.75\%}} \\
		\midrule
		
			\multirow{6}{*}{ROI}&\multirow{3}{*}{auto-bidding policy $\mu_{\tau-1}$}&\makecell[c]{3.90}&\makecell[c]{3.58}&\makecell[c]{3.88}&\makecell[c]{3.20} \\
		\cmidrule(r){3-6}
		&&\makecell[c]{3.95}&\makecell[c]{3.65}&\makecell[c]{3.89}&\makecell[c]{3.31} \\
		\cmidrule(r){3-6}
		&&\makecell[c]{\textbf{+1.28\%}}&\makecell[c]{\textbf{+1.96\%}}&\makecell[c]{\textbf{+0.26\%}}&\makecell[c]{\textbf{+3.20\%}} \\
		\cmidrule(r){2-6}
		&\multirow{3}{*}{\makecell[c]{USCB\\ (state-of-the-art)}}&\makecell[c]{3.90}&\makecell[c]{3.47}&\makecell[c]{3.76}&\makecell[c]{3.47} \\
		\cmidrule(r){3-6}
		&&\makecell[c]{3.95}&\makecell[c]{3.57}&\makecell[c]{3.82}&\makecell[c]{3.55} \\
		\cmidrule(r){3-6}
		&&\makecell[c]{\textbf{+1.28\%}}&\makecell[c]{\textbf{+2.88\%}}&\makecell[c]{\textbf{+1.60\%}}&\makecell[c]{\textbf{+2.48\%}} \\
		\midrule
		
			\multirow{6}{*}{CPA}&\multirow{3}{*}{auto-bidding policy $\mu_{\tau-1}$}&\makecell[c]{20.71}&\makecell[c]{21.32}&\makecell[c]{22.38}&\makecell[c]{23.21} \\
		\cmidrule(r){3-6}
		&&\makecell[c]{20.29}&\makecell[c]{21.05}&\makecell[c]{22.35}&\makecell[c]{23.44} \\
		\cmidrule(r){3-6}
		&&\makecell[c]{\textbf{-2.01\%}}&\makecell[c]{\textbf{-1.27\%}}&\makecell[c]{\textbf{-0.13\%}}&\makecell[c]{\textbf{-1.01\%}} \\
		\cmidrule(r){2-6}
		&\multirow{3}{*}{\makecell[c]{USCB\\ (state-of-the-art)}}&\makecell[c]{20.71}&\makecell[c]{22.52}&\makecell[c]{20.31}&\makecell[c]{24.52} \\
		\cmidrule(r){3-6}
		&&\makecell[c]{20.29}&\makecell[c]{22.31}&\makecell[c]{20.11}&\makecell[c]{23.60} \\
		\cmidrule(r){3-6}
		&&\makecell[c]{\textbf{-2.01\%}}&\makecell[c]{\textbf{-0.93\%}}&\makecell[c]{\textbf{-0.98\%}}&\makecell[c]{\textbf{-3.91\%}} \\
		\midrule

			\multirow{6}{*}{ConBdg}&\multirow{3}{*}{auto-bidding policy $\mu_{\tau-1}$}&\makecell[c]{847,403.12}&\makecell[c]{170,176.24}&\makecell[c]{191,818.98}&\makecell[c]{215,695.14} \\
		\cmidrule(r){3-6}
		&&\makecell[c]{856,913.14}&\makecell[c]{169,115.70}&\makecell[c]{192,455.85}&\makecell[c]{215,828.40} \\
		\cmidrule(r){3-6}
		&&\makecell[c]{\textbf{+1.12\%}}&\makecell[c]{{-0.62\%}}&\makecell[c]{\textbf{+0.33\%}}&\makecell[c]{\textbf{+0.06\%}} \\
		\cmidrule(r){2-6}
		&\multirow{3}{*}{\makecell[c]{USCB\\ (state-of-the-art)}}&\makecell[c]{847,403.12}&\makecell[c]{143,182.16}&\makecell[c]{150,943.92}&\makecell[c]{229,492.24} \\
		\cmidrule(r){3-6}
		&&\makecell[c]{856,913.14}&\makecell[c]{146,688.25}&\makecell[c]{154,786.67}&\makecell[c]{229,141.06} \\
		\cmidrule(r){3-6}
		&&\makecell[c]{\textbf{+1.12\%}}&\makecell[c]{\textbf{+2.45\%}}&\makecell[c]{\textbf{+2.55\%}}&\makecell[c]{{-0.15\%}} \\

				\bottomrule
			\end{tabular}
			\label{table:SORL_complete}
		\end{table}

	\section{Broader Impact}
	\label{app:social_impact}
	In this paper, we propose a SORL framework to improve the state-of-the-art auto-bidding policies with direct explorations in the real-world advertising system (RAS). To the best of our knowledge, we are the first to systematically analyze the IBOO problem in auto-bidding and complete resolve it with an online RL manner. The derived auto-bidding policy can benefit both the advertisers and the companies at the same time, which can generate huge social and economic benefits. We believe that the proposed SORL framework will be the next generation of auto-bidding paradigm.
	In addition, 
	the IBOO problem does not only exists in the auto-bidding. In fact, it resembles the sim2real problem in many other realms, such as robotics. The proposed SORL framework is a general method which can be easily applied to other applications to solve the sim2real problem.


{\small
\begin{thebibliography}{99}

	\bibitem{RL:wangjun}
	D. Wu, X. Chen, X. Yang, et al., Budget constrained bidding by model-free reinforcement learning in display advertising, in \emph{Proceedings of the 27th ACM International Conference on Information and Knowledge Management,} Oct. 2018.
	\bibitem{RL:haoxiaotian}
	X. Hao, Z. Peng, Y. Ma, et al., Dynamic knapsack optimization towards efficient multi-channel sequential advertising, in \emph{Proceedings of the 37th International Conference on Machine Learning,} Nov. 2020.
	\bibitem{alibaba}
	Alibaba. 2022. alimama, https://www.alimama.com/index.htm.
	\bibitem{RL:cai}
	H. Cai, K. Ren, W. Zhang, et al., Real-time bidding by reinforcement learning in display advertising, in \emph{Proceedings of the 10th ACM International Conference on Web Search and Data Mining,} Jan. 2017.
	\bibitem{RL:USCB}
	Y. He, X. Chen, D. Wu, et al., A unified solution to constrained bidding in online display advertising, in \emph{Proceedings of the 27th ACM SIGKDD Conference on Knowledge Discovery \& Data Mining}, Aug. 2021.
	\bibitem{google}
	Google. 2022. google ads, https://ads.google.com/.
	\bibitem{stage}
	Z. Wang, L. Zhao, B. Jiang B, et al., Cold: Towards the next generation of pre-ranking system, \emph{arXiv preprint arXiv:2007.16122,} Aug. 2020.
	\bibitem{sim2real_cv}
	D. Carl, and A. Zisserman, Sim2real transfer learning for 3d human pose estimation: motion to the rescue, in \emph{Advances in Neural Information Processing Systems 32,} 2019.
	\bibitem{sim2real_robot}
	S. Höfer, K. Bekris, A. Handa, et al., Sim2Real in robotics and automation: Applications and challenges, \emph{IEEE transactions on automation science and engineering,} vol. 18, no. 2, pp. 398--400, Apr. 2021.
	\bibitem{auction}
	X. Liu, C. Yu, Z. Zhang, et al., Neural auction: End-to-end learning of auction mechanisms for e-commerce advertising, in \emph{ Proceedings of the 27th ACM SIGKDD Conference on Knowledge Discovery \& Data Mining,} Aug. 2021.
	\bibitem{BCQ}
	S. Fujimoto, D. Meger, and D. Precup, Off-policy deep reinforcement learning without exploration,
	in \emph{Proceedings of 36th International Conference on Machine Learning,} May. 2019.
	\bibitem{OPE}
	C. Voloshin, HM. Le, N. Jiang, and Y. Yue, Empirical study of off-policy policy evaluation for reinforcement learning, in \emph{ 35th Conference on Neural Information Processing Systems Datasets and Benchmarks Track (Round 1),} Jun. 2021.
	\bibitem{CQL}
	A. Kumar, A. Zhou, G. Tucker, and S. Levine, Conservative q-learning for offline reinforcement learning, in \emph{Advances in Neural Information Processing Systems 33,} 2020.
	\bibitem{AWAC}
	A. Nair, A. Gupta, M. Dalal, and S. Levine, Awac: Accelerating online reinforcement learning with offline datasets, \emph{arXiv preprint arXiv:2006.09359}, 2020.
	\bibitem{CMDP}
	E. Altman, Constrained Markov decision processes: stochastic modeling, \emph{Routledge,} 1999.
	\bibitem{DDPG}
	TP. Lillicrap, JJ. Hunt, A. Pritzel, et al., Continuous control with deep reinforcement learning, \emph{arXiv preprint arXiv:1509.02971,} 2015.
	\bibitem{CSC}
	H. Bharadhwaj, A. Kumar, N. Rhinehart, et al., Conservative safety critics for exploration, \emph{arXiv preprint arXiv:2010.14497,} 2020.
	\bibitem{irl}
	DRR. Scobee, SS. Sastry, Maximum likelihood constraint inference for inverse reinforcement learning, \emph{arXiv preprint arXiv:1909.05477,} 2019.
	\bibitem{icrl}
	U. Anwar, S. Malik, A. Aghasi, et al., Inverse Constrained Reinforcement Learning, \emph{arXiv preprint arXiv:2011.09999,} 2020.
	\bibitem{safeRL:query}
	B. Thananjeyan, A. Balakrishna, U. Rosolia, et al., Safety augmented value estimation from demonstrations (saved): Safe deep model-based rl for sparse cost robotic tasks, \emph{IEEE Robotics and Automation Letters,} vol. 5, no. 2, pp. 3612--3619, 2020.
	\bibitem{liyapnouv}
	F. Berkenkamp, M. Turchetta, A. Schoellig, et al., Safe model-based reinforcement learning with stability guarantees, in \emph{Advances in neural information processing systems,} 2017.
	\bibitem{learning-based}
	T. Koller, F. Berkenkamp, M. Turchetta, et al., Learning-based model predictive control for safe exploration, \emph{IEEE Conference on Decision and Control (CDC),} 2018.
	\bibitem{O2O}
	S. Lee, Y. Seo, K. Lee, et al., Offline-to-Online Reinforcement Learning via Balanced Replay and Pessimistic Q-Ensemble, in \emph{Conference on Robot Learning,} 2022.
	\bibitem{parameter}
	TL. Paine, C. Paduraru, A. Michi, et al., Hyperparameter selection for offline reinforcement learning, \emph{arXiv preprint arXiv:2007.09055,} 2020.
	\bibitem{sim2real:thermal_power}
	X. Zhan, H. Xu, Y. Zhang, et al., Deepthermal: Combustion optimization for thermal power generating units using offline reinforcement learning, \emph{arXiv preprint arXiv:2102.11492,} 2021.
	\bibitem{BCQ_CQL}
	R. Qin, S. Gao, X. Zhang, et al. NeoRL: A near real-world benchmark for offline reinforcement learning, \emph{arXiv preprint arXiv:2102.00714,} 2021.
	\bibitem{offline_review}
	S. Levine, A. Kumar, G. Tucker, et al., Offline reinforcement learning: Tutorial, review, and perspectives on open problems, \emph{arXiv preprint arXiv:2005.01643,} 2020.
	\bibitem{dir_1}
	P. Jan, M. Katharina, and A. Yasemin, Relative Entropy Policy Search,  in \emph{Proceedings of the AAAI Conference on Artificial Intelligence,}  2010.
\bibitem{dir_2}
X. Peng, A. Kumar, G. Zhang, and S.
Levine. Advantage-Weighted Regression: Simple and
Scalable Off-Policy Reinforcement Learning, \emph{arXiv preprint arXiv:1910.00177,} 2019.
\bibitem{BCQ_1}
A. Kumar, J. Fu, M. Soh, G.  Tucker, and S. Levine, Stabilizing off-policy q-learning via bootstrapping error reduction. In \emph{Advances in Neural Information Processing Systems 32,} 2019. 
\bibitem{BCQ_2}
N. Siegel, J.  Springenberg, F.  Berkenkamp, A. Abdolmaleki, M. Neunert,T. Lampe, R. Hafner,and M. Riedmiller, Keep doing what worked: Behavioral modelling priors for offline reinforcement learning. \emph{arXiv preprint arXiv:2002.08396,} 2020.
\bibitem{parameter_selection}
A. Kumar, A. Singh, S. Tian, et al., A workflow for offline model-free robotic reinforcement learning, \emph{arXiv preprint arXiv:2109.10813,} 2021.
\bibitem{5}
M. Turchetta, F.  Berkenkamp, and A. Krause, Safe exploration in finite markov decision processes with gaussian processes, in \emph{Advances in Neural Information Processing Systems 29,} 2016.
\bibitem{6}
G. Dalal, K. Dvijotham, M. Vecerik, T. Hester, C. Paduraru,  and Y. Tassa, Y. Safe exploration in continuous action spaces, \emph{arXiv preprint arXiv:1801.08757,}  2018.
\bibitem{7}
R. Laroche, P. Trichelair,  and  R. T. Des Combes, Safe policy improvement with baseline bootstrapping, in \emph{Proceedings of 36th International Conference on Machine Learning,}  2019.
\bibitem{multi-agent}
Z. Guan, H. Wu, Q.  Cao, et al., Multi-Agent Cooperative Bidding Games for Multi-Objective Optimization in e-Commercial Sponsored Search,  in \emph{Proceedings of the 27th ACM SIGKDD Conference on Knowledge Discovery \& Data Mining,} 2021.
\bibitem{onestep}
D. Brandfonbrener, W. F. Whitney, R. Ranganath, and J. Bruna, Offline RL Without Off-Policy Evaluation,
in \emph{Advances in Neural Information Processing Systems 34,} 2020. 
\bibitem{CRR}
Z. Wang, A. Novikov, K. Zolna, J.  Merel, J.  S., S.
Reed, B. Shahriari, N. Siegel, C. Gulcehre, N. Heess, et al,. Critic regularized
regression. \emph{Advances in Neural Information Processing Systems 33,} 2019.
\bibitem{BAIL}
X. Chen, Z. Zhou, Z. Wang, .C. Wang, Y. Wu, and K. Ross, Bail: Bestaction imitation learning for batch deep reinforcement learning, \emph{Advances in Neural Information
Processing Systems 33,} 2019.
\bibitem{BRAC}
Y. Wu, G. Yucker, and O. Nachum, Behavior Regularized Offline Reinforcement Learning, \emph{arXiv preprint arXiv:1911.11361}, 2019.
	
\end{thebibliography}
}
\end{document}